\documentclass{article}

    \usepackage[final]{neurips_template/neurips_2023}

\usepackage[utf8]{inputenc} % allow utf-8 input
\usepackage[T1]{fontenc}    % use 8-bit T1 fonts
\usepackage{hyperref}       % hyperlinks
\usepackage{url}            % simple URL typesetting
\usepackage{subcaption, booktabs}       % professional-quality tables
\usepackage{amsfonts}       % blackboard math symbols
\usepackage{nicefrac}       % compact symbols for 1/2, etc.
\usepackage{microtype}      % microtypography
\usepackage{xcolor}         % colors
\usepackage{graphicx}       % figures (niv addition)
\usepackage{algorithm}      % algorithms (niv addition)
\usepackage[noend]{algpseudocode}  % algorithms (niv addition)
\usepackage{mathtools}      % math (niv addition)
\usepackage[capitalise]{cleveref}   % references (niv addition)
\usepackage{amsthm}         % math (niv addition)
\usepackage{bbm}            % math (shahar addition)
\usepackage{subcaption}     % figures (niv addition)
\usepackage{cleveref}       % figures (niv addition)
\usepackage[font=small]{caption}    % figures (niv addition)
\usepackage{wrapfig}
\algblockdefx{InParallel}{EndInParallel}{\textbf{do (1) and (2) in parallel:}}{}
\algnotext{EndInParallel}
\algblockdefx{NumFor}{EndNumFor}[2]{\textbf{(#1) for #2 do}}{}
\algnotext{EndNumFor}
\algblockdefx{Input}{EndInput}{\textbf{Input:}}{}
\algnotext{EndInput}

\newcommand{\abs}[1]{\left\vert #1 \right\vert}

\def\reals{{\mathcal R}}

\newcommand{\ignore}[1]{}

\def\reals{{\mathbb R}}

\def\bold0{\mathbf{0}}

\newcommand\E{\mathbb{E}}

\def\al#1\eal{\begin{align}#1\end{align}}
\def\als#1\eals{\begin{align*}#1\end{align*}}
\newcommand{\non}{\nonumber\\}
\newtheorem{theorem}{Theorem}[section]

\newtheorem{lemma}{Lemma}[section]

\newtheorem{assumption}{Assumption}[section]

\newcommand{\gnote}[1]{}
\newcommand{\snote}[1]{}
\newcommand{\MSnote}[1]{}
\newcommand{\dnote}[1]{}
\newcommand{\RB}[1]{}
\newcommand{\EM}[1]{}

\title{DropCompute: simple and more robust  distributed synchronous training via compute variance reduction}

\author{
Niv Giladi\textsuperscript{1,2}\thanks{Equal contribution}\quad
Shahar Gottlieb\textsuperscript{1,2}\footnotemark[1]\quad
Moran Shkolnik\textsuperscript{1,2}\quad
Asaf Karnieli\textsuperscript{2}\AND
Ron Banner\textsuperscript{2}\quad
Elad Hoffer\textsuperscript{2}\quad
Kfir Yehuda Levy\textsuperscript{1}\quad
Daniel Soudry\textsuperscript{1}\quad
\\[0.2cm]
\textsuperscript{\textbf{1}}Technion - Israel Institute of Technology\\
\textsuperscript{\textbf{2}}Habana-Labs
\\[0.2cm]
\small{\texttt{\{giladiniv, moranshkolnik, elad.hoffer, kfiryehud, daniel.soudry\}@gmail.com}}\\
\small{\texttt{\{sgottlieb, akarnieli, rbanner\}@habana.ai}}
}

\begin{document}

\maketitle

\begin{abstract}
\textbf{Background.} Distributed training is essential for large scale training of deep neural networks (DNNs). The dominant methods for large scale DNN training are synchronous (e.g. \textit{All-Reduce}), but these require waiting for all workers in each step. Thus, these methods are limited by the delays caused by straggling workers.\\  
\textbf{Results.} We study a typical scenario in which workers are straggling due to variability in compute time. We find an analytical relation between compute time properties and scalability limitations, caused by such straggling workers. With these findings, we propose a simple yet effective decentralized method to reduce the variation among workers and thus improve the robustness of synchronous training. This method can be integrated with the widely used \textit{All-Reduce}. Our findings are validated on large-scale training tasks using 200 Gaudi Accelerators. A reference implementation\footnote[2]{\url{https://github.com/paper-submissions/dropcompute}} is provided.
\end{abstract}

\section{Introduction}

Deep Neural Networks (DNNs) training continues to scale over size and computational footprint, as a result of a higher number of trainable parameters, wider and deeper models, and growing amounts of training data. As improvements in model quality (measured by test loss, for example) \citep{scalingnlp} lead over hardware capabilities \citep{hwlottery}, this scale-up translates into a need for a growing number of training devices working in tandem \citep{chowdhery2022palm}, turning distributed training to the standard approach for training DNNs on a large scale. 

\begin{figure}[th!]
  \centering
  \begin{subfigure}[b]{0.45\textwidth}
      \includegraphics[width=\textwidth]{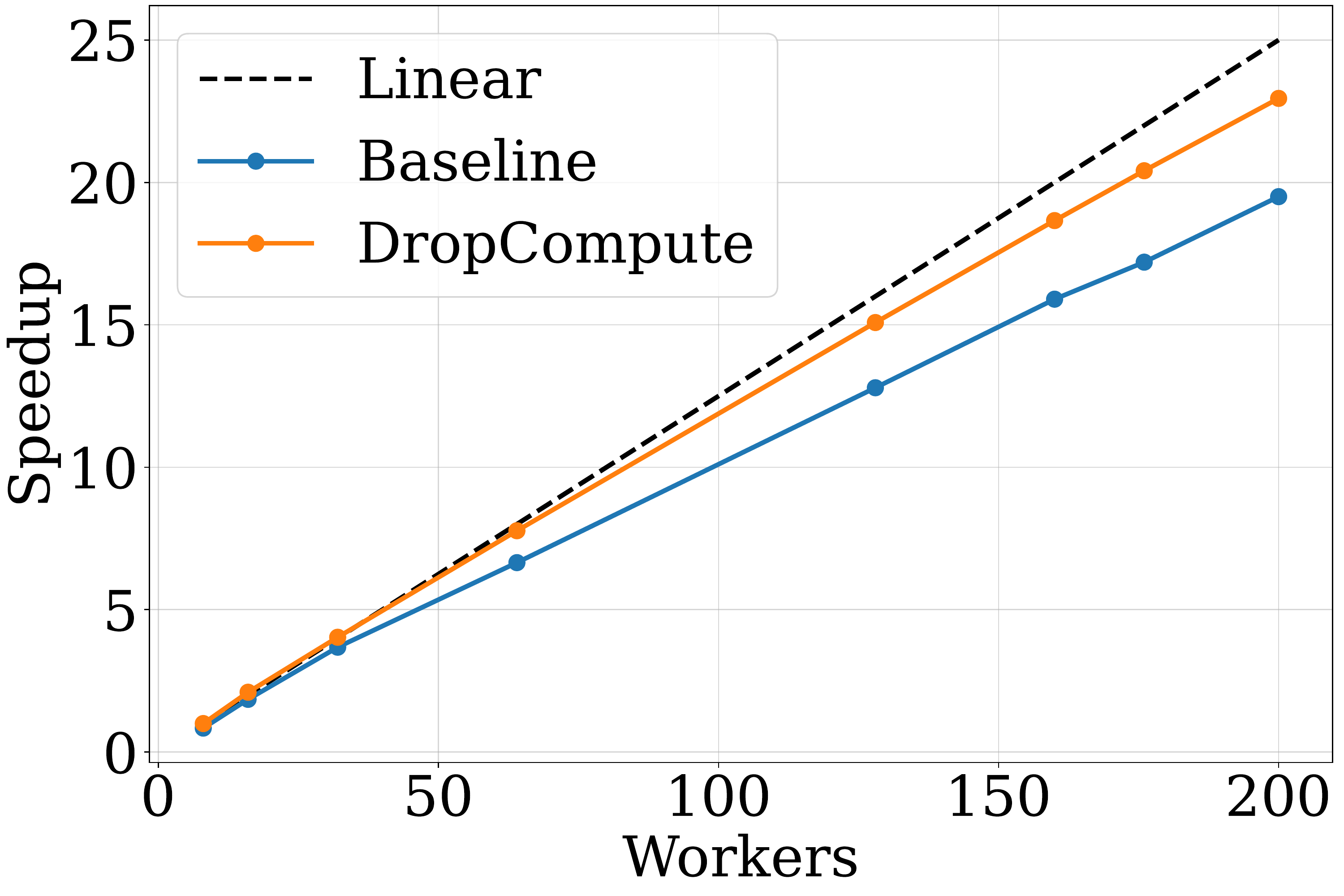}
  \end{subfigure}
  \hfill
  \begin{subfigure}[b]{0.45\textwidth}
      \includegraphics[width=\textwidth]{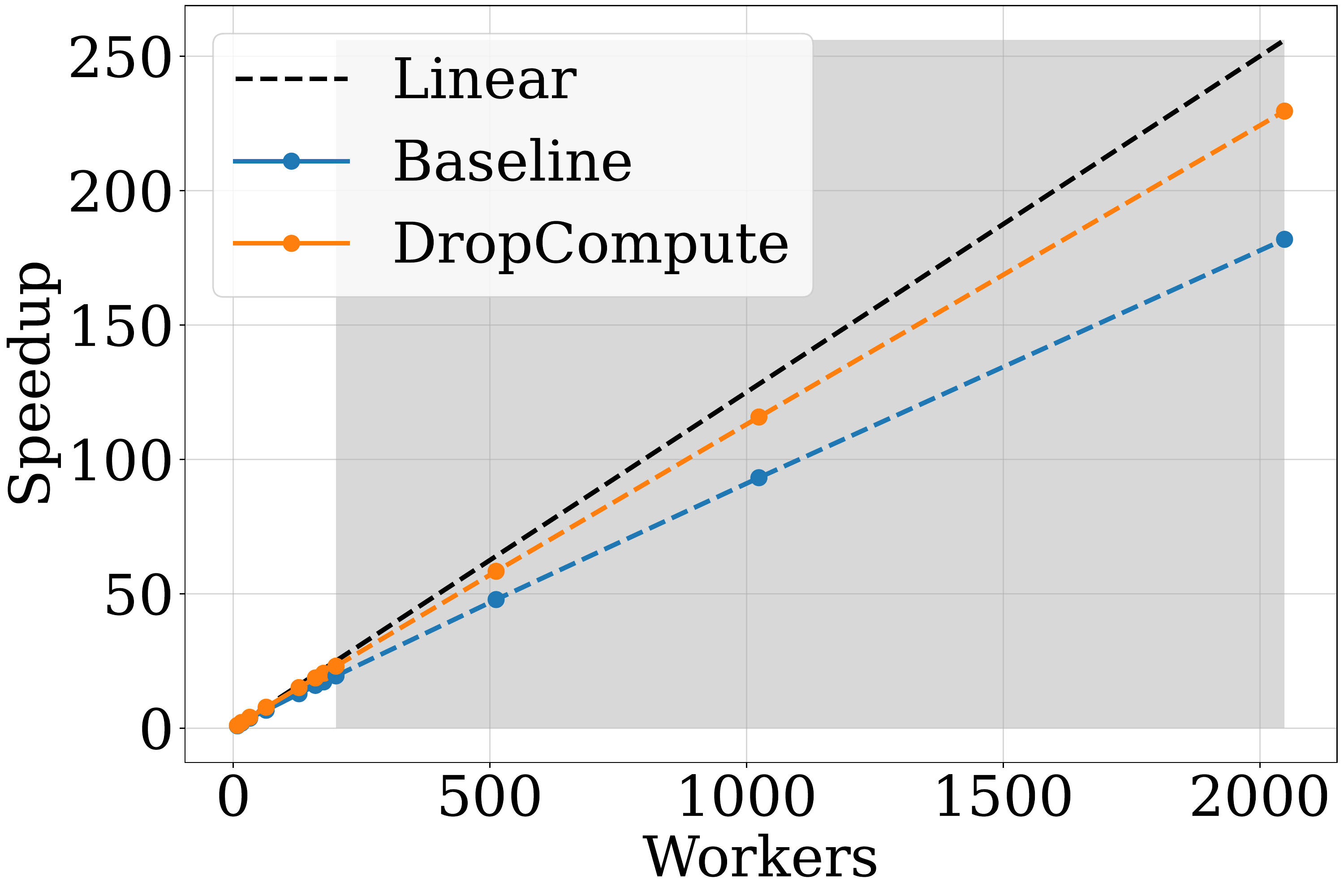}
  \end{subfigure}
  \caption{\textbf{\textit{DropCompute} improves robustness and scalability of synchronous training.} A scale graph, showcasing the proposed method runtime performance of synchronous training of a 1.5 billion parameter model with additive noise, simulating compute variance. The baseline represents existing synchronous training and the dashed black line is linear scaling. (left) Real measurements of up to 200 workers. (right) An extrapolation to 2048 workers using a theoretical estimation, also proposed in the paper. More details are provided in section \ref{sec:runtime_perf}.}
  \label{fig:abstract} \vspace{-2mm}
\end{figure}\vspace{-2mm}

Distributed training typically refers to three parallelism paradigms --- data parallel, model parallel and layer pipelining \citep{DemystifyingDistributed}. Several variants and hybrid solutions exist in modern implementations such as tensor parallel \citep{Megatron2021} and parameter sharding \citep{rajbhandari2020zero, rasley2020deepspeed}. These can be used separately or combined as they are orthogonal to each other. Mainly, data parallelism is straightforward, where the data is sharded among workers, and all workers share the same global model state. At each step, workers compute gradients locally and then aggregate them before taking an optimization step. 
When training synchronously, workers update their parameters in lockstep. This ensures that all workers hold a consensus on the same model and that gradients are averaged over all workers before being applied to the model. This approach is easy to implement and allows for good convergence properties, and correspondingly is the prevalent optimization method.

Although state-of-the-art models use synchronous optimization for training, synchronous methods scale poorly, as stragglers and communication overhead might severely deteriorate system utilization. Many resources are invested in alleviating these issues in large-scale training systems since even minor improvements can be worth hundreds of thousands of dollars. These issues are exacerbated as the required scale grows, even in homogeneous high-performance computing clusters \citep{petrini2003case, hoefler2010characterizing}. In this paper, we are interested in cases where significant computing variance between the workers exists. This includes (but is not limited to) straggling workers. 

For instance, certain learning tasks entail heterogeneity in the required computation of data, such as varying sentence lengths in language processing, or different image sizes and frame numbers in computer vision. In addition, recent state-of-the-art models use all three parallelism paradigms (data (DP), tensor (TP), and pipeline (PP) parallelism), thus each data parallel node is a set of processing units (accelerators) communicating between them (via TP and PP) to calculate the model gradients collectively. This could potentially intensify compute variance between data parallel workers. Moreover, sub-optimal hardware systems can also lead to straggling workers. Although some sources of heterogeneity can be mitigated by non-trivial engineering work, they still persist as challenges that can introduce compute variance and substantial performance drops. This is further discussed in appendix \ref{appendix:discussion}.
%due to changes over time in hardware state or load might occur and cause variance in compute and specifically lead to straggling workers.
%Many of these issues can be handled by non-trivial engineering work: Faulty hardware can be tested regularly and replaced, host overhead can be avoided by latency hiding and optimization of user script, and inefficient load balancing can be handled per workload by various tricks such as sample padding and packing. However, each of these issues in itself is a single point of failure --― which, if triggered, will cause a substantial slowdown in large-scale systems. As compute variance grows, the utilization deteriorates, such that more workers remain idle waiting for slower workers to complete calculating their gradients \cite{chen2016revisiting,chen2019round,ji2022ep4ddl}.
%Moreover, the effect of stragglers and compute variance on the training speed is expected to get worse as the distributed scale increases.

In this paper, we suggest a simple, yet effective method called \textit{DropCompute} to improve the robustness and scalability of synchronous optimization in the face of compute variance. We model the compute time as a random variable and show that under reasonable assumptions, a tail of straggling workers slows down the system at a rate that is not proportional to the contributed compute by these straggling workers.
We harness the gradient accumulation method widely used in Large Language Models (LLMs) \citep{ott-etal-2018-scaling, roberta} to implement the method in a few lines of code on a relevant large-scale learning task.

The contributions of our work include:
\begin{itemize}
    \item \textit{DropCompute}: a novel, decentralized method to better handle heterogeneity or stragglers without additional hyper-parameters. \textit{DropCompute} is hardware and framework agnostic, runs on top of existing optimizers, and can also be combined with other methods that improve other aspects of robustness such as communication overhead.
    \item A theoretical convergence proof of SGD with stochastic batch size, as in \textit{DropCompute}.
    \item A theoretical runtime analysis on standard synchronous training and the proposed method. We find an approximation of the expected speedup using \textit{DropCompute}, and show this speedup goes to infinity as the number of workers grows. 
    \item An empirical evaluation of the proposed method on a relevant large scale task, using up to 200 accelerators connected with high bandwidth communication. For example, when using \textit{DropCompute} before and after system optimizations (for both software and hardware) we show accelerations of 18\% and 5\%, respectively.
    \end{itemize}

% We model the compute time as a random variable and show that under reasonable assumptions, a tail of straggling workers slows down the system at a rate that is not proportional to the contributed compute by these straggling workers.
% Our method automatically identifiesת in a decentralized fashion when to drop a small portion of the compute to maintain high utilization.
% We harness the gradient accumulation method widely used in Large Language Models (LLMs) \citep{ott-etal-2018-scaling, roberta} to implement the method in a few lines of code on a relevant large scale learning task. We demonstrate both actual speedup gains (e.g., Figure \ref{fig:abstract}) and reduced compute variance when \textit{DropCompute} is applied. \textit{DropCompute} is hardware and framework agnostic, runs on top of existing optimizers, and can also be combined with other methods that improve other aspects of robustness such as communication overhead.

\section{Related Work}

The challenge of training deep neural networks on a large scale has been extensively explored. With rapidly growing models and data sizes, numerous works tackled the weaknesses in synchronous DNN training on a large scale and suggested methods to alleviate these weaknesses.

\textbf{Redundancy methods.} This line of work addresses the straggling worker problem using a redundancy mechanism. Redundant workers or redundant data are used such that straggling workers will not slow down the entire system \citep{chen2016revisiting, bitar2020stochastic}. These methods provide better robustness to synchronous training, even in the event of a complete failure of a subset of the workers or considerable communication slowdown. However, the robustness is limited by the redundancy factor, and more generally, more compute resources are required and full utilization cannot be achieved. In addition, some coordination method is required to keep the training synchronous, i.e., keeping consensus between the model replicas of the workers. In particular, \citet{chen2016revisiting, bitar2020stochastic} use a centralized approach of a parameter server to determine which workers are left out at each iteration. Modern large-scale systems use decentralized variants of \textit{All-Reduce} \citep{baidu_ringreduce, patarasuk2009bandwidth}, so it is not trivial to determine which workers should be considered at each step, given that each worker can see a different subset of straggling workers. Moreover, combining redundancy with communication primitive collectives (e.g., \textit{All-Reduce}) requires adaptation to existing underlying frameworks \citep{sanders2019sequential}.

\textbf{Asynchronous optimization.} Another approach is introducing asynchrony to the optimization. Asynchronous training is inherently more scalable than synchronous training by being robust to all kinds of workers and communication faults. This includes periodic synchronization by exchanging parameters every $\tau$ optimization steps \citep{stich2018local, Lin2020Don't, wang2021cooperative, zhang2015deep, Wang2020SlowMo, li2020federated, li2020breaking}, approximate distributed averaging where each worker communicates with a subset of workers each step \citep{jiang2017collaborative, lian2017can, assran2019stochastic, yang2020mitigating}, and many more. These works provide better scale-up properties and improve time performance. The main drawback is asynchronous optimization itself. In practice, the convergence is less stable, and more optimization steps are needed to generalize as well as synchronous training. In addition, hyperparameters should be chosen more precisely to guarantee convergence and generalization properties \citep{giladi2019stability, mitliagkas2016asynchrony}. Due to these issues, asynchronous methods are less commonly used on a large scale.

\textbf{Sparse and compressed communication.} Alternatively, several works addressed only the communication overhead. A common technique is to reduce the amount of data exchanged by the workers at each step. This can be done by gradient pruning \citep{xu2021deepreduce}, gradient compression \citep{1bitsgd, chen2020scalecom, 1bitadam} or low-rank approximation \citep{vogels2019powersgd}. These works reduce the communication overhead in a deterministic form while ignoring any compute variance across processing workers. This makes these works orthogonal to ours and potentially can be combined with our method.

\section {Reducing Compute Variance} \label{section:distributed_training}
This paper proposes a method called \textit{DropCompute} that improves the robustness of synchronous training by reducing compute variance. First, we describe the vanilla synchronous training framework. Then, we introduce the proposed approach.

\subsection{Problem setup} \label{subsection:problem_setup}
We start with a model formulation for data-parallel synchronous SGD training with $N$ workers, where each worker holds a replica of the model parameters $\theta$. In parallelism paradigms that combine TP or PP with DP, $N$ represents the count of data-parallel workers, not the total number of workers involved in training. Given a dataset $\mathcal{D}$, we are interested in minimizing the empirical loss
\begin{equation*}
    \mathcal{L}(\mathcal{D},\theta)=\frac{1}{\abs{\mathcal{D}}}\sum_{z\in\mathcal{D}}\ell(z,\theta),
\end{equation*}
where $\ell(z,\theta)$ is the loss with respect to data-point $z$ and the model parameters $\theta$.
At each step, the workers calculate gradients based on a local batch and then aggregate the gradients before taking an optimization step. An equivalent strategy would be that the workers aggregate the parameters after taking an optimization step based on the local gradients. The aggregation is done in a decentralized fashion, such as \textit{AllReduce}. 
%The optimal parameters are determined by minimizing the empirical loss $\ell(X, Y, \theta)$, with $y$ representing the labels.

We also consider the use of gradient accumulation to increase the size of the local batch. Gradient accumulation breaks down each worker's local batch into $M$ micro-batches for computation, which enables reaching a large global batch size beyond hardware capacity. This is common practice in training of LLM on large scale where a substantial number of accumulations are being utilized \citep{smith2022using, nvidiaMlperfGPT3}.
% such as GPT3 where 24 accumulations are used [] and NLG 530B with 160 accumulations []
In each iteration $i$ and accumulation $m$, the gradients of the loss function with respect to the model parameters are computed, denoted by 
\begin{equation*}
    g_{n}^{(m)}(\theta_i) = \nabla \mathcal{L}(\mathcal{D}_{i,n}^{(m)}, \theta_i)\,,
\end{equation*}
where $\mathcal{D}_{i,n}^{(m)}\subseteq \mathcal{D}$ is the micro-batch $m$ of worker $n$, sampled without replacement from $\mathcal{D}$, and we assume a constant micro-batch size $|\mathcal{D}_{i,n}^{m}|$.
% Respectively, the local batch is $\mathcal{D}_{i,n}=\{\mathcal{D}_{i,n}^{(1)}, \mathcal{D}_{i,n}^{(2)}, \cdots, \mathcal{D}_{i,n}^{(M)}\}$, and the global batch is $\mathcal{D}_i=\{\mathcal{D}_{i,1}, \mathcal{D}_{i,2}, \cdots, \mathcal{D}_{i,N}\}$. We assume the global batch size remain constant $b=|\mathcal{D}_i|$.
The gradients accumulated by worker $n$ at step $i$ are
\begin{equation*}
    g_{n}(\theta_i) = \frac{1}{M} \sum_{m=1}^{M} g_{n}^{(m)}(\theta_i) \,.
\end{equation*}
Finally, the workers aggregate and average the computed gradients to update the model parameters
\begin{equation}
    \theta_{i+1} = \theta_i - \eta g(\theta_i);\quad  g(\theta_i) = \frac{1}{N} \sum_{n=1}^{N} g_{n}(\theta_i)    \,,
    \label{eq:theta}
\end{equation}
where $\eta$ is the learning rate. 
Equation \ref{eq:theta} requires all workers to receive all gradients before updating the parameters. This communication restriction is what makes the training synchronous and ensures the model parameters are the same for all workers. However, due to this restriction, the slowest worker at each step dictates the iteration time. More generally, any variation in computation time among workers will lead to workers with idle time. Therefore, to improve the efficiency and robustness of the synchronous training process, we need to reduce the computation time variance among workers (namely, the compute variance). 
% In section \ref{section:analysis} we model the iteration time as a random process to estimate the effects of compute variance on training time.

\subsection{Our method: \textit{DropCompute}} \label{sec:our_method}
To mitigate the effect of compute variance and to improve the robustness of synchronous training, we propose a simple yet effective method called \textit{DropCompute}. \textit{DropCompute} reduces the compute variance by introducing a \textit{compute threshold} after which workers have to stop calculating local gradients and move directly to the communication phase, i.e., Equation \ref{eq:theta}.

In each iteration $i$, each worker $n$ measures the time while calculating its local batch, and compares it against a given threshold $\tau$, which is set as described in section \ref{section:automatic_threshold}. 
If that time exceeds the \textit{compute threshold}, the worker stops and sends the gradients it has calculated so far. The rest of the training remains unchanged, thus synchronous training is intact.
The pseudo-code of \textit{DropCompute} is shown in Algorithm \ref{algorithm:drop_compute}. 
%This method maintains decentralized synchronous training
This method maintains a decentralized approach
while improving its robustness to compute variance and stragglers in particular. Since our method drops samples, the batch size is no longer deterministic. Specifically, the total batch size (the total number of samples used in one iteration by all workers) can be smaller than the maximal total batch size  $b_{\max}=NM|\mathcal{D}_{i,n}^{(m)}|$. Nevertheless, we show both theoretically and empirically that this makes no difference to convergence or generalization when the number of computed samples remains the same. Next, we analyze this method in section \ref{section:analysis} and evaluate it in section \ref{section:experiments}.

\begin{algorithm}[h!]
    \caption{DropCompute on worker $n$ at iteration $i$}
    \label{algorithm:drop_compute}
    \begin{algorithmic}[1]
        \Input  \,\,model parameters $\theta_i$; total number of micro-batches $M$; compute threshold $\tau$
            \State local batch data $\mathcal{D}_{i,n}=\{\mathcal{D}_{i,n}^{(1)}, \mathcal{D}_{i,n}^{(2)}, \cdots, \mathcal{D}_{i,n}^{(M)}\}$;  step time $T_n$ to compute local batch $\mathcal{D}_{i,n}$
        \EndInput
        \State \textbf{Initialize} step time $T_n=0$  and accumulated gradients $g_{n}(\theta_i)=0$       
        \InParallel
            \NumFor{1}{$m=1, \dots, M$}
                \State \hskip1.5em $g_{n}^{(m)}(\theta_i) = \nabla \mathcal{L}(\mathcal{D}_{i,n}^{(m)}, \theta_i)$ \Comment{Compute gradient}
                \State \hskip1.5em $g_{n}(\theta_i) \leftarrow g_{n}(\theta_i) + g_{n}^{(m)}(\theta_i)/M$ \Comment{Accumulate gradients (atomic)}
            \EndNumFor
            \State \textbf{(2) wait for $T_n>\tau$ and break for loop (1)}
            
        \EndInParallel
        \State \textbf{Output: }$g_{n}(\theta_i)$ \Comment{\textit{AllReduce}}
    \end{algorithmic}
\end{algorithm}\vspace{-2mm}

\section{Method Analysis} \label{section:analysis}

After establishing the notion of compute variance and the formulation of the proposed method, we analyze synchronous training and the potential value of the proposed method on the training time performance. To do so, we start with convergence guarantees when using \textit{DropCompute}. Then, we theoretically analyze the computation time with synchronous training and when using \textit{DropCompute}. Through this analysis, we estimate the potential speedup over vanilla synchronous training.

\subsection{Convergence analysis of \textit{DropCompute}} \label{subsec:convergence}
Using \textit{DropCompute}, the batch size is no longer fixed, but a random variable, which can also potentially depends on the data samples. To the best of our knowledge, this setting is somewhat different from existing convergence guarantees. Therefore, in this section, we provide convergence guarantees for \textit{DropCompute} by analyzing the convergence of SGD with stochastic batch size.

\begin{assumption} \label{assumption:convergence}
    Following the notations in section \ref{subsection:problem_setup}, consider a possibly non-convex smooth loss function $\mathcal{L}(\mathcal{D},\theta)$, with a global minimum $\theta^*\in\reals^d$,  and the following (commonly used) assumptions
    \begin{enumerate}
        \item \textbf{L-smooth}: All functions $\mathcal{L}(\cdot,\theta)$ are with L-Lipschitzian gradients.
        \item \textbf{Unbiased estimation}: Each $\nabla\ell(z,\theta_i) ,z\in\mathcal{D}$ is an unbiased estimate of $\nabla \mathcal{L}(\mathcal{D},\theta_i)$ which is the true (full batch) gradient  at $\theta_i$\footnote{This is easily satisfied when all workers can access all data.}. Namely,
        $\forall i: \E[\nabla\ell(z,\theta_i)\vert \theta_i] = \nabla\mathcal{L}(\mathcal{D},\theta_i)\,.$
        \item \textbf{Bounded variance}: The variance of the stochastic gradient is bounded by a constant $\sigma$,
        $$\forall i:  \E[\|\nabla\ell(z,\theta_i)-\nabla\mathcal{L}(\mathcal{D},\theta_i)\|^2\vert \theta_i] \leq \sigma^2\,.$$
    \end{enumerate}
\end{assumption}

\begin{theorem} \label{theorem:non_convex} 
    Under the above assumption, applying SGD with \textit{DropCompute} (Algorithm \ref{algorithm:drop_compute}), ensures
    \begin{equation} \label{eq:NonCvxSGD_GUARANTEES_final}
        \mathbb{E}\|\nabla\mathcal{L}(\mathcal{D},\bar{\theta})\|^2 \leq \frac{2L b_{\max} (\mathcal{L}(\mathcal{D},\theta_1) - \mathcal{L}(\mathcal{D},\theta^*))}{K} + \frac{2\sigma\sqrt{L(\mathcal{L}(\mathcal{D},\theta_1) - \mathcal{L}(\mathcal{D},\theta^*)) }}{\sqrt{K}}~,
    \end{equation}
    where $b_{\max}$ is the maximal total batch size, $K$ is the total number of samples that are used throughout the training, $\theta_1$ is the initial model, and $\bar{\theta}$ is a random sample of $\theta_i$ from the trajectory obtained by Algorithm \ref{algorithm:drop_compute}, where $i$ is selected with probability proportional to the total batch size at iteration $i$. The expectation is with respect to the randomization introduced due to sampling from $\mathcal{D}$ throughout the optimization process and with respect to choosing $\bar{\theta}$.
\end{theorem}
    
The bound in Theorem~\ref{theorem:non_convex} is similar to existing fixed batch-size guarantees~\citep{dekel2012optimal}. The second term in the bound does not degrade with the batch sizes, and behaves like $O(1/\sqrt{K})$, while the first term behaves like $O(b_{\max} /{K})$. This implies that as long as $b_{\max}  \leq O(\sqrt{K})$, the second term is dominant and we are in the regime of linear speedup. This shows we can attain a linear speedup despite using changing batch sizes, as long as the maximal batch size is bounded.
    
Similarly, in Theorem \ref{theorem:convex_case} in appendix \ref{appendix:convex_proof} we show that the loss itself converges, in the convex case along with a proof. The proof of Theorem \ref{theorem:non_convex} is provided in appendix \ref{appendix:non_convex_proof}. Lastly, we discuss the impact of the stochastic batch size on generalization in appendix \ref{app:generalization_discussion}.

\subsection{Iteration time in standard synchronous distributed training} \label{sec:analysis_iteration_time}

We start with finding a closed-form expression for the cumulative distribution function (CDF) of the iteration time, denoted as $T$, defined as
\begin{equation*}
    T=\max(T_1,T_2,...,T_N)\,,
\end{equation*}
where $T_n$ represents the time taken by worker $n$ to compute its local batch, which follows some cumulative probability function
$F_{T_n}(x)=\mathbb{P}(T_n<x),$
which we assume is independent for each worker. 
Let $F_T(x)$ represent the cumulative distribution function and $f_T(x)$ represent the probability density function of the maximum iteration time $T$. The relation between $F_T(x)$ and $F_{T_n}(x)$ is
\begin{equation*}
    F_T(x) = \mathbb{P}\left(\max \left(T_1, \ldots, T_N\right) \leq x\right) =\prod_{n=1}^N F_{T_n}(x)\,.
\end{equation*}
Differentiating with respect to $x$ and applying the chain rule gives:
\begin{equation*}
    f_T(x)=\frac{d F_T}{d x}(x)=\sum_{n=1}^N f_{T_n}(x)\prod_{n'\neq n}^{N}F_{T_{n'}}(x) \,.
\end{equation*}

%When the time distributions for all workers' iterations are identical and independent (i.i.d), this simplifies to the well-known formula:

In the special case where all of the workers' iteration time distributions are identically and independently distributed (i.i.d.), this reduces to the well-known formula:
\begin{equation}
    \label{eq:latency_pdf}
    f_T(x)=N \cdot f_{T_n}(x) \cdot F_{T_n}(x)^{N-1}
\end{equation}
%where $f_{T_n}(x)$ is the probability density function and $F_{T_n}(x)$ is the cumulative distribution function of $T_n$.
If the iteration time of each worker is distributed normally ($\sim \mathcal{N}(\mu, \sigma^2)$), the expected value of $T$ can be approximated as shown by \citet{bailey2014pseudomathematics}: 
\begin{equation}
\mathbb{E}(T) \approx \sigma\cdot \left( (1-\gamma)\cdot \Phi^{-1}\left(1-\frac{1}{N} \right) + \gamma\cdot \Phi^{-1}\left(1-\frac{1}{e\cdot N}\right)\right) + \mu
\label{eq:estimate_T}
\end{equation}
where $\Phi$ is the CDF of the standard normal distribution, and $\gamma$ is the euler-mascheroni constant. Asymptotically the total iteration time is $\mathbb{E}[T]=\Theta(\sqrt{\text{log} N})$.
When the number of micro-batches $M\gg1$, we can make a similar approximation under Central Limit Theorem (CLT) conditions. 
More details are in appendix \ref{appendix:speedup_analysis}.
%When workers' iteration time is not (i.i.d) CLT approximations may still hold as is the case in Figure \ref{fig:experiement_vs_theory}.
\subsection{Iteration time and number of micro-batches with \textit{DropCompute}}

When using \textit{DropCompute} with a constant threshold $\tau$, each worker is preempted at $\tilde{T}_n=\min\left\{\tau,T_n\right\}$ and joins the \textit{AllReduce}. Therefore, the total iteration time with \textit{DropCompute} is
$$ \tilde{T} + T^c = \min\left\{\tau,T\right\}+T^c \, ,$$
where $T^c$ is a serial latency present in each iteration, which includes the \textit{AllReduce} step.
This upper limit serves to clip extreme values of $T_n$, effectively constraining the range of potential outcomes for $\tilde{T}$. As a result, the compute time variability decreases, leading to a narrower distribution and enhanced compute efficiency. These effects are illustrated in Figure \ref{fig:experiement_vs_theory}.

As a consequence of preempting each worker at $\tilde{T}_n$, the number of micro-batches computed in each step varies. Denote as $t_n^{(m)}$  the compute latency of a single micro-batch $m$ for worker $n$, and $T_n^{(m)}=\sum_{j=1}^mt_n^{(j)}$. We can define the average number of micro-batches computed by each worker before reaching threshold $\tau$ as
$$\tilde{M}=\frac{1}{N}\sum_{n=1}^N\sum_{m=1}^M \left\{
                            \begin{array}{lr}
                                1, & \text{if } T_n^{(m)}<\tau\\
                                0, & \text{otherwise }
                            \end{array}\right\}~.$$
Under CLT conditions, the expected value for $\tilde{M}$ can be approximated in a closed form:
\begin{equation}
    \mathbb{E}[\tilde{M}] \approx \sum_{m=1}^M \Phi\left(\frac{\tau - m\cdot \mu}{\sqrt{m\cdot \sigma^2}} \right)
    \label{eq:estimate_M}
\end{equation} 
where $\mu, \sigma^2$ are the mean and variance for a single micro-batch $t_n^{(m)}$ compute latency, and $\Phi$ is the CDF of the standard normal distribution. 
This approximation closely fits the real value of $\tilde{M}$ and can be used to analyze the expected gain from \textit{DropCompute}. More details in appendix \ref{appendix:speedup_analysis}.
\begin{figure}[h!]
    \centering
    \begin{subfigure}[]{0.48\textwidth}
        \includegraphics[width=\textwidth]{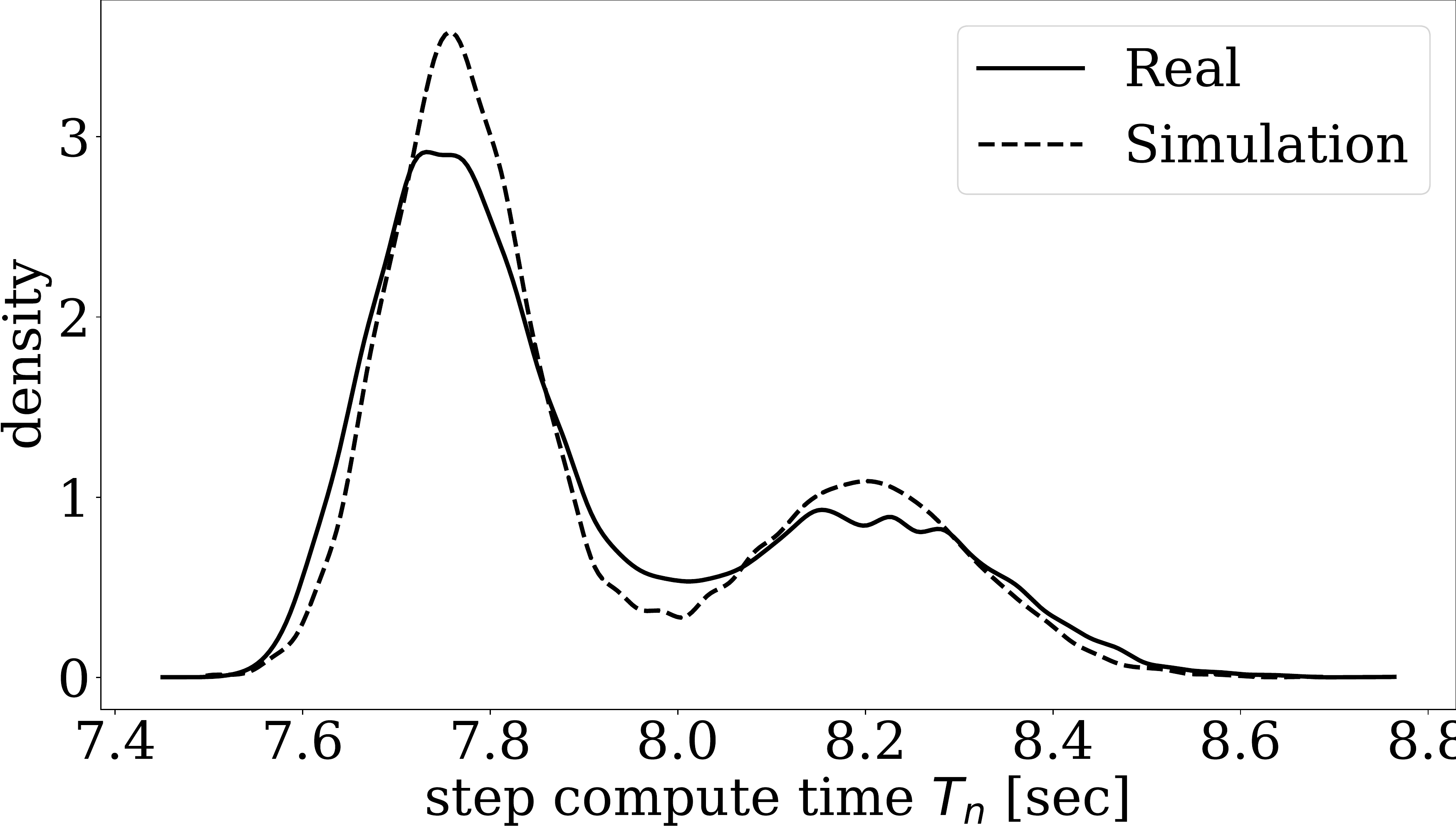}
    \end{subfigure}
    \hfill
    \begin{subfigure}[]{0.48\textwidth}
        \includegraphics[width=\textwidth]{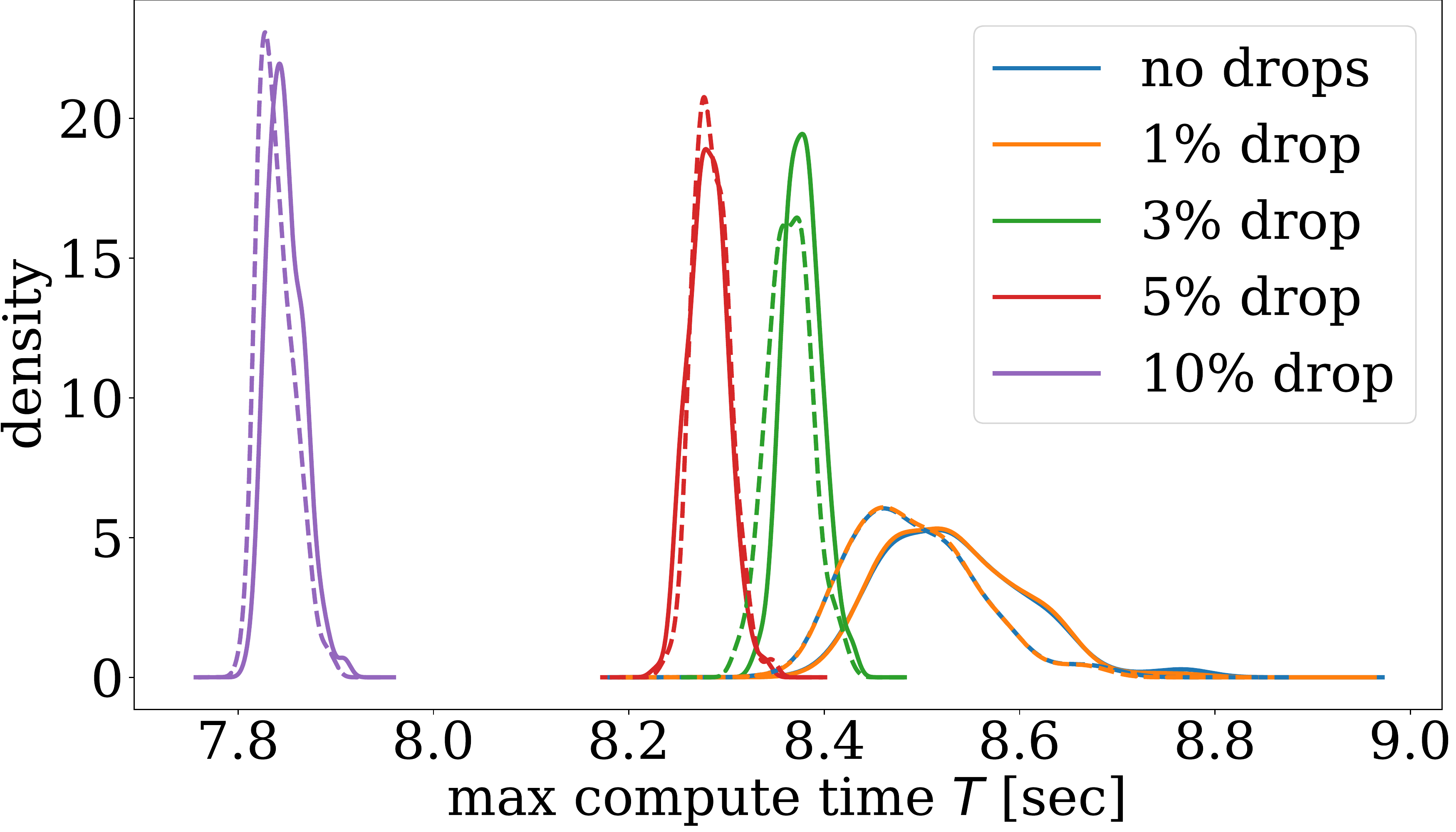}
    \end{subfigure}
    \caption{\textbf{Reduction in variance and mean iteration time using \textit{DropCompute}.} Iteration time distribution of 200 workers using \textit{DropCompute} on BERT 1.5B. (left) Step time $T_n$ distribution of all workers, without \textit{DropCompute}. (right) Maximum iteration time $T$, across all workers, using \textit{DropCompute}, with different drop rates. 
    The dashed `simulation' distribution is generated by drawing $T_n$ randomly from an independent normal distribution separately for each worker, using the empiric mean and variance of that worker.}
    \label{fig:experiement_vs_theory}\vspace{-2mm}
\end{figure}

\subsection{Choosing the threshold} \label{section:automatic_threshold}
The throughput of the system can be seen as the number of micro-batches computed per second. For $N$ workers, this can be written as $NM/(T+T^c)$. To evaluate the effectivness of \textit{DropCompute}, we consider the difference in throughput between the baseline and when using \textit{DropCompute}. Doing so, we can define the effective speedup for $\tau$ as:
\begin{equation}
    S_{\mathrm{eff}}(\tau)= \frac{\text{DropCompute  Throughput}}{\text{Baseline Throughput}} = \frac{N\tilde{M}/ (\min\left\{\tau,T\right\}+T^c)}{NM/(T+T^c)} = \frac{\tilde{M}(T+T^c)}{M(\min\left\{\tau,T\right\}+T^c)}
    \label{eq:auto_eff}
\end{equation}
Given the statistical characteristics of the training setting, it is possible to estimate analytically the expected value of the effective speedup $\mathbb{E}[S_\mathrm{eff}(\tau)]$ by using Equations \ref{eq:estimate_M} and \ref{eq:estimate_T}.
Moreover, when plugging in the asymptotic form of $\mathbb{E}[T]$, we find the expected speedup increases to infinity with $N$
$$ \mathbb{E}[T]=\Theta(\sqrt{\text{log} N}) \;\; \Rightarrow \;\; \mathbb{E}[S_\mathrm{eff}(\tau)](N)\underset{N\to\infty}{\longrightarrow}\infty$$
%An even better estimation can be done when $\mathbb{E}[T]$ is given, using the Algorithm described in Appendix \ref{appendix:auto_threshold}. These estimations are demonstrated in figures \ref{fig:estimate_s_eff}.

As shown in figure \ref{fig:estimate_s_eff_b}, Equation \ref{eq:estimate_T} is less accurate when samples deviate from a normal distribution.
\begin{figure}[h!]
    \centering
    \begin{subfigure}[b]{0.32\textwidth}
        \caption{}
        \label{fig:estimate_s_eff_a}
        \includegraphics[width=\textwidth]{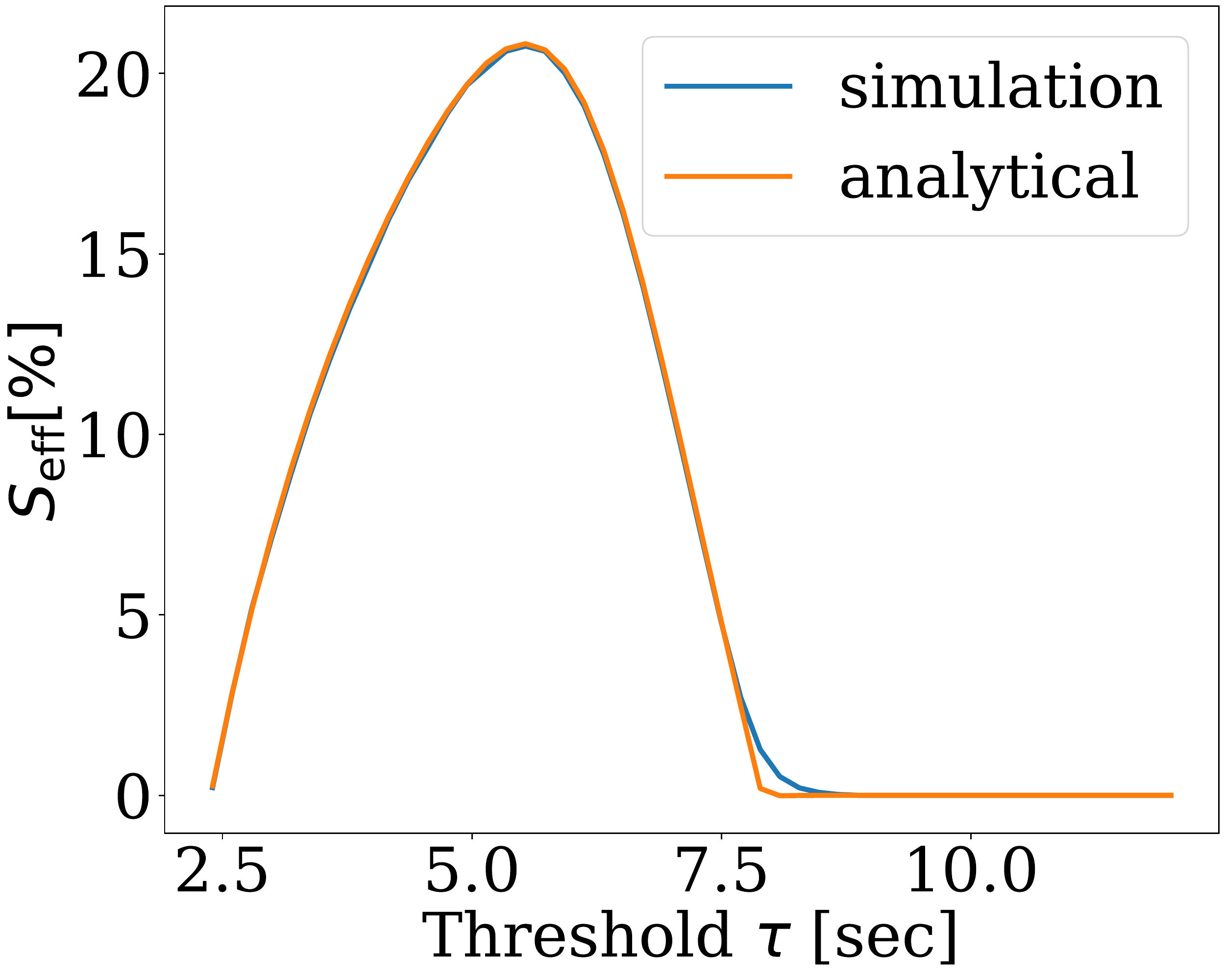}
    \end{subfigure}
    \hfill
    \begin{subfigure}[b]{0.32\textwidth}
        \caption{}
        \label{fig:estimate_s_eff_b}
        \includegraphics[width=\textwidth]{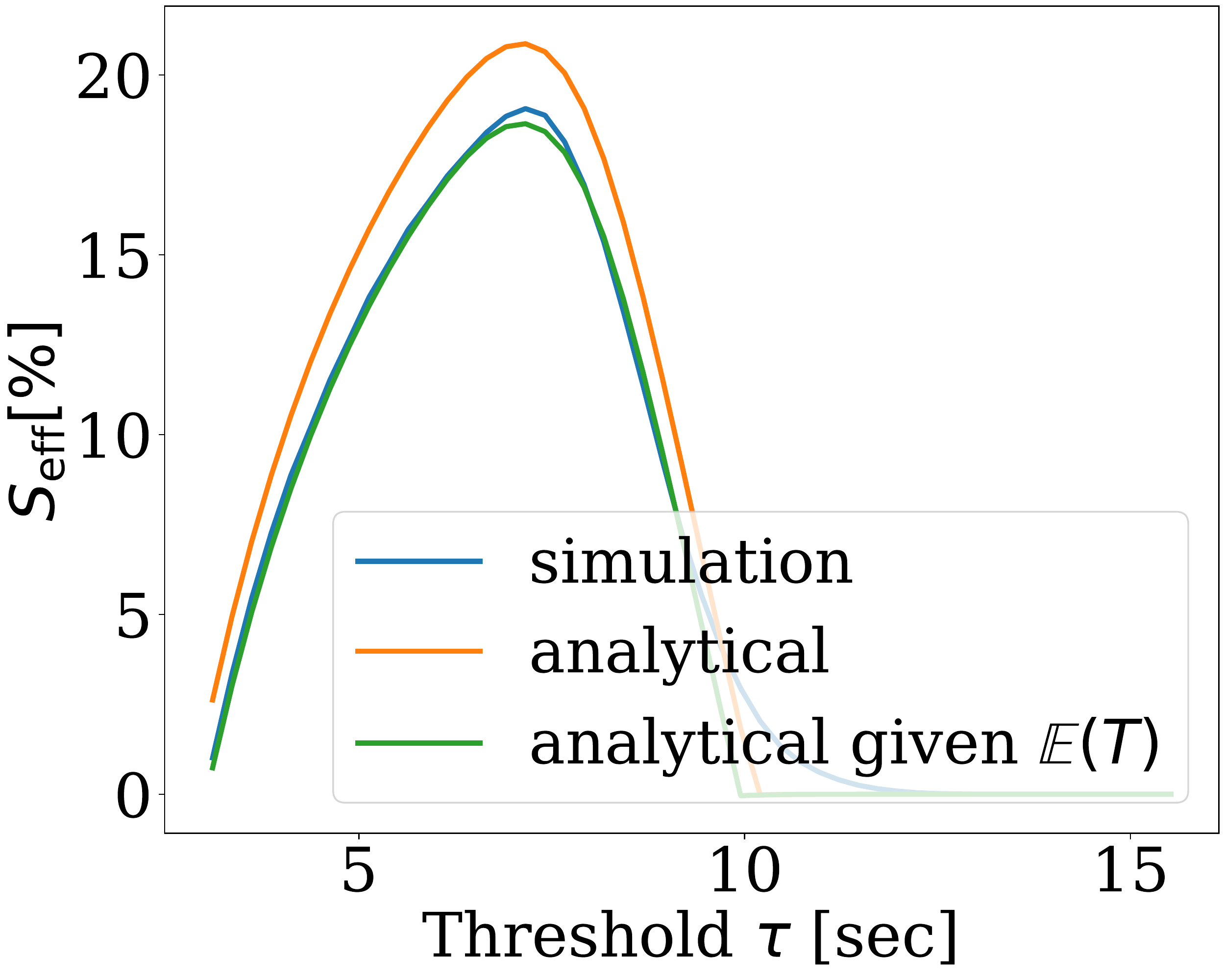}
    \end{subfigure}
    \hfill
    \begin{subfigure}[b]{0.32\textwidth}
        \centering
        \caption{}
        \label{threshold_choosing}
        \includegraphics[width=\textwidth]{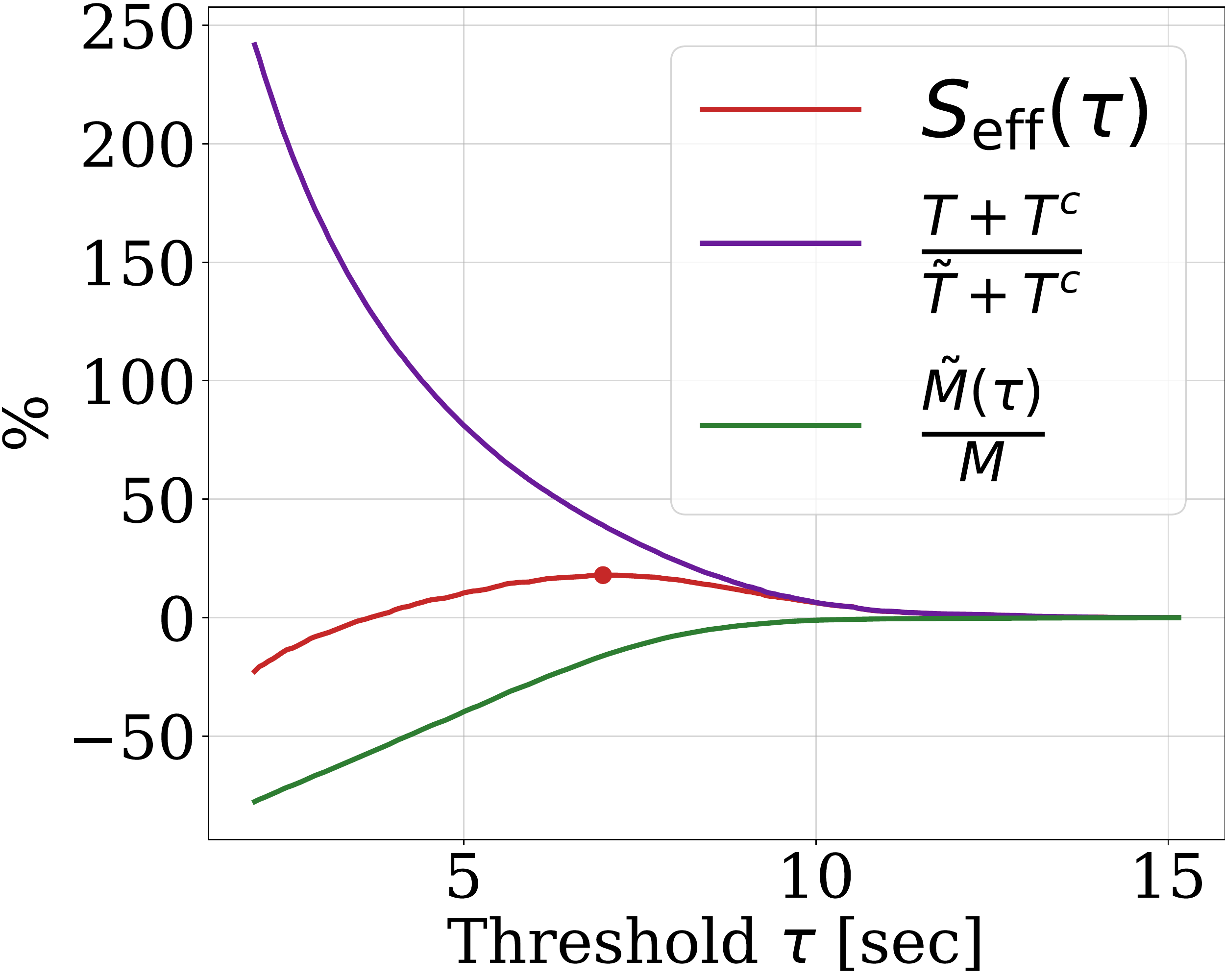}
    \end{subfigure} 
    \caption{\textbf{Statistical characteristics of the micro-batch computation latency $t_n^{(m)}$ can provide a reliable estimate of the effective speedup $S_\mathrm{eff}$}. The graphs depict $S_\mathrm{eff}(\tau)$ based on Equation \ref{eq:auto_eff}. Calculations rely on samples of $t_n^{(m)}$ to calculate $\tilde{M}$ and $T$, before plugging them into Equation \ref{eq:auto_eff}. In the `simulation' curves, we directly use the samples to calculate $\tilde{M}$ and $T$. For the `analytical' curves, only the mean and variance of $t_n^{(m)}$ are used to approximate $\tilde{M}$ and $T$ using Equations \ref{eq:estimate_M} and \ref{eq:estimate_T}, respectively. For `analytical given $\mathbb{E}[T]$', $\tilde{M}$ is approximated using Equations \ref{eq:estimate_M}, but $\mathbb{E}[T]$ is calculated directly from the samples. More details in appendix \ref{appendix:speedup_analysis}. Panels: (a) The $t_n^{(m)}$ samples follow a normal distribution. (b) samples are taken from BERT1.5B pre-training with simulated delay as described in Section \ref{sec:runtime_perf}. (c) \textbf{The optimal compute threshold can be found automatically.} The effective speedup, micro-batch completion rate, and step speedup as a function of $\tau$ in a simulated delay environment.}
    \label{fig:estimate_s_eff}
\end{figure}
To find the optimal compute threshold, we synchronize the empirical distribution of micro-batch compute latency between all workers after a few iterations. Given this distribution, we find $T$ and $\tilde{M}$, and search in a decentralized way for the threshold $\tau^*$ that maximizes the effective speedup $S_{\mathrm{eff}}(\tau)$  defined in Equation \ref{eq:auto_eff}. 
Overall, the cost of synchronizing the empirical distribution and finding $\tau^*$ is negligible, compared to a full training session, because it happens only once in a training session.
%\begin{equation*}
%    \tau^*=\mathrm{argmax}_\tau S_{\mathrm{eff}}(\tau)
%\end{equation*}
Lowering the threshold leads to reduced compute time but higher compute drop rates. Figure \ref{threshold_choosing} highlights this trade-off and the optimal $\tau^*$ is marked.

%In workloads with larger compute variance we expect lower threshold values and higher drop rates. 
%This method is demonstrated in figure \ref{threshold_choosing}, where we automatically find such optimal compute threshold in an experiment with  BERT1.5B, 200 workers (details are provided in section \ref{section:experiments}). A detailed pseudo-code for automatically choosing of compute threshold can be found in appendix \ref{appendix:auto_threshold}
% \begin{figure}[h!]
%     \centering
%     \includegraphics[width=0.45\textwidth]{figures/optimal_threshold.pdf}
%     \caption{\textbf{The optimal compute threshold can be found automatically.} The effective speedup, micro-batch completion rate, and step speedup as a function of compute threshold in a simulated delay environment. The workers exchange empirical distributions in the first few iterations so that each worker calculates the effective speedup, micro-batch completion rate and step speedup deterministically in a decentralized way based on Equation \ref{eq:auto_eff}.}
%     \label{threshold_choosing}
% \end{figure}

\subsection{Compensating for dropped samples} \label{section:dropped_samples_compensation}
% The effective speedup metric $S_\mathrm{eff}$ takes into account dropped samples as a slowdown in the same proportion as the drop rate. Thus, allowing us to perform extra calculations in order to achieve the theoretical speedup. The amount of extra time spent on redundant calculations may be as much as $R=\left(M / \tilde{M}-1\right)$, times the amount of computation performed without applying \textit{DropCompute}.
% For example, when dropping 10\% of the samples, we should be able to run $\sim11\%$ extra calculations.

% This can be achieved in various ways. The most simple method of compensation for LLM training, is adding an extra $R\cdot I_\mathrm{base}$ steps to the training, where $I_\mathrm{base}$ is the number of training steps used without applying \textit{DropCompute}. In practice, even fewer steps are required to achieve the original accuracy, as shown in Figure \ref{fig:time_to_train}, which results in an even higher effective speedup.
% other possible methods include an increased maximal batch size by $R\%$, or explicitly saving the dropped samples for later re-computation by modifiyng the data loader. These methods are tested in Table \ref{table:compensating_dropped_samples}

The effective speedup metric $S_\mathrm{eff}$, accounts for dropped samples by treating them as a source of slowdown in direct proportion to the drop rate. This consideration enables us to execute additional computations to achieve the theoretical speedup. The extent of extra time spent on redundant calculations can be as much as $R=(M / \tilde{M}-1)$ times the computational effort required when not applying \textit{DropCompute}.

For instance, when 10\% of the samples are dropped, we can expect to perform approximately 11\% more calculations. Achieving this can be approached in several ways. One straightforward compensation method for LLM training involves adding an extra $R\cdot I_\mathrm{base}$ steps to the training process, where $I_\mathrm{base}$ represents the number of training steps conducted without using \textit{DropCompute}. In practice, achieving the original accuracy often requires even fewer additional steps, as illustrated in Figure \ref{fig:time_to_train}, resulting in an even higher effective speedup.

% Other potential methods include increasing the maximal batch size by $R\%$ or explicitly preserving the dropped samples for later re-computation by modifying the data loader. These approaches are rigorously tested and compared in Table \ref{table:compensating_dropped_samples}

Another method of compensating for the dropped samples is to increase the maximal batch size. When increasing the batch by $R$ and dropping in average $1-\tilde{M}/M$, we keep the average batch size the same as without using \textit{DropCompute}, hence compensating for the lost samples. 
A third method, orthogonal to the first two, is resampling dropped samples before starting a new epoch to diversify the overall samples seen by the model. 
% A third method of compensation is to explicitly save the indices for dropped samples for later computation. Then, every few steps all workers will retrieve extra samples from a pool of dropped indices and compute them as extra micro-batches.
These approaches are rigorously tested and compared in Table \ref{table:compensating_dropped_samples}

\section{Experiments} \label{section:experiments}
To be useful, \textit{DropCompute} must possess two properties. First, it should not compromise the accuracy of the trained model. This property is put to test in section \ref{sec:generalization_perf} where we fully train BERT-Large and ResNet-50 \citep{devlin2018bert, resnet}, each on a different task, with different drop rates to compare accuracy.
Second, \textit{DropCompute} should maintain a high level of runtime performance, especially when compute variance or straggling workers exist and vanilla synchronous training time deteriorates. Section \ref{sec:runtime_perf} tests runtime performance of \textit{DropCompute} by training a 1.5 billion parameter language model, BERT1.5B \citep{devlin2018bert} with additive noise to the compute time of each worker.   

\textbf{Experimental setup.} The analysis of all BERT models is performed on the same dataset as \citet{devlin2018bert}, which is a concatenation of Wikipedia and BooksCorpus with 2.5B and 800M words respectively. The finetuning of the pretrained models is performed on SQuAD-v1 \citep{rajpurkar-etal-2016-squad}. We verify the generality of \textit{DropCompute} by additional evaluation of a ResNet-50 model for image classification on ImageNet \citep{imagenet_cvpr09}. The experiments depicted in section \ref{sec:runtime_perf} and section \ref{sec:generalization_perf} are executed on Habana Gaudi-1  and Gaudi-2 accelerators, respectively, with high performance network \citep{habana2020-whitepaper}.

\subsection{Generalization performance} \label{sec:generalization_perf}
The sole difference in the optimization when \textit{DropCompute} is applied is that the batch size is not deterministic, but stochastic, as explained in section \ref{sec:our_method}. To complement theorem \ref{theorem:non_convex}, we examine the generalization performance achieved with a stochastic batch size on two popular tasks. 

\textbf{Image classification.} To evaluate the generality of stochastic batch size and \textit{DropCompute} in particular, we evaluate the Top-1 accuracy of a ResNet-50 model on the Imagenet dataset using our method. Since it is not common to use gradient accumulation in large scale training of this task, we simulate the drops such that each worker randomly drops its local batch, so the total batch size is stochastic. This simulated environment enables us to examine the extent of drop rate we can use without compromising accuracy. Figure \ref{fig:resnet50_accuracy} in appendix \ref{appendix:image_classification} shows that up to $10\%$ drop rate, which is more than what \textit{DropCompute} operates on, there is a negligible deterioration in accuracy.

\begin{table}[]\vspace{-2mm}
\begin{subtable}{0.47\textwidth}
    \caption{Varying drop rate, no compensation}
    \begin{tabular}{| c c |}
         \hline
         \quad\% Drop rate \quad & \quad F1 score on dev set \quad\quad \\ [0.5ex] 
         \hline\hline
         0\% & 91.32 $\pm$ 0.15 \\ 
         \hline
         2.5-3\% & 91.34 $\pm$ 0.04  \\
         \hline
         5.5-6\% & 91.44 $\pm$ 0.02  \\
         \hline
         10-11\% & 91.19 $\pm$ 0.02  \\
         \hline
    \end{tabular}
    \label{table:accuracy}\vspace{0.5mm}
\end{subtable}
\begin{subtable}{0.49\textwidth}
    \caption{10\% drop rate, with compensation}
    \begin{tabular}{| c c |}
        \hline
        Compensation method & F1 score on dev set  \\ [0.5ex] 
        \hline\hline
        None & 91.19 $\pm$ 0.02 \\ 
        \hline
        11\% extra steps & 91.40 $\pm$ 0.08  \\
        \hline
        11\% increased batch size & 91.38 $\pm$ 0.08  \\
        \hline
        re-computation & 91.19 $\pm$ 0.11  \\
        \hline
    \end{tabular}
    \label{table:compensating_dropped_samples}\vspace{0.5mm}
\end{subtable}
\caption{\textbf{Maintaining the accuracy of BERT-Large pretraining.} Fine-tuning results on SqUAD v1.1, where the F1 score is obtained by the pretrained model. \textbf{(a)} The effect of different drop rates during pretraining on the final accuracy, without compensating for the dropped samples. \textbf{(b)} When 10\% drop rate is used during pretraining, with different methods of compensating for the dropped samples.}\vspace{-5mm}
\end{table}

\textbf{Large language model.} Training LLMs is resource intensive, typically using large batch sizes, which makes \textit{DropCompute} appealing. We evaluate \textit{DropCompute} method on this task by fully pretraining BERT-Large model several times, each with a different drop rate. We follow the optimization regime described in \citet{you2019large} with a batch size of 64K for phase-1 and 32K for phase-2 (more details are provided in appendix \ref{appendix:generalization}). Each of the pretrained models is fine-tuned on the SQuAD task 3 times with different initializations. Fine-tuning is performed without drops, as it is not a large scale resource consuming task. Table \ref{table:accuracy} shows the average accuracy ($\pm$ standard deviation) obtained for each drop rate. As shown, \textit{DropCompute} at drop rates of up to $10\%$ have negligible accuracy difference. Higher values measured up to $20\%$ of dropped gradients provide acceleration with a small yet discernible hit on accuracy. We note that these results are for a fixed budget of steps. In the presence of compute variance, the effective speedup indicates that additional steps can be executed while still maintaining competitive runtime performance. This notion is demonstrated in section \ref{sec:runtime_perf}.

% {\renewcommand{\arraystretch}{1.1}
% \begin{table}[h!]
%     \centering
%     \caption{\textbf{Maintaining accuracy of BERT-Large pretraining.} Fine-tuning results on SqUAD v1.1, where The F1 score is obtained by the pretrained model, each with a different drop rate during pretraining.}
%     \begin{tabular}{|c c c c |}
%          \hline
%          Batch size & Steps & \% Drop rate & F1 score on dev set  \\ [0.5ex] 
%          \hline\hline
%          64K/32K & 8600 & 0\% & 91.32 $\pm$ 0.15 \\ 
%          \hline
%          64K/32K & 8600 & 2.5-3\% & 91.34 $\pm$ 0.04  \\
%          \hline
%          64K/32K & 8600 & 5.5-6\% & 91.44 $\pm$ 0.02  \\
%          \hline
%          64K/32K & 8600 & 10-11\% & 91.13 $\pm$ 0.02  \\
%          \hline
%     \end{tabular}
    
%     \label{table:accuracy}
% \end{table}
% }

\subsection{Runtime performance} \label{sec:runtime_perf}
The main purpose of our proposed method is to maintain runtime performance when compute variance is present. We examine this by measuring the speedup of \textit{DropCompute} over standard synchronous training in several settings. First, we measure the potential speedup for different drop rates and training settings by post analysis of synchronous training without drops. In addition, we introduce compute variance by training with additive noise, and measure actual speedups using \textit{DropCompute}. The experiments in this section are performed on BERT1.5B. Details are provided in appendix \ref{appendix:runtime_performace}.

% In this section, we measure the potential and actual speedup in terms of runtime performance when applying \textit{DropCompute}. We train BERT-1.5B for a fixed number of steps with and without drops. For training without drops we also calculate speedup curves, given different compute thresholds and corresponding potential drop rates and speedups.

\textbf{Training with different number of workers and micro-batches.}
We evaluate the potential speedup of \textit{DropCompute} on several training settings with natural heterogeneity and no drops. For each setting, we post analyze what would have been the speedup for different drop rates. As can be seen in Figure \ref{fig:speedup}, \textit{DropCompute} exhibits increasing benefits with a growing number of workers and compute requirements. However, there are diminishing returns in terms of speedup with more accumulations. This could possibly be explained by the amortization time of a large number of micro-batches.

\begin{figure}[h!]
  \centering
  \begin{subfigure}[]{0.49\textwidth}
      \includegraphics[width=\textwidth]{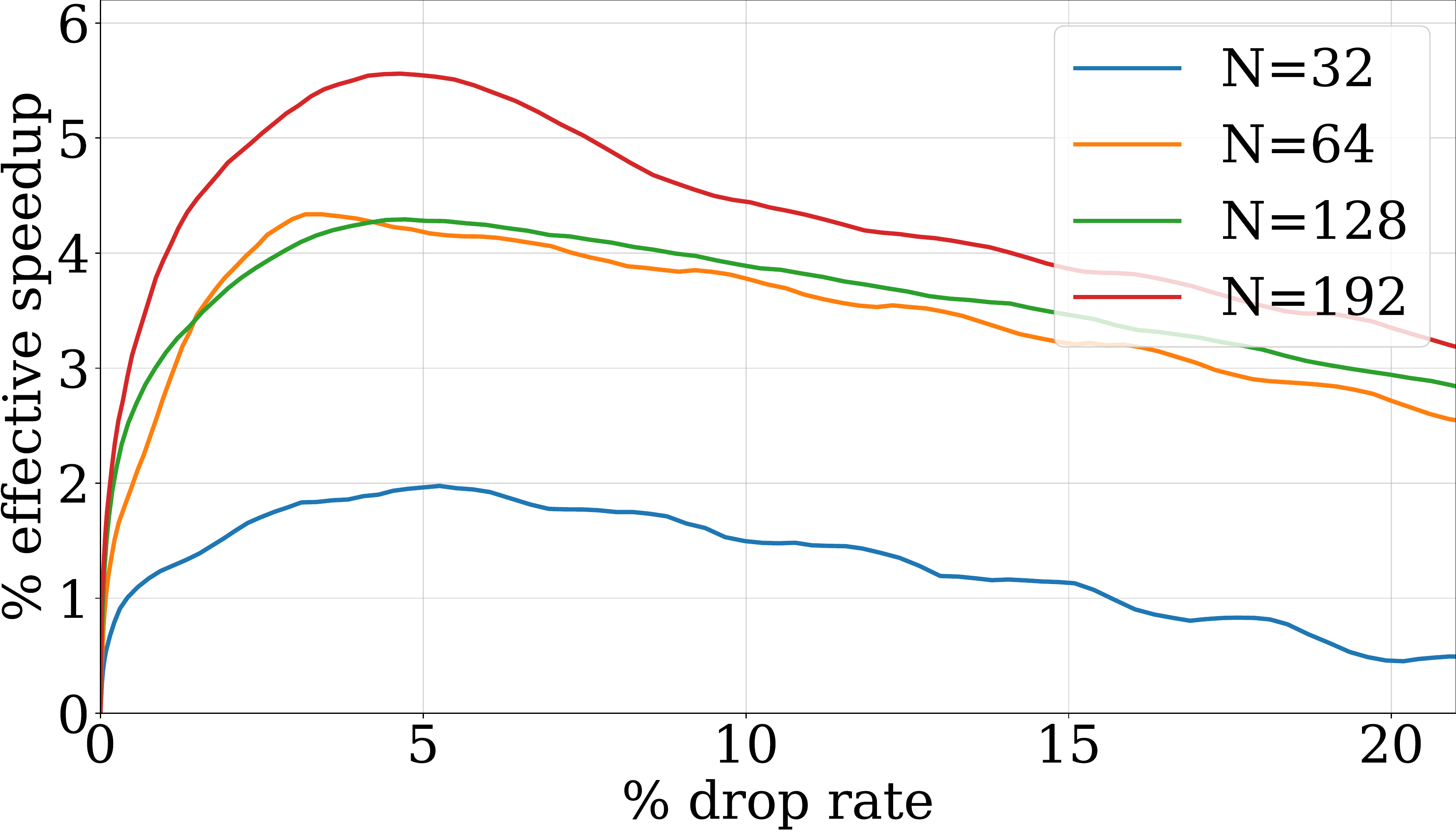}
  \end{subfigure}
  \hfill
  \begin{subfigure}[]{0.49\textwidth}
      \includegraphics[width=\textwidth]{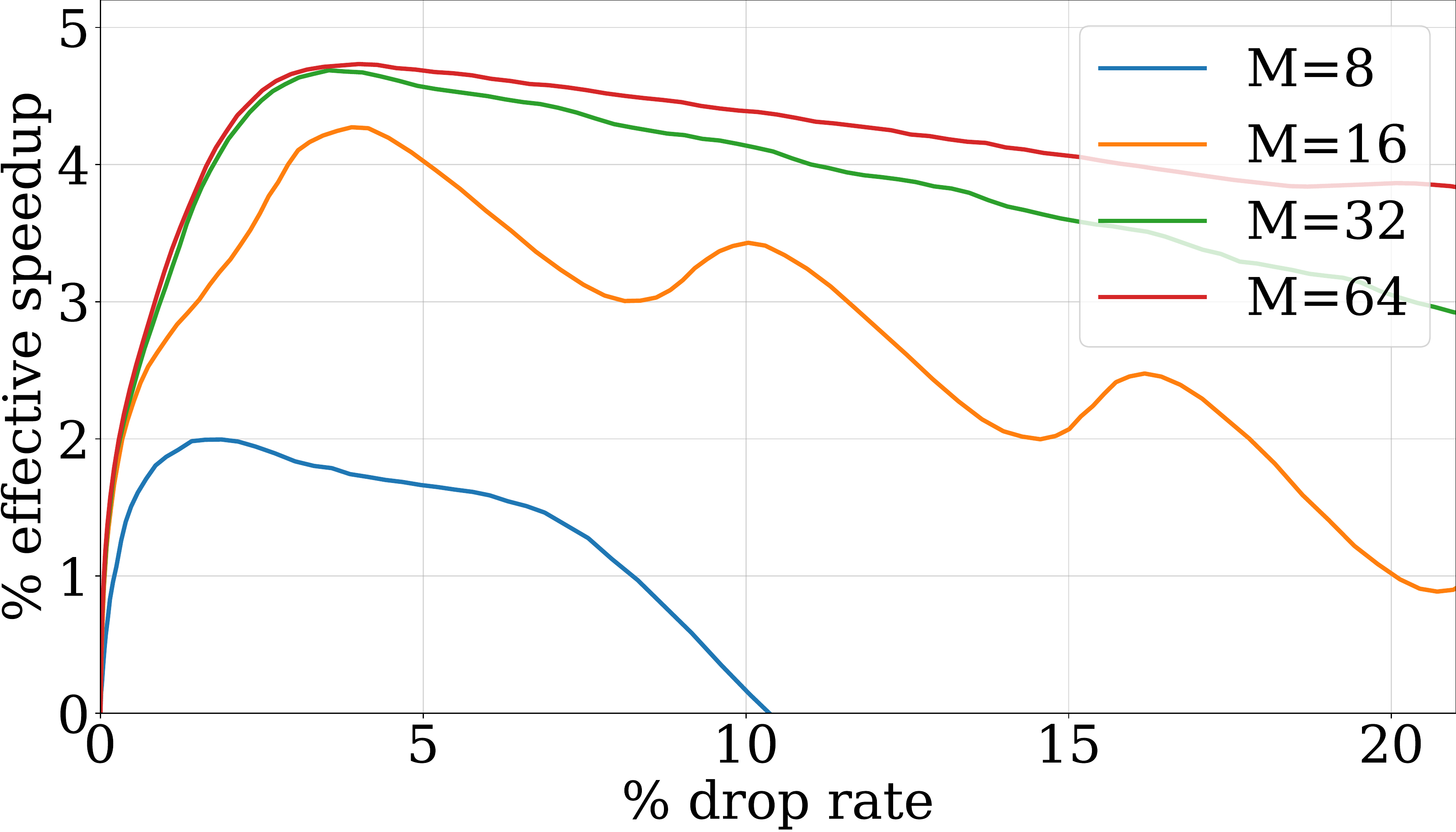}
  \end{subfigure}
  \caption{\textbf{\textit{DropCompute} exhibits increasing benefit on a large scale.}
  Effective speedup versus drop rate with (left) 32 accumulations and varying workers, and (right) 112 workers and varying number of accumulations.}
  \label{fig:speedup}
\end{figure}

\textbf{Simulated delay environment.} Although \textit{DropCompute} may have value when the workers' compute latency variance is low, its significance becomes crucial when the workers' compute latency exhibits high variability. To evaluate our method, we introduce a delay environment where random latency is added to each micro-batch computation. This additive noise follows a bounded log-normal distribution. Detailed information and motivation regarding the additive noise are in appendix \ref{appendix:runtime_performace}.
The experiments are executed with 12 gradient accumulations and a local batch size of 192.
In Figure \ref{fig:abstract}, the negative impact of compute variance on scalability is demonstrated and mitigated using \textit{DropCompute}. The results in Figure \ref{fig:abstract} also correspond to section \ref{section:analysis} and Equation \ref{eq:eff_speedup}, where a theoretical extrapolation follows the same trend line.
When utilizing \textit{DropCompute} in this setup, achieving the same training loss as the baseline might requires additional training steps, however, it leads to a notable reduction in overall training time. Figure \ref{fig:time_to_train} demonstrates it in  a training session with 64 workers, where approximately 3\% more steps is needed to reach the same loss, in 13\% less time.

\begin{figure}[h!]
  \centering
  \begin{subfigure}[]{0.49\textwidth}
      \includegraphics[width=\textwidth]{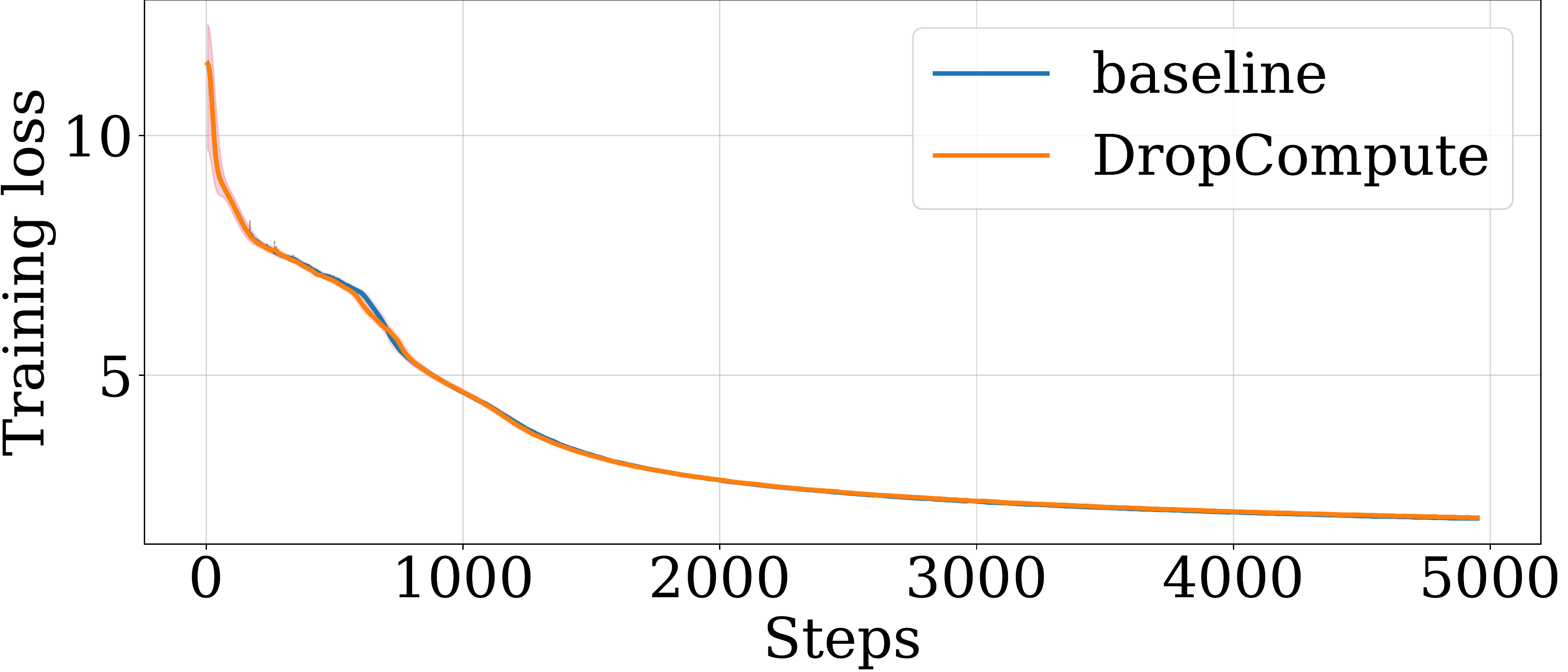}
  \end{subfigure}
  \hfill
  \begin{subfigure}[]{0.49\textwidth}
      \includegraphics[width=\textwidth]{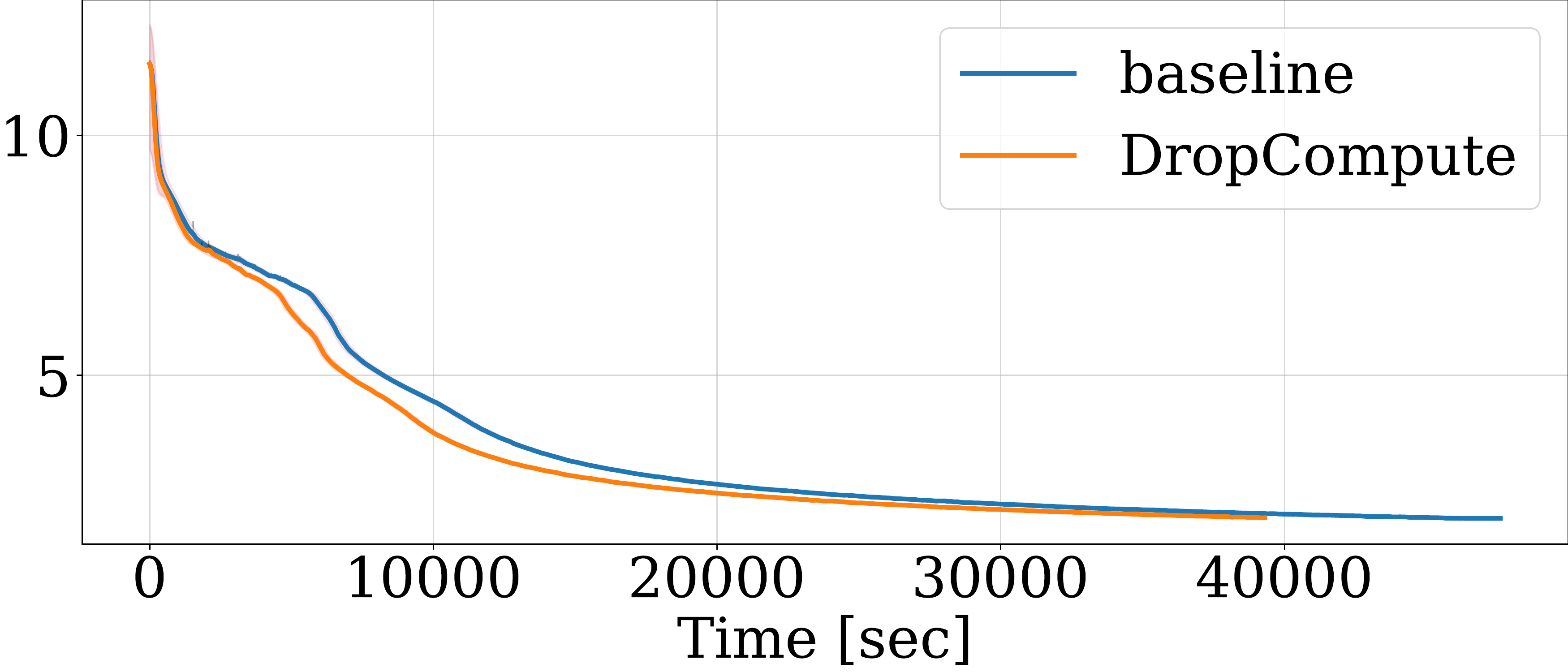}
  \end{subfigure}
  \caption{\textbf{\textit{DropCompute} improves training time for workers with compute variance.} Train loss curve of BERT1.5B pretraining, in a simulated delay environment. (left) Horizontal axis in training steps, and (right) horizontal axis in training time.}
  \label{fig:time_to_train}\vspace{-5mm}
\end{figure}

\section{Discussion}
\textbf{Summary.} Efficient scalable systems are a key component to enable the continued development of deep learning models. To this day, state-of-the-art models rely on synchronous distributed optimization. The challenge to maintain synchronous training as an efficient solution grows larger with the quickly growing model sizes and data. Therefore, improving the robustness and scalability of distributed synchronous training is an important endeavor. This paper tackles the challenge of maintaining synchronous training scalable in the face of compute variance. We propose \textit{DropCompute} to improve the robustness of synchronous training. 
Workers drop their remaining compute when they reach a compute threshold, determined by exchanging and analyzing the compute latency distribution. We find that for a small percentage of dropped data, a much larger percentage of time can be saved, depending on the compute latency distribution of the workers. In addition, we provide theoretical convergence guarantees and runtime predictions.
%As workloads increase in size and the number of participating workers grows, it becomes increasingly crucial to minimize variance among these workers. As demonstrated in this paper, any failure to achieve uniformity among parallel workers directly translates into a loss of performance. This is where \textit{DropCompute} comes into play, effectively recovering this lost performance and potentially saving substantial engineering efforts.
%These efforts may otherwise be spent addressing issues such as faulty hardware, clock throttling, host preemption/overhead, and inefficient load balancing, among others.
%An example of such sub-optimal system can be found in Appendix \ref{appendix:discussion}.
%Even a minor enhancement in robustness and runtime performance can yield significant financial benefits in the context of large-scale training, potentially saving hundreds of thousands of dollars. This highlights the value and novelty of our approach.
We further discuss the motivation behind \textit{DropCompute} and how it effectively solves the problem in appendix \ref{appendix:discussion}.

\textbf{Limitations.} While \textit{DropCompute} is simple and straightforward, it deals with system efficiency, and as such, the user-level implementation provided is not optimal. Mainly, the provided implementation is limited by using many gradient accumulations and integrating compute timeout in between them. However, we believe that this is not a major concern since having multiple gradient accumulations is a common practice in training LLM on a large scale and is used in state-of-the-art training configurations \citep{smith2022using, nvidiaMlperfGPT3}. 
In addition, \textit{DropCompute} addresses variance that originates from the compute stage of the training iteration and does not solve the potential issue of network variance during the all-reduce stage. 

\textbf{Future directions.} \textit{DropCompute} is described and analyzed in this paper as a method built on top of synchronous training. However, this method can be integrated with other possibly asynchronous methods such as periodic synchronization. In appendix \ref{app:local_sgd}, we implement \textit{DropCompute} on top of Local-SGD \citep{Lin2020Don't} and show that \textit{DropCompute} can also improve the robustness of Local-SGD to stragglers.
% In model-parallel setting, \textit{DropCompute} works on a data-parallel level. This decreases the number of workers $N$, but increases the variance between workers --- since it now includes communication between tensor/pipeline-parallel workers. An interesting extension for \textit{DropCompute} would be to apply it on model-parallel frameworks such as \citep{rasley2020deepspeed, Megatron2021}.
A different extension for \textit{DropCompute} is to apply it during the model backward calculation and save the partial gradients that were already calculated. This would generalize \textit{DropCompute} for workloads that do not utilize gradient accumulations. However, it will require further study as it differs from the stochastic batch-size setting where the entire data sample is either saved or dropped.
\vspace*{-3mm}
\section*{Acknowledgments}
We thank Itay Hubara for technical advising and valuable comments on the manuscript.
The research of DS was Funded by the European Union (ERC, A-B-C-Deep, 101039436). Views and opinions expressed are however those of the author only and do not necessarily reflect those of the European Union or the European Research Council Executive Agency (ERCEA). Neither the European Union nor the granting authority can be held responsible for them. DS also acknowledges the support of Schmidt Career Advancement Chair in AI.

\bibliography{bibliography}
\bibliographystyle{neurips_template/neurips}

\appendix
\newpage

% \title{Supplementary Material}

% \begin{document}
% \maketitle
\part*{Appendix}

\section{Further discussion} \label{appendix:discussion}
In this section, we will further elaborate on the motivation for using \textit{DropCompute} and how it mitigates existing problems in large-scale training.

\subsection{Motivation}
The primary objective of \textit{DropCompute} lies in the mitigation of compute latency variance among workers. This raises the question of the significance of compute variance in the context of our research.
Compute variance can arise from various sources, including but not limited to faulty hardware, clock throttling, host preemption/overhead, inefficient load balancing, connectivity issues (particularly in model parallel settings), and more.
Inefficient load balancing is especially in particular, when dealing with with dynamic sentence/image sizes \citep{tan2021efficientnetv2, dehghani2023patch, raffel2020exploring} because it requires special treatment for each model and data set, often at the expense of performing redundant work.
Addressing these issues typically requires intricate engineering efforts:
\begin{itemize}
    \item Regular testing and replacement of faulty hardware.
    \item Mitigation of host overhead through the implementation of latency-hiding techniques and script optimization.
    \item Management of inefficient load balancing on a per-workload basis, employing strategies such as sample padding and packing.
\end{itemize}
However, it is essential to recognize that each of these issues represents a potential single point of failure. The triggering of any one of them can result in a substantial performance degradation within large-scale systems.
Some of our early experiments exhibited naturally such sub-optimal behavior where we clearly see a large variance in compute latency between iterations and workers (as shown in figure \ref{fig:sub_optimal_system_latency}). As our goal is to improve robustness (i.e. performance for outlier cases), these cases are important. Moreover, after the compute variance was reduced by HW and SW optimizations, we were still left with some compute variance (as shown in figure \ref{fig:experiement_vs_theory}).

\begin{figure}[h!]
  \centering
  \begin{subfigure}[]{0.49\textwidth}
      \includegraphics[width=\textwidth]{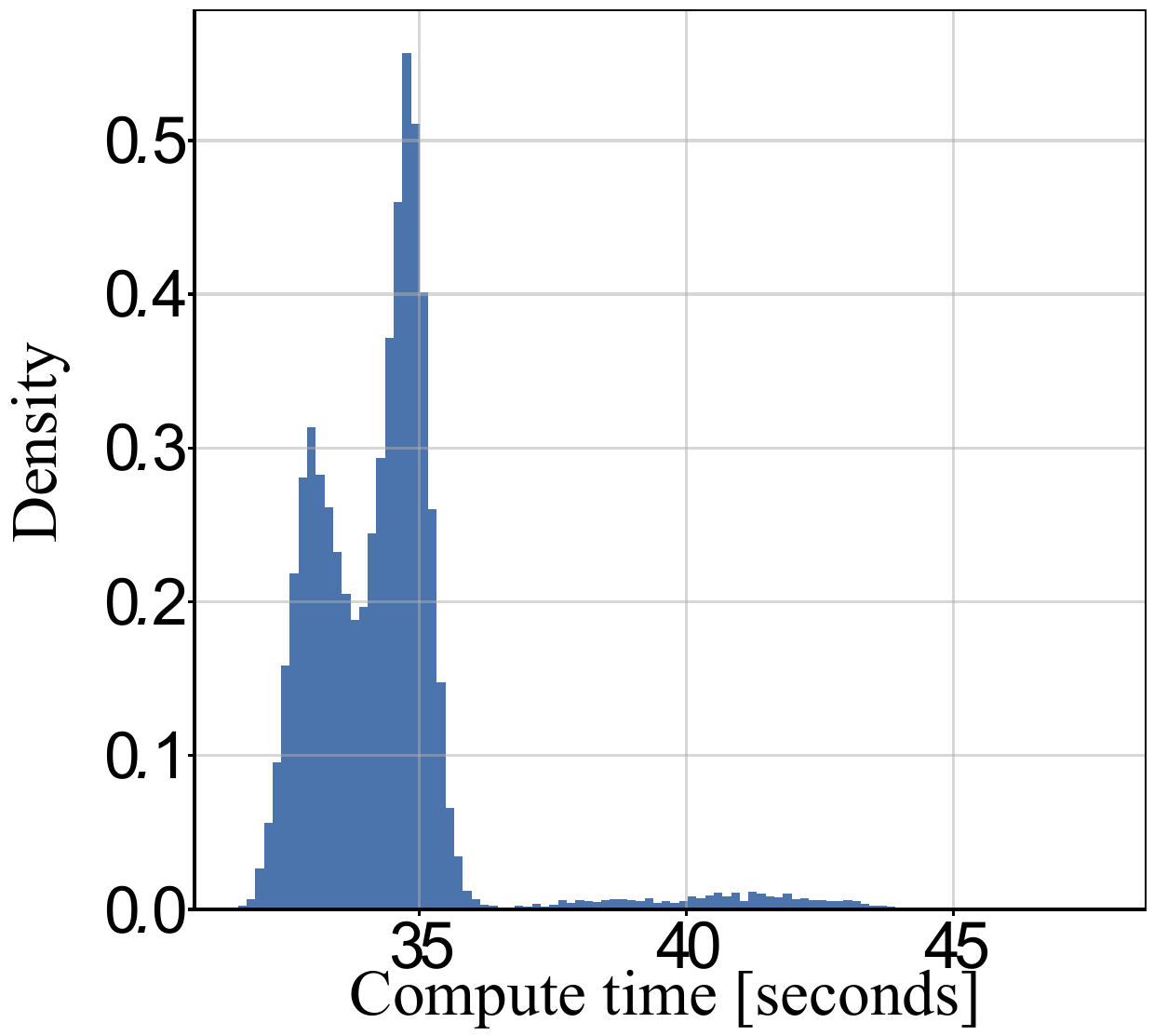}
  \end{subfigure}
  \hfill
  \begin{subfigure}[]{0.49\textwidth}
      \includegraphics[width=\textwidth]{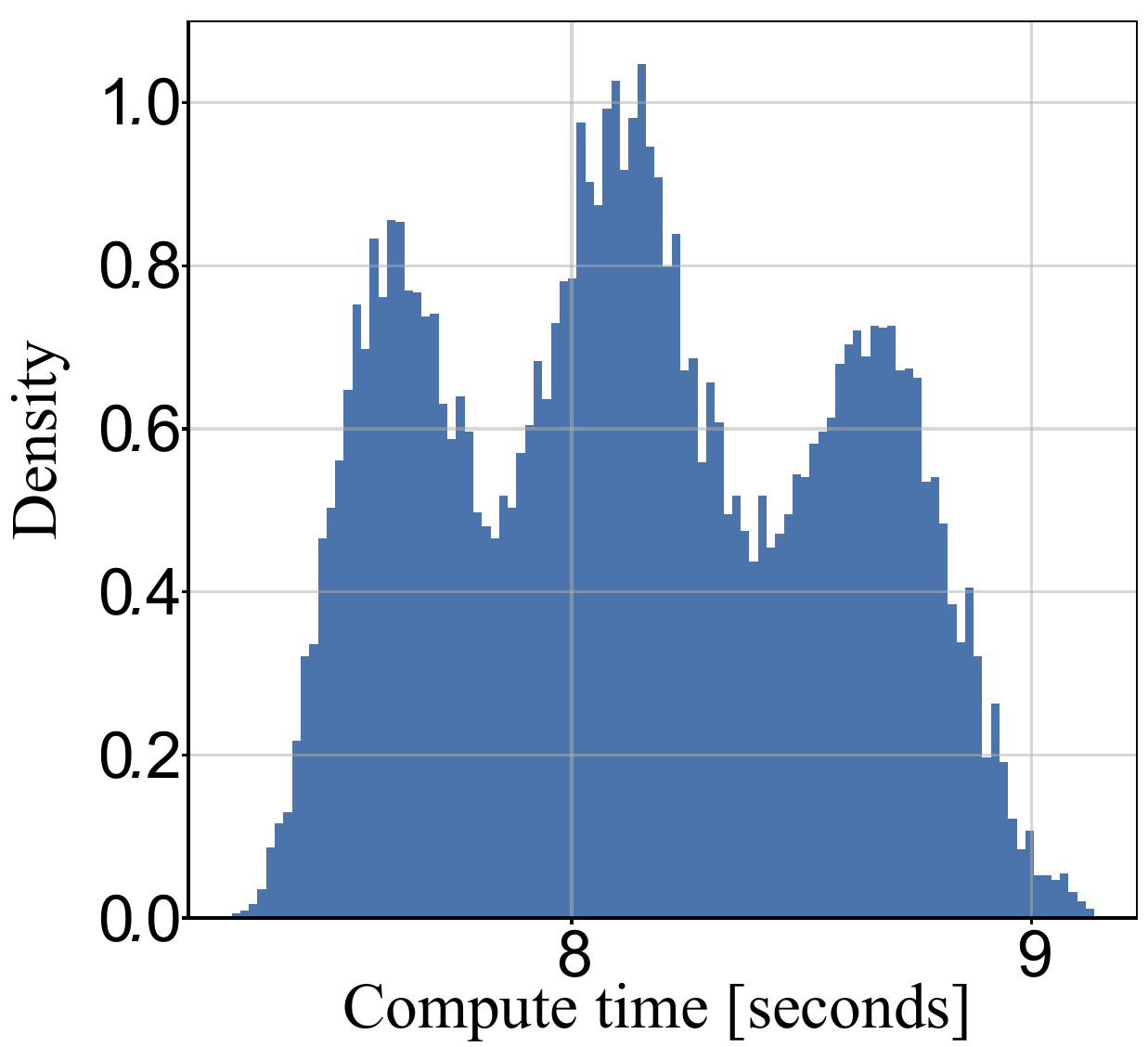}
  \end{subfigure}
  \caption{\textbf{Single iteration latency histogram in a sub-optimal system.} 
  These histogram were recorded in BERT1.5B training, before optimizing our system.
  (left) Training with 162 workers and 64 gradient accumulations. When applying \textit{DropCompute} in this setting, we achieved $\sim18\%$ performance boost.  (right) Training with 190 workers and 16 gradient accumulations.}
  \label{fig:sub_optimal_system_latency}
\end{figure}
These examples lead us to the conclusion that, in practice, workers do not finish computation at the same time, and this can have a significant impact on the training speed. Moreover, the effect of stragglers and compute variance on the training speed is expected to get worse as the distributed scale increases. This is due to the maximal worker distribution relation stated in equation \ref{eq:latency_pdf}. When modeling the additive latency as normally distributed, the average maximal worker latency, $\mathbb{E}[T]$, increases with the number of workers $N$ as $\mathbb{E}[T]=\Theta(\sqrt{\text{log} N})$ as shown in appendix \ref{appendix:speedup_analysis}. 

\subsection{Effectiveness}
% The slowdown which is a result of compute variance can be avoided effortlessly using \textit{DropCompute}.
% For example, in a sub-optimal system with stragglers, shown in figure \ref{fig:sub_optimal_system_latency} (left), we recovered ~18\% of the run-time performance. Given different systems and noise distribution, the contribution of \textit{DropCompute} can be even larger, as shown in appendix \ref{appendix:noise_analysis}. In addition, even when assuming normal distributed noise, the theoretical analysis shows a greater speedup as the number of workers $N$ increases: 
% $S_\mathrm{eff}(N)\underset{N\to\infty}{\longrightarrow}\infty$.
% On our own system, after the compute variance was reduced by HW and SW optimizations, we still gained a performance boost of 5\% on 196 workers (figure \ref{fig:experiement_vs_theory}), and the speedup is increased as the scale increases. These examples show that \textit{DropCompute} makes large-scale training more robust, and “recovers” the lost performance due to any stochastic performance outlier. 

% Last, but not least, even a small gain to large-scale training performance can be substantial - for example, if we assume 10\$--32\$/hour/8xA100 (according to AWS pricing), then saving 5\% of the training time for the 70B model such as \citep{touvron2023llama} would have saved 107,50\$--344,064\$, and 67,686\$--216,598\$ for a larger model such as \citep{scao2022bloom}.

Mitigating slowdowns resulting from compute variance can be achieved easily buy using \textit{DropCompute}. For instance, in a sub-optimal system with stragglers, as illustrated in Figure \ref{fig:sub_optimal_system_latency} (left), we were able to recover approximately 18\% of the runtime performance. The contribution of \textit{DropCompute} can be even more significant in different systems with varying noise distributions, as demonstrated in Appendix \ref{appendix:noise_analysis}.

Furthermore, even when assuming a normal distribution of noise, theoretical analysis indicates a substantial speedup as the number of workers $N$ increases:
$S_\mathrm{eff}(N)\underset{N\to\infty}{\longrightarrow}\infty$.
On the tested system, after reducing compute variance through hardware and software optimizations, we achieved a 5\% performance boost with 196 workers (see Figure \ref{fig:experiement_vs_theory}). Notably, this speedup continues to increase as the scale of the system grows. These examples underscore how \textit{DropCompute} enhances the robustness of large-scale training, effectively recovering lost performance attributed to stochastic performance outliers.

Lastly, it's important to emphasize that even a modest improvement in large-scale training performance can yield significant cost savings. For example, assuming a cost of \$10--32/hour/8xA100 (according to AWS pricing), saving 5\% of the training time for a 176B model, such as \citep{scao2022bloom}, would result in savings ranging from \$67,686 to \$216,598. For longer-trained models, like \citep{touvron2023llama}, the savings could reach \$107.50 to \$344,064."

\section{Experiments} \label{appendix:experiments}

\subsection{Runtime performance experiments} \label{appendix:runtime_performace}
In this section, we provide details for the experiments of section \ref{sec:runtime_perf}.

\textbf{Experiment details.}
As mentioned in the paper in section \ref{sec:runtime_perf}, we pre-train BERT1.5B following \citet{habana-deepspeed-bert}. The experiments in this section use up to 200 Gaudi accelerators with high bandwidth inter-connectivity. The training is done with a maximum input sequence length of 128 and 80 maximum predictions per sequence. The training regime consists of a local batch size of 196, 12 gradient accumulations, LANS optimizer \citep{zheng2020accelerated}, and a learning rate of 0.0015. Due to the large capacity of the model, we used ZeRO optimizer stage 1 to fit the model in memory \citep{rajbhandari2020zero}. 
%\subsection{Runtime performance} \label{appendix:runtime_performace}
% \textbf{Heterogeneity and stragglers.} Large scale training is prone to sporadic stragglers and heterogeneity. This can happen even in highly optimized homogeneous systems. Such scenarios can severely deteriorate the runtime performance of synchronous training. It is therefore vital to verify that our method can handle these cases effectively. We exhibit two such examples and showcase that \textit{DropCompute} is extremely effective in avoiding slowdown due to these variabilities. Figure \ref{fig:stragglers} (a) and (b) show an example where there is heterogeneity between the workers and by dropping $2.5\%$ of data a compute speed up of $10\%$ can be obtained. Figure \ref{fig:stragglers} (c) and (d)  show an example where a single server with 8 workers is straggling, causing a long tail in the compute distribution, and by dropping less than $1\%$ of compute, the training time is faster by more than $10\%$.

\textbf{Simulated delay.} Many frameworks use padding to allow for constant input length which improves hardware efficiency \citep{kosec2021packing}. However, some learning tasks inherently involve dynamic shapes, such as translation \citep{ott-etal-2018-scaling} and multi-task sequences \citep{2020t5}. These use cases motivate us to explore scenarios of dynamic length via simulation. 
To demonstrate the value of \textit{DropCompute} in dealing with compute variance we added to each micro-batch compute time an additional random waiting time. The additive noise is based on a Log-normal distribution since it is typical for user post lengths in internet discussions \citep{sobkowicz2013lognormal}, which are used as training data in recent language models \citep{radford2019language}. To make this setting more realistic, we scale down and bound the noise so that each accumulation takes $\times1.5$ longer on average, and, in extreme cases, can take up to 6 times longer. This allows us to simulate stragglers and high compute variance while keeping a conservative limit on iteration time.
Thus, the additive noise takes the form of
$$\epsilon=\min\left(\frac{1}{\alpha}Z, \beta\right) ,\;\;\;\;\;\;\;\;\; Z\sim \mathrm{LogNormal}(4,1) \,.$$ 
This noise was added to each accumulation
$$t_n^{(m)}\gets t_n^{(m)} + \mu\cdot \epsilon \,,$$ where $\mu$ is the mean value for $t_n^{(m)}$, $\alpha=2\exp(4.5)$ and $\beta=5.5$ are the scaling and bounding constants, and the log-normal parameters (4,1) fit user post lengths, as seen in \citet{sobkowicz2013lognormal}. As illustrated in Figure \ref{fig:epsilon_dist}, the noise distribution leads to each micro-batch latency increased by up to $6\mu$, while the majority of accumulations have low latency. Further analysis on the effect of noise properties is discussed in \ref{appendix:noise_analysis}.

\begin{figure}[h!]
  \centering
  \begin{subfigure}[]{0.49\textwidth}
      \includegraphics[width=\textwidth]{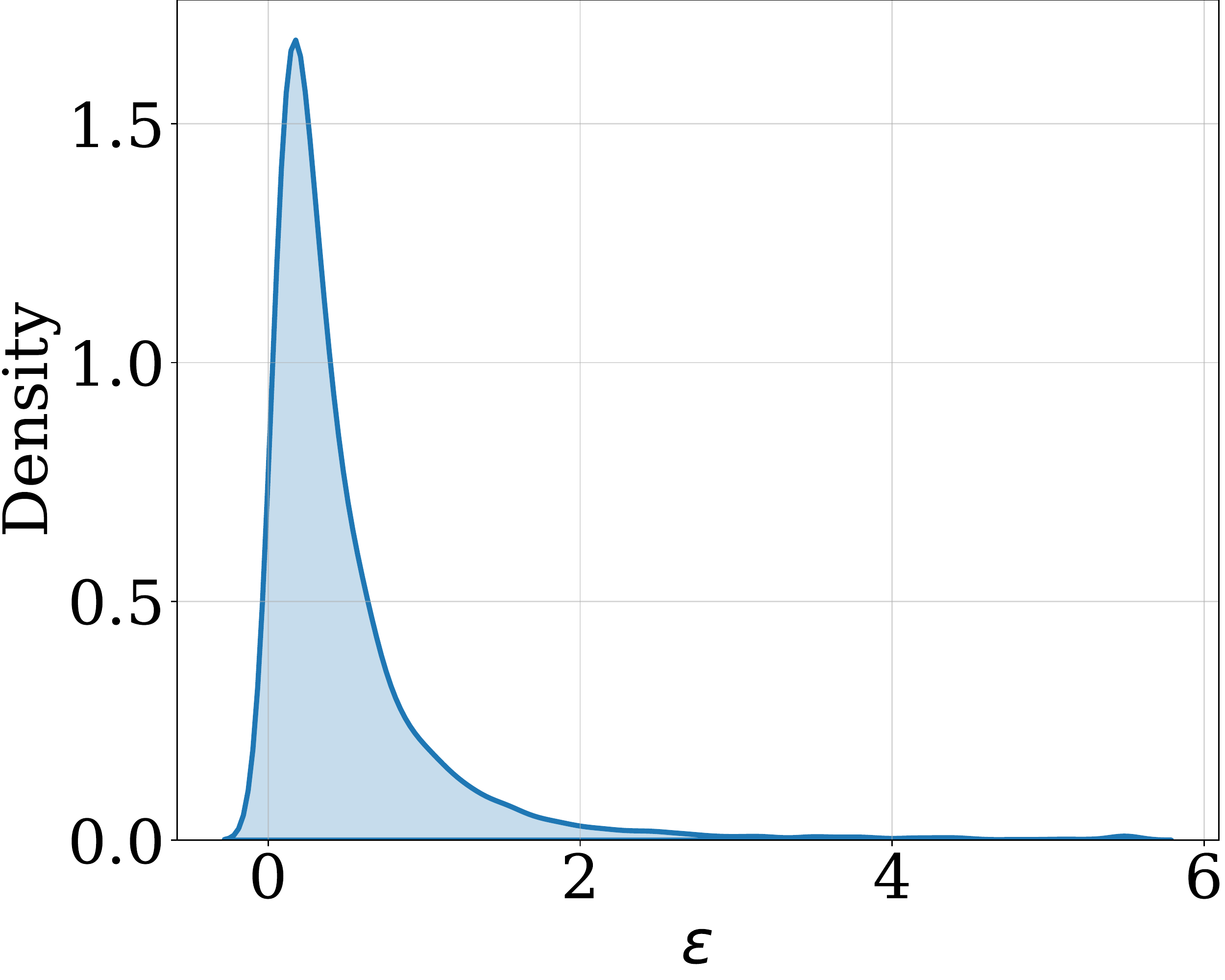}
  \end{subfigure}
  \hfill
  \begin{subfigure}[]{0.49\textwidth}
      \includegraphics[width=\textwidth]{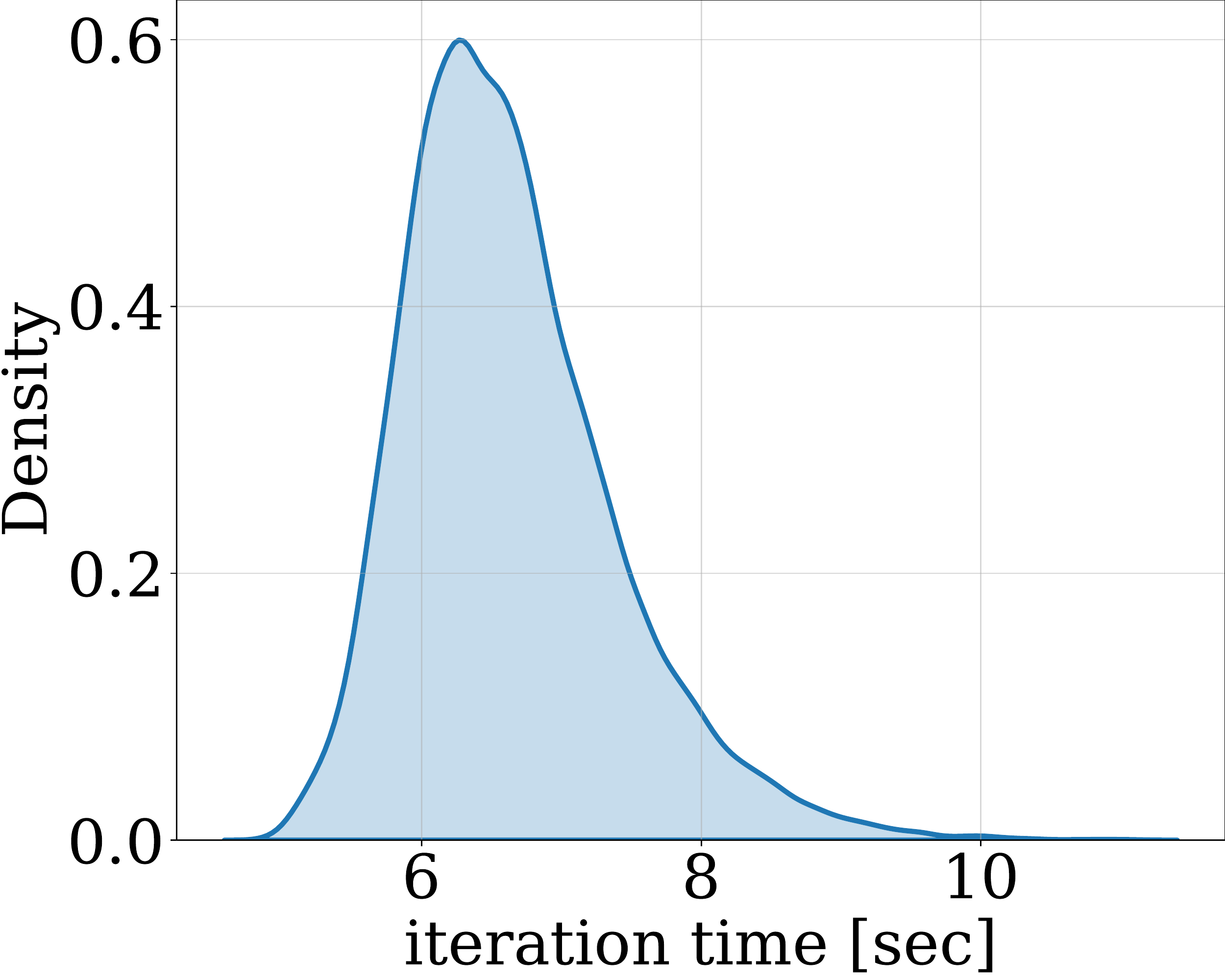}
  \end{subfigure}
  \caption{\textbf{The latency distribution in a simulated delay environment.} (left) The distribution of the additive noise $\epsilon$, added to each accumulation. (right) The distribution for iteration time $T_n$, with 12 accumulations, each with added noise, in BERT1.5B training.}
  \label{fig:epsilon_dist}
\end{figure}

% \begin{figure*}[th!]
%   \centering
%   \begin{subfigure}[]{0.238\textwidth}
%     \caption{}
%     \label{subfig:stragglers_1}
%     \includegraphics[width=\textwidth]{icml_figures/speedup/speedup_vs_drop_n_192.pdf}
%   \end{subfigure}
%   \hfill
%   \begin{subfigure}[]{0.238\textwidth}
%     \caption{}
%     \label{subfig:stragglers_2}
%     \includegraphics[width=\textwidth]{icml_figures/speedup/dist_compute_time_n_192.pdf}
%   \end{subfigure}
%   \hfill
%   \begin{subfigure}[]{0.238\textwidth}
%     \caption{}
%     \label{subfig:stragglers_3}
%     \includegraphics[width=\textwidth]{icml_figures/speedup/speedup_vs_drop_n_160.pdf}
%   \end{subfigure}
%   \hfill
%   \begin{subfigure}[]{0.238\textwidth}
%     \caption{}
%     \label{subfig:stragglers_4}
%     \includegraphics[width=\textwidth]{icml_figures/speedup/dist_compute_time_n_160.pdf}
%   \end{subfigure}
%   \caption{\textbf{\textit{DropCompute} efficiency on heterogeneous and straggling workers.} (a) Compute speedup versus drop rate with 16 accumulations and 192 workers. (b) Compute time distribution corresponding to (a). (c) Compute speedup versus drop rate with 64 accumulations and 160 workers. (d) Compute time distribution corresponding to (c).}
%   \label{fig:stragglers}
% \end{figure*}

\subsection{Generalization experiments} \label{appendix:generalization}
In this section, we provide details for the experiments of section \ref{sec:generalization_perf}.

\subsubsection{Large language models}

Here we provide more details about how the LLM experiment was executed as well as additional graphs related to the LLM experiment described in section \ref{sec:generalization_perf}.

\textbf{Experiment details.} As mentioned in the paper, in section\ref{sec:generalization_perf} we follow \citet{you2019large} optimization regime with LAMB optimizer. Specifically, for phase-1 where the sequence length is 128 tokens per sample, we use a batch size of 64K, the learning rate is 0.006, the warmup ratio is 0.2843, and the steps number is 7038. For phase-2 where the sequence length is 512, we use a batch size of 32K, the learning rate is 0.004, the warmup ratio is 0.128 and the steps number is 1563. The experiments were executed on 64 workers.

\textbf{Batch size distribution.} As explained in section \ref{sec:generalization_perf} we fully pretrain a BERT-Large model with \textit{DropCompute} several times, each with a different drop rate. Figure \ref{fig:batch_dist} shows the empirical batch distribution of each of the drop rates in phase-1.  

\begin{figure}[h!]
  \centering
  \begin{subfigure}[]{0.32\textwidth}
      \caption{}
      \includegraphics[width=\textwidth]{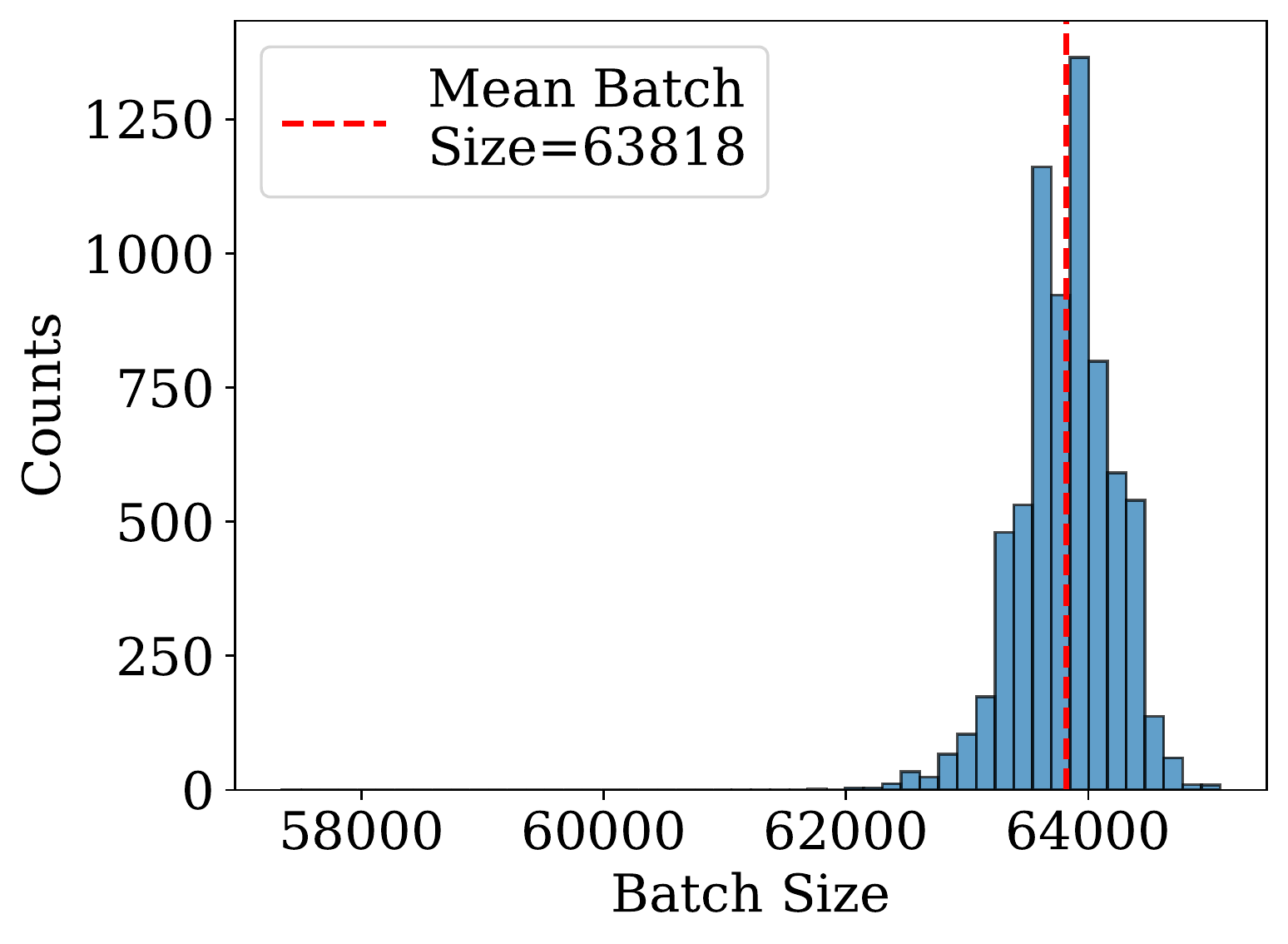}
  \end{subfigure}
  \hfill
  \begin{subfigure}[]{0.32\textwidth}
      \caption{}
      \includegraphics[width=\textwidth]{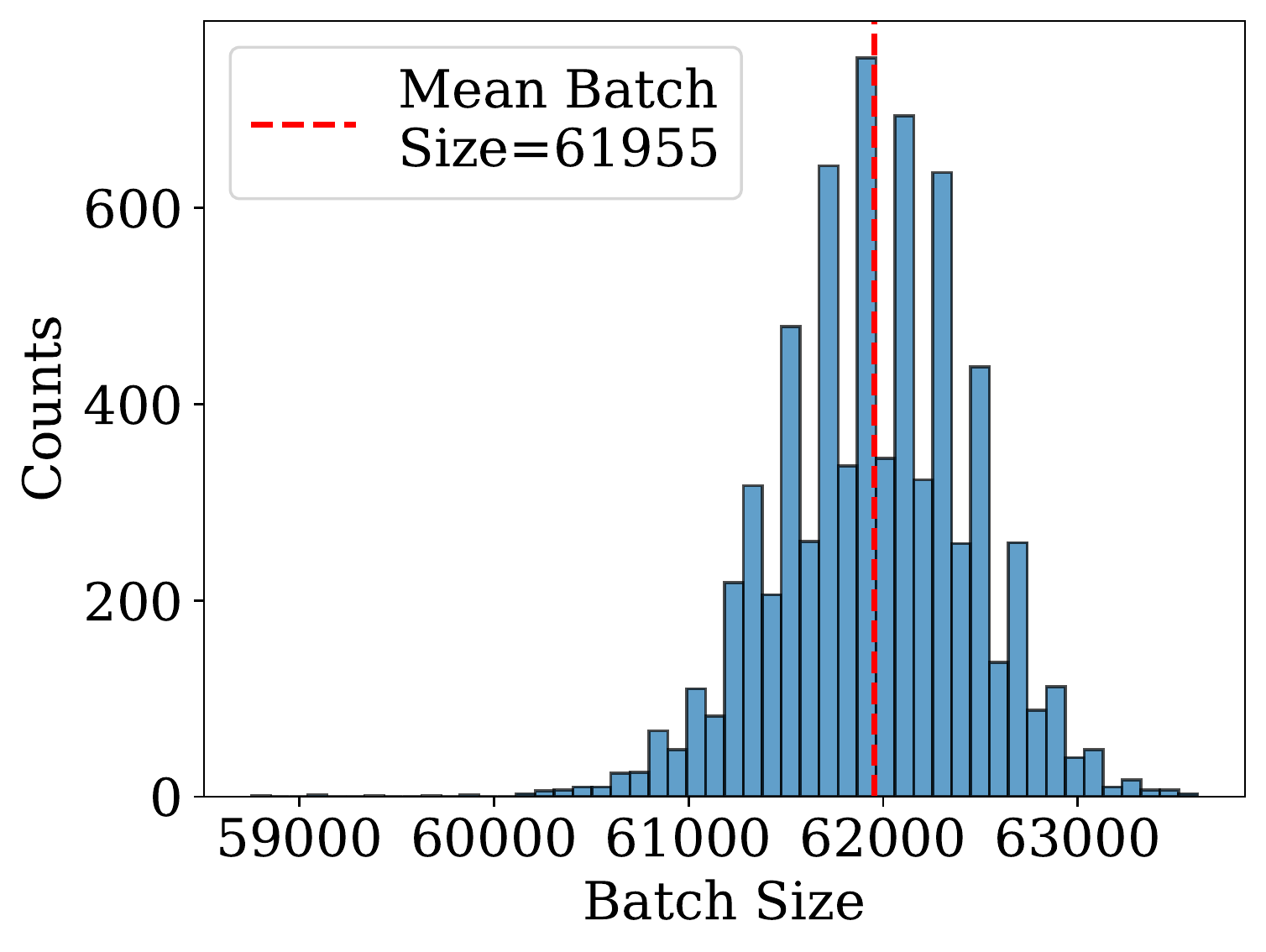}
  \end{subfigure}
  \hfill
  \begin{subfigure}[]{0.32\textwidth}
      \caption{}
      \includegraphics[width=\textwidth]{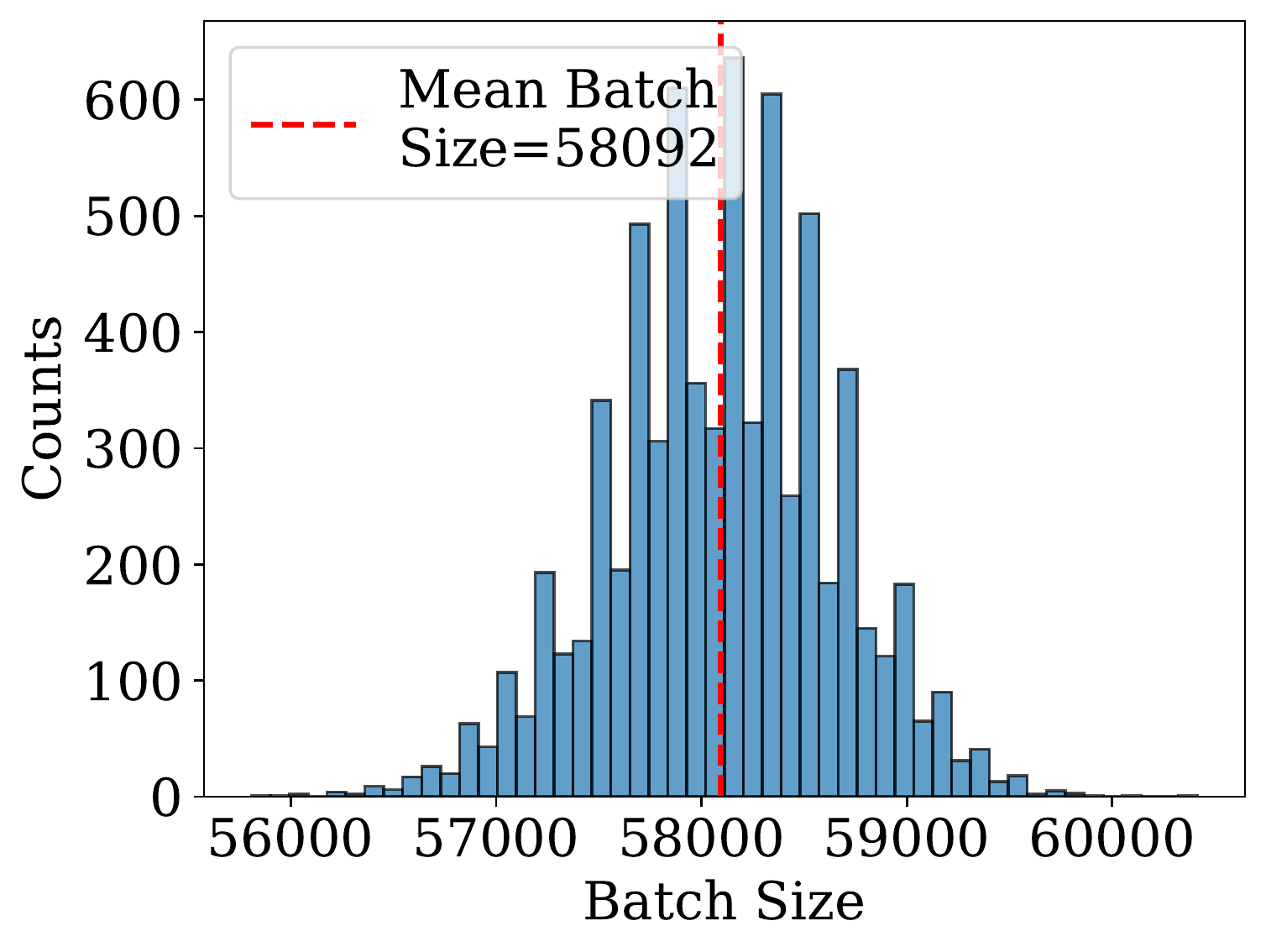}
  \end{subfigure}
  \caption{\textbf{Batch size distribution.} BERT-Large phase-1 pretraining batch size distribution when using \textit{DropCompute} and drop rate of (a) 2.5\% , (b) 5.5\%,
    and (c) 11.5\%}
  \label{fig:batch_dist}
\end{figure}

\textbf{Convergence loss.} In addition to the results depicted in Table \ref{table:accuracy}, we show the convergence of the training loss with the different drop rates in Figure \ref{fig:loss_1563}.

\begin{figure}[h!]
  \centering
  \begin{subfigure}[]{0.49\textwidth}
      \includegraphics[width=\textwidth]{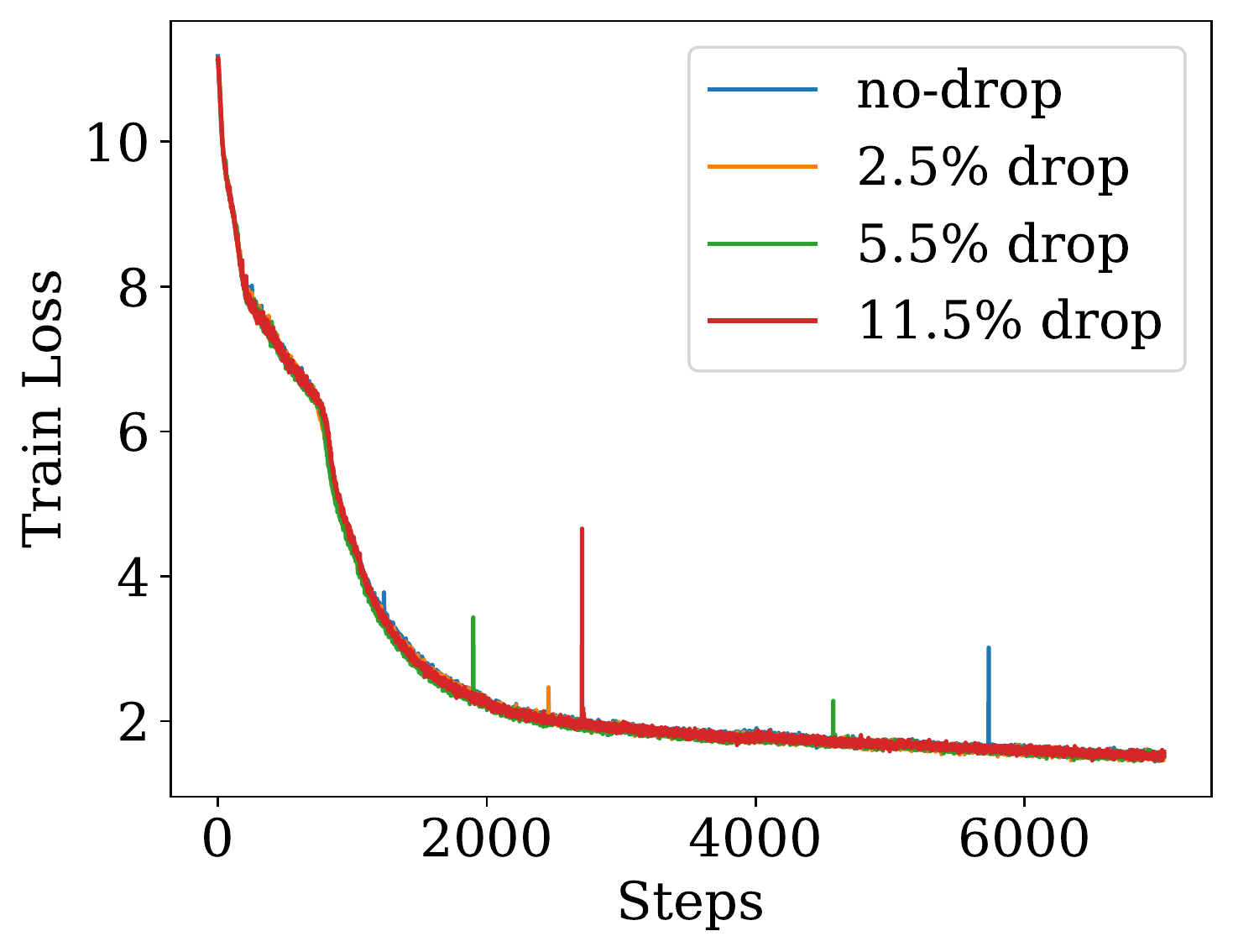}
  \end{subfigure}
  \hfill
  \begin{subfigure}[]{0.49\textwidth}
      \includegraphics[width=\textwidth]{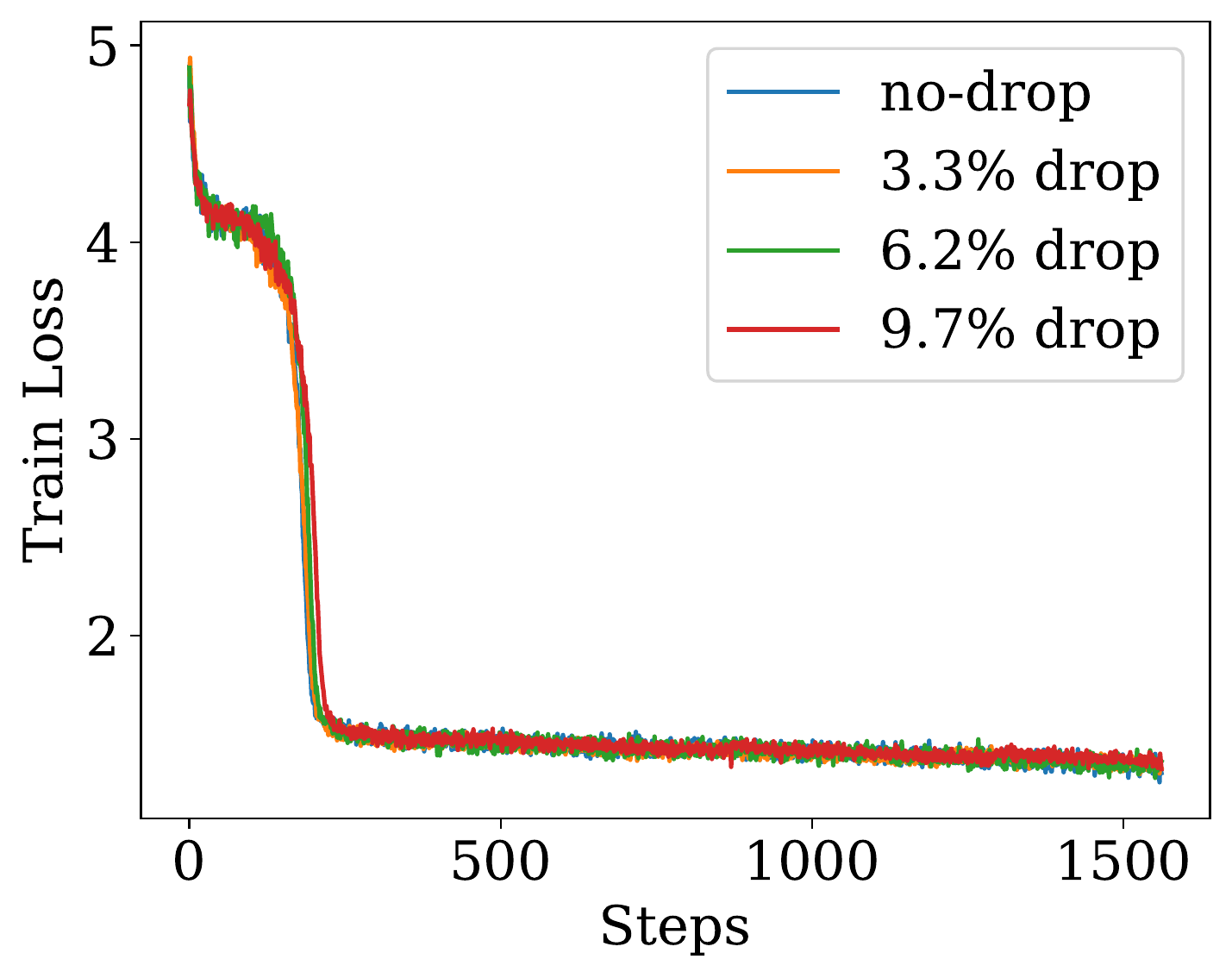}
  \end{subfigure}
  \caption{\textbf{Train loss convergence.} BERT-Large phase-1 (left) and phase-2 (right) pretraining train loss for different drop rates.}
  \label{fig:loss_1563}
\end{figure}

\subsubsection{Image classification} \label{appendix:image_classification}
This section provides the details of the image classification experiment described in section \ref{sec:generalization_perf} as well as Figure \ref{fig:resnet50_accuracy} which is referenced from the paper. 

\textbf{Experiment details.} To simulate \textit{DropCompute}, at each training step, the gradients of each worker are set to zero with a probability of $P_{\text {drop }}$. We repeat each training process 3 times with different initializations. To examine the generalization of \textit{DropCompute} over different optimizers, we implement our method on two popular training regimes of ResNet50. First, we follow the optimization regime described in \citet{goyal2017accurate} that uses SGD with 32 workers and a global batch size of 4096. Second, we follow \citet{mlperf} that uses LARS \citep{you2017large} with 8 workers and a global batch size of 2048.

\begin{figure}[h!]
  \centering
  \begin{subfigure}[]{0.49\textwidth}
      \includegraphics[width=\textwidth]{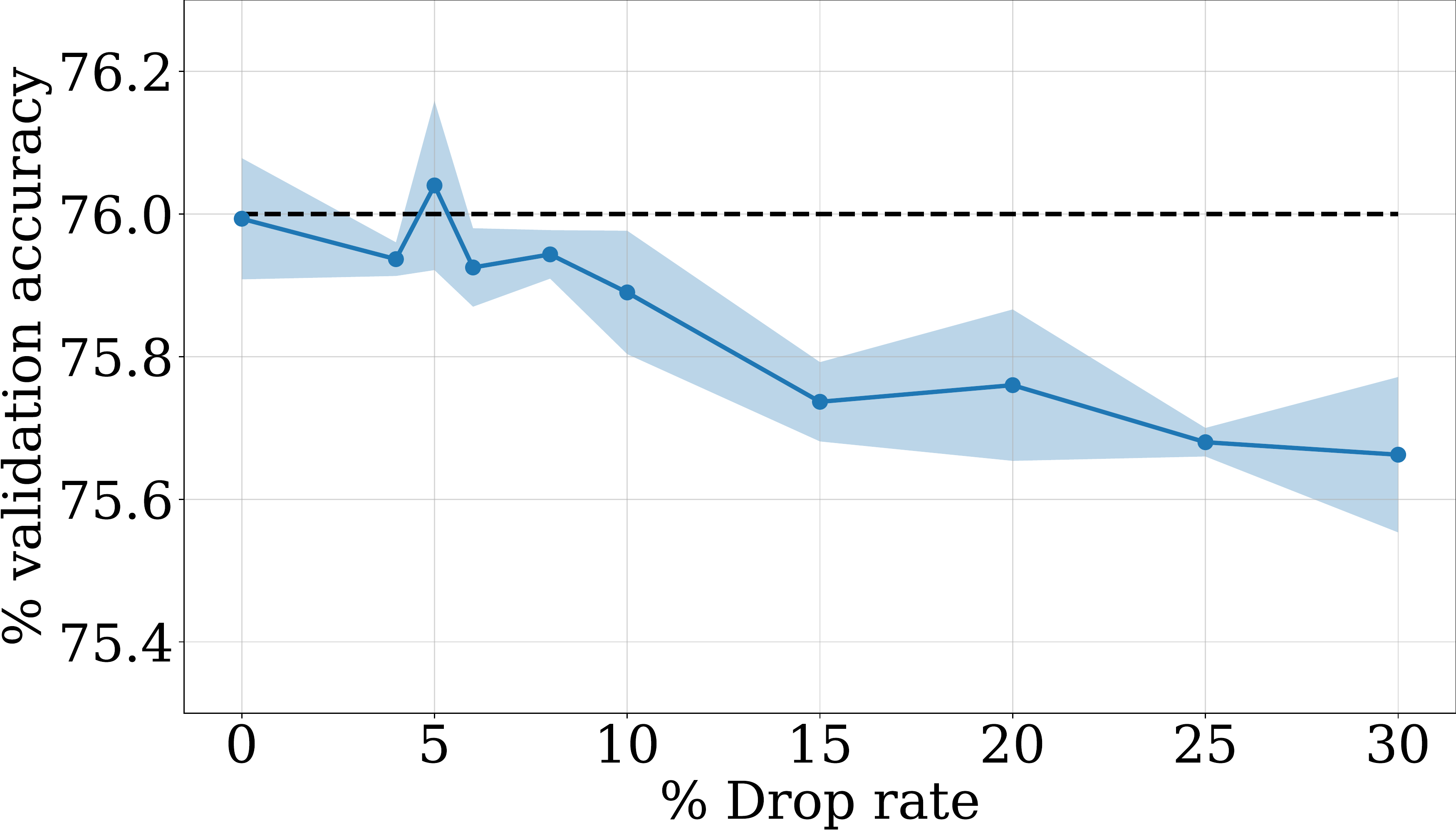}
  \end{subfigure}
  \begin{subfigure}[]{0.49\textwidth}
      \includegraphics[width=\textwidth]{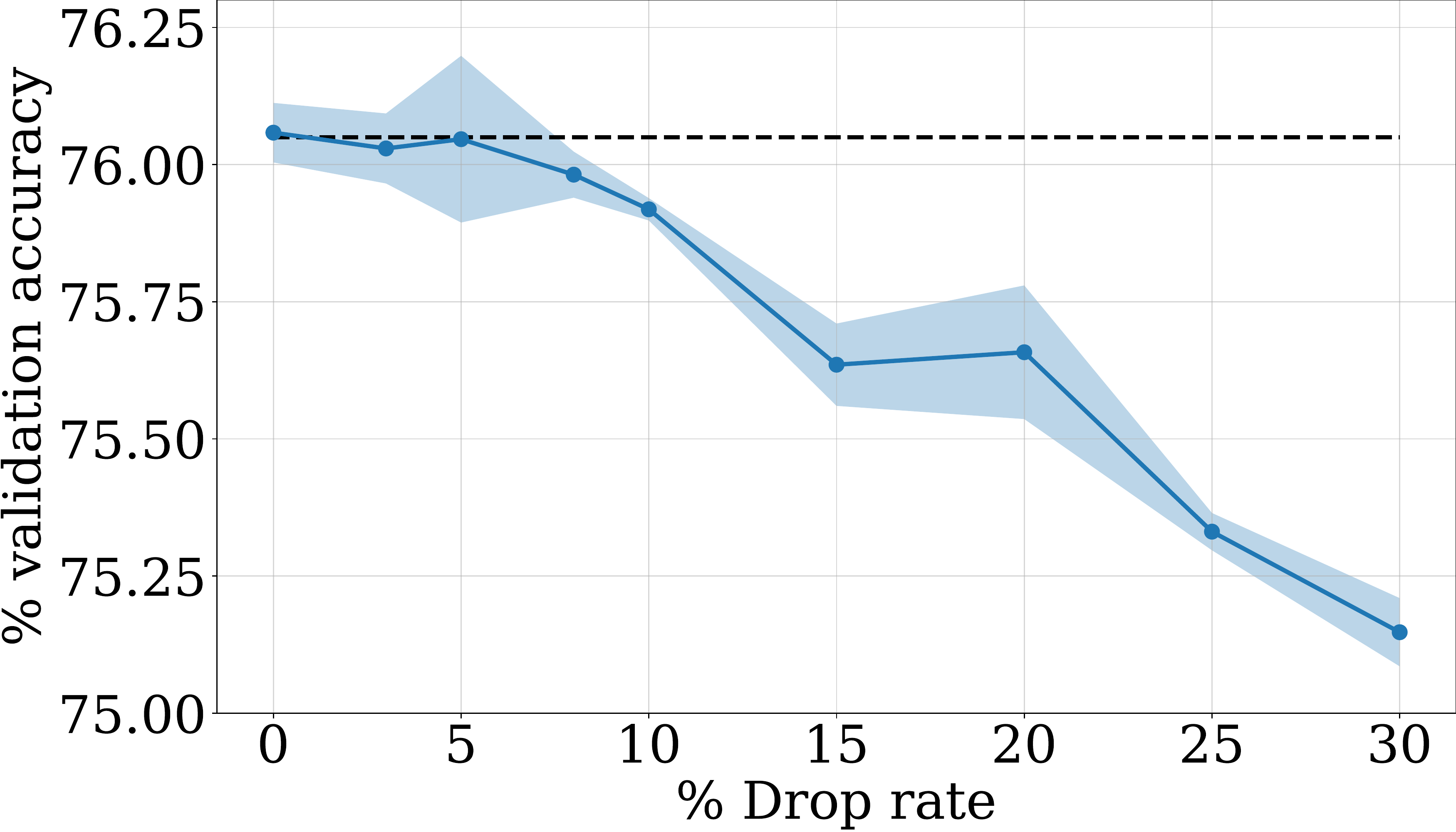}
  \end{subfigure}
  \caption{\textbf{Generalization over varying drop rates.} Top-1 validation accuracy of ResNet50 trained on ImageNet with varying simulated drop rates. The dashed line is the baseline accuracy without drops. The solid line is the average over 3 runs, and the blue area is the standard deviation. (left) Training regime with SGD \citep{goyal2017accurate}. (right) Training regime with LARS \citep{mlperf}. Up to a 10\% drop rate, there is a negligible accuracy deterioration.}
  \label{fig:resnet50_accuracy}
\end{figure}

\textbf{Learning rate correction.} Previous works showed that the learning rate should be scaled with respect to the batch size \citep{hoffer2017train, goyal2017accurate}. With a stochastic batch size and specifically \textit{DropCompute}, it is possible that a learning rate correction should be considered to maintain accuracy with the same number of steps. We examine such corrections when training with stochastic batch size. First, we decrease the learning rate by a constant factor, equal to the average drop rate. Specifically, for an average drop rate $P_{\text {drop}}\in[0,1]$ we multiply the learning rate by $(1-P_{\text {drop}})$. A different correction we consider is a stochastic correction, such that in each step we divide the gradients by the computed batch size, instead of the original batch size. This result in a different normalization in each step depending on the actual dropped samples. We note that for the latter, the workers have to synchronize the computed batch of each worker at each step. This is generally can be done during the \textit{AllReduce}, with negligible overhead. We repeat the training of ResNet50 on ImageNet as described in \citet{goyal2017accurate} to evaluate the generalization without correction and with the two suggested corrections. We use 128 workers, batch size 8192, and use ghost batch norm (GBN) \citep{hoffer2017train} to match batch normalization of 32 samples and restore the results in \citet{goyal2017accurate}. As can be seen in Figure \ref{fig:resnet50_lr_correction}, for low drop rates, there is no superior correction method, and no correction generally achieves the same generalization. Yet, it is possible that a learning rate correction could potentially improve generalization on a different task or with a different optimizer.

\begin{figure}[h!]
  \centering
  \begin{subfigure}[]{0.49\textwidth}
    \includegraphics[width=\textwidth]{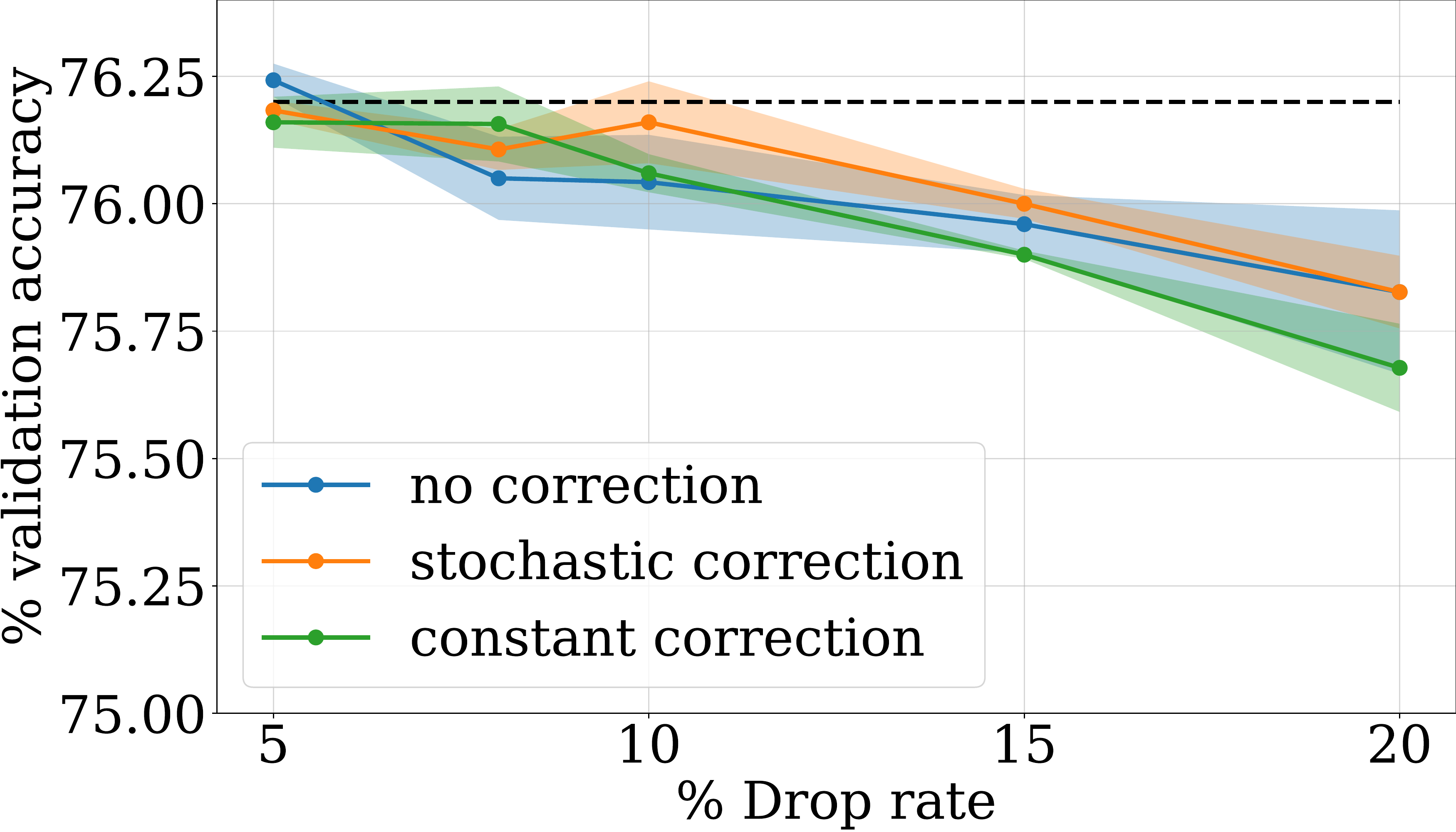}
  \end{subfigure}
  \caption{\textbf{Learning rate correction is not necessary for low drop rates.} Top-1 validation accuracy of ResNet50 trained on ImageNet with varying simulated drop rates, using the learning-rate correction methods described in \ref{appendix:image_classification}. The dashed line is the baseline accuracy without drops. Each solid line is the average over 3 runs, and the area around it is the standard deviation. Up to a 10\% drop rate, there is a negligible accuracy deterioration regardless of the correction applied.}
  \label{fig:resnet50_lr_correction}
\end{figure}

\subsection{Local-SGD} \label{app:local_sgd}
Periodic synchronization methods, such as Local-SGD, provide better scalability properties than synchronous methods. By exchanging parameters less frequently, communication overhead is mitigated. For compute variance and straggling workers in particular, the robustness of these methods greatly depends on the distribution of the compute time between workers. For example, when straggling workers appear randomly with homogeneous distribution, Local-SGD can mitigate the straggling workers slowdowns to some extent; this is because of the amortization effect in synchronizing periodically once every several steps. On the other hand, if straggling workers appear from a small set of workers such as a single server, a realistic scenario, Local-SGD acts more closely to synchronous training as the worst-case scenario is when a single worker always straggling behind. \textit{DropCompute} can be easily integrated with Local-SGD by leveraging periodic synchronization instead of gradient accumulations. We implement \textit{DropCompute} on top of Local-SGD by comparing the compute time with a threshold at each local step. We show that when stragglers are apparent, \textit{DropCompute} can improve the robustness of Local-SGD. We randomly slow down workers to simulate stragglers in two scenarios as described in Figure \ref{fig:local_sgd}. The experiment setting is 32 workers training on ResNet50 and ImageNet. At each local step, each worker is selected to be straggler with a $4\%$ chance. This way, there is at least 1 straggler for each local step on average. We measure relative speedup compared to synchronous training in terms of step time, both for Local-SGD and with \textit{DropCompute} on top of Local-SGD. As can be seen, with \textit{DropCompute} (set to $~6.2\%$ drop rate in this experiment) we improve the robustness of Local-SGD. \
\begin{figure}[h!]
    \centering
    \begin{subfigure}[]{0.45\textwidth}
        \includegraphics[width=\textwidth]{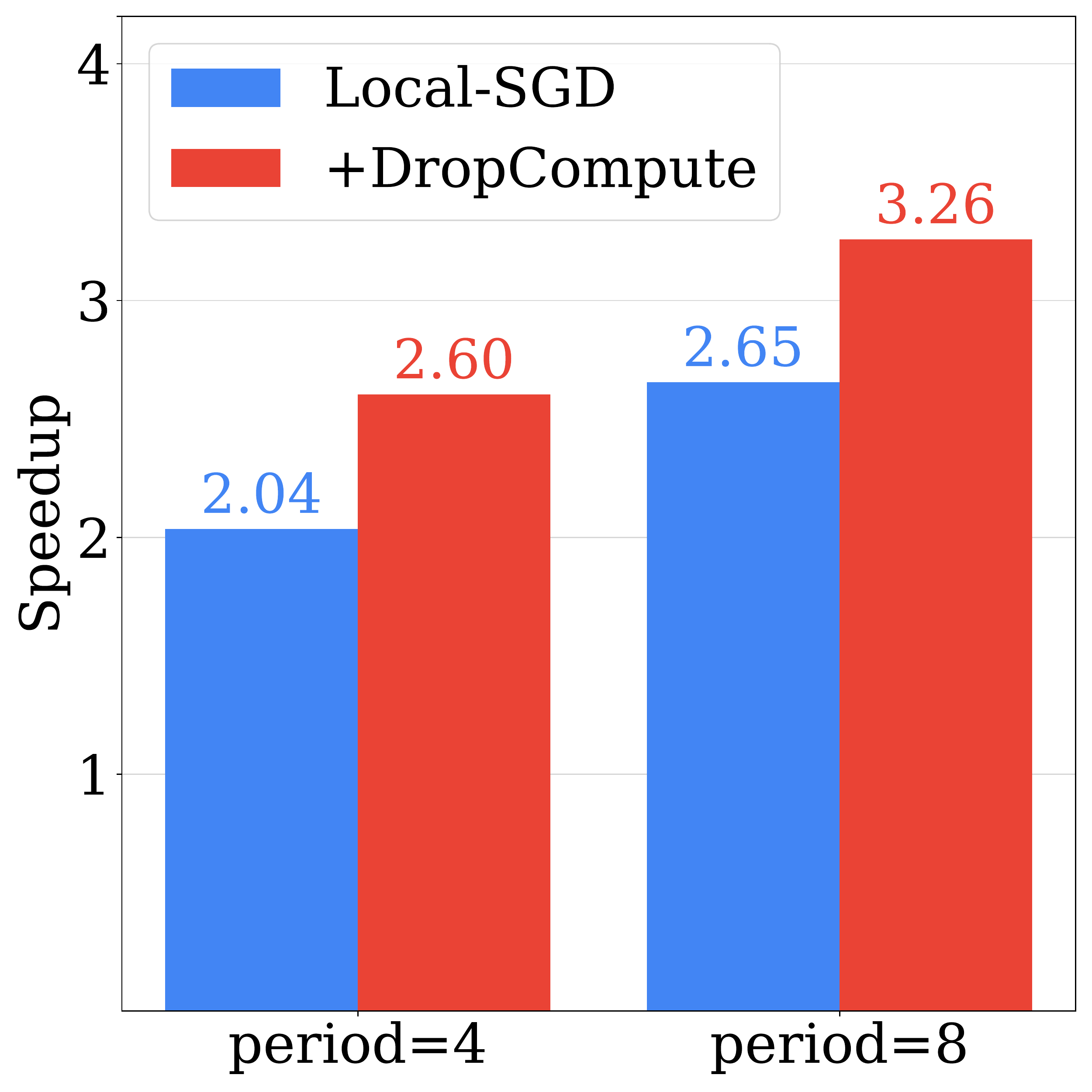}
    \end{subfigure}
    \hfill
    \begin{subfigure}[]{0.45\textwidth}
        \includegraphics[width=\textwidth]{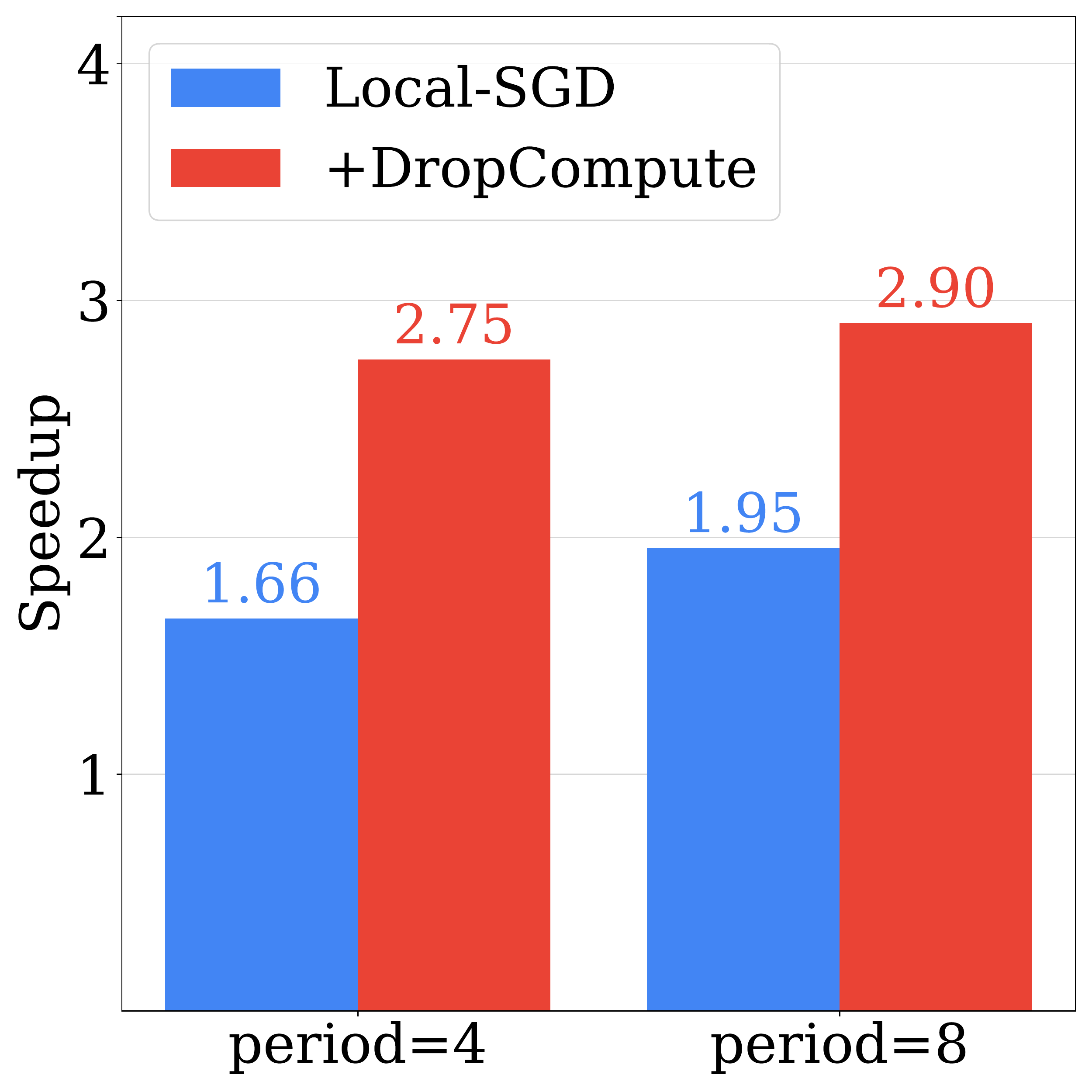}
    \end{subfigure}
    \caption{\textbf{\textit{DropCompute} can be integrated with Local-SGD to improve robustness.} Speedup of each method compared to synchronous training in a straggling workers environment. In each local step workers are randomly drawn to be stragglers. If selected as a straggler, the worker waits 1 second before continuing. The synchronization period of Local-SGD is denoted in the x-axis. (left) Uniform stragglers. (right) Single server stragglers.}
    \label{fig:local_sgd}
\end{figure}

\section{Analyzing the effective speedup using \textit{DropCompute}} 
In this section, we provide more details on the process of choosing the threshold $\tau^*$ that will maximize the effective speedup. We begin by giving technical details on the process used during training, given samples drawn from the empirical distribution of the compute latency. Next, we continue to explore and establish the analytic connection between the latency statistics and the effective speedup.

\subsection{Automatic selection of the drop threshold} \label{appendix:auto_threshold}

In Algorithm \ref{algorithm:optimal_threshold} below we present the algorithm used for automatic selection of the drop threshold $\tau$. In this algorithm, $t_{i,n}^{(m)}$ is the time it takes worker $n$ to process micro-batch $m$ at step $i$ and $T_i^c$ is the time spent on communication for at step $i$.
This data is measured by each worker, and then synchronized between all workers after $I$ iterations.
After the synchronization step each worker will have the same drop threshold $\tau$, which depends on both his own speed and the compute latency of the other workers.
Since $T_i^c$ is used in the normalization of the effective speedup, the chosen threshold takes into account both compute and communication time.

%\textbf{Variable Definitions:}\\ 
%$n$: worker index\\
%$i$: current iteration\\
%$m$: current micro-batch index\\
%$[\tau_0,\tau_1,...]$: potential thresholds to choose from\\

\begin{algorithm}[h!]
    \caption{Automatic choice for the optimal threshold $\tau$}
    \label{algorithm:optimal_threshold}
    \begin{algorithmic}
        \Statex \textbf{Input}:
        \State number of workers $N$
        \State number of iterations $I$
        \State number of micro-batches per step $M$
        \State micro-batch time samples $\{\{t\}_{i,n}^{(m)}\}^{i\in[1:I], n\in[1:N], m\in[1:M]}$
        \State communication time for each iteration $T_i^c$
        \State potential thresholds $[\tau_0, \tau_1, ...]$
        
        \For{$\tau \in [\tau_0,\tau_1,...]$} 
            \For{$i = 1$ to $I$}
                \State Initialize completed micro-batch count:  $\tilde{M}_i(\tau)=0$
                \State Initialize compute step latency for all workers:  $T_i=0$
                \For{$n = 1$ to $N$}
                    \State Initialize single worker step compute latency: $T_{i,n}=0$
                    \For{$m = 1$ to $M$}
                        \State $T_{i,n} \gets T_{i,n} + t_{i,n}^{(m)} $
                        \State $\tilde{M}_i(\tau) \gets \tilde{M}_i(\tau) + \frac{1}{N}\cdot\left\{
                            \begin{array}{lr}
                                1, & \text{if } T_{i,n}<\tau\\
                                0, & \text{otherwise }
                            \end{array}\right\}$
                    \EndFor
                    \State $T_i \gets \text{max}(T_i, T_{i,n})$   \Comment{compute time of the slowest worker at step (i)}
                \EndFor
                \State $S_i(\tau) = \frac{T_i + T_i^c}{\min(\tau,T_i)+T_i^c}\cdot \frac{\tilde{M}_i(\tau)}{M}$ \Comment{Effective speedup for step (i)}
            \EndFor    
            \State $S_{\mathrm{eff}}(\tau) = \frac{1}{I}\sum_{i=1}^{I}{S_i(\tau)}$  \Comment{Mean speedup for threshold ($\tau$)}
        \EndFor
        \State $\tau^* \gets  \text{argmax}_\tau\left(S_{\mathrm{eff}}(\tau) \right)$
    \end{algorithmic}
\end{algorithm}

\subsection{DropCompute speedup analytic analysis} \label{appendix:speedup_analysis}
In this section we further explore the relation between the compute latency distribution and the effective speedup. We will derive a closed-form representation of the effective speedup by making certain assumptions. First we assume that 
\begin{assumption} \label{assu: iid} $t_n^{(m)}$ is i.i.d with finite mean $\mu$ and finite variance $\sigma^2$.\end{assumption}
Note that in the assumption above, for simplicity of presentation we assumed all workers as identical and so $\mu$ and $\sigma^2$ are identical. However, it is possible to derive similar properties with nonidentical workers, each with their own $\mu_n$, $\sigma_n$.
Next, denote the time for completion of micro-batch $m$ as $T_n^{(m)}=\sum_{j=1}^m t_n^{(j)}$. Then, we assume
\begin{assumption} \label{assume:gaussian} $T_n^{(m)}\text{ is Gaussian }\sim\mathcal{N}(m\mu,m\sigma^2)$ for $m>\sqrt{M}\,.$\end{assumption}
This assumption holds in the limit $M\to\infty$ given Assumption \ref{assu: iid}, from the Central Limit Theorem (CLT). Lastly, denoting $\tau$ as the threshold used, we assume
\begin{assumption} \label{assume:tau_limit} $\tau>\frac{M\mu}{2}\,.$\end{assumption} This bound can be considered as the minimum threshold allowed, in order for \textit{DropCompute} to be effective. Taking a lower threshold will result in unacceptable high drop rate. 

Using these assumptions we first derive analytical expression for the iteration time $\mathbb{E}[T]$ and the mean completed number of gradient accumulations $\mathbb{E}[\tilde{M}]$ with \textit{DropCompute}. Then, we combine these expressions to obtain an expression for the mean effective speed $\mathbb{E}[S_{\mathrm{eff}}]$. 

Figure \ref{fig:estimate_s_eff_b} shows an example of how close is the derived expression of $\mathbb{E}[S_{\mathrm{eff}}]$ to the value calculated by using the algorithm  described in section \ref{appendix:auto_threshold}. The `analytical' curve is slightly off due to the inaccuracy of the Gaussian Assumption \ref{assume:gaussian} in calculating $\mathbb{E}[T]$, as we discuss below. We therefore added another curve `analytical given $\mathbb{E}(T)$', which uses same the derived expression for $\mathbb{E}[S_{\mathrm{eff}}]$ but replacing value of $\mathbb{E}[T]$, with the empiric mean: $\overline{T}=\frac{1}{I}\sum_{i=1}^I \max_n\left\{T_{n,i}^{(M)} \right\}$ where: $i\in[1:I]$ are the iterations measured.

\textbf{Iteration time.} 
as written in section \ref{sec:analysis_iteration_time}, the iteration time for all workers is 
$$T=\max_n \left(T_n^{(M)}\right) = T^c + \max_n \left( \sum_{m=1}^M t_n^{(m)} \right)\,.$$
When $T_n^{(M)}\sim\mathcal{N}(M\mu, M\sigma^2)\,.$
the expected value of $T$ can be approximated as \citep{bailey2014pseudomathematics}: 
\begin{equation} \label{eq:approximate_t}
    \mathbb{E}[T] \approx \sqrt{M\sigma^2}\cdot\left( (1-\gamma)\cdot \Phi^{-1}\left(1-\frac{1}{N} \right) + \gamma\cdot \Phi^{-1}\left(1-\frac{1}{e\cdot N} \right)\right) + M\mu + T^c\,.
\end{equation}
We can derive the asymptotic behavior of Eq. \ref{eq:approximate_t} by:
$$ \Phi(x) = \frac{1}{2}+\frac{1}{2}\text{erf}\frac{x}{\sqrt{2}} \sim 1-\frac{1}{x\sqrt{2\pi}}e^{-x^2/2}$$
$$\Downarrow$$
$$\text{log}(1-\Phi(x))\sim -\frac{x^2}{2}-\text{log}(x\sqrt{2\pi})\sim -\frac{x^2}{2},\;\;\; x \rightarrow\infty$$
$$\Downarrow$$
$$\Phi^{-1}(1-y)\sim\sqrt{-2\text{ log }y},\;\;\; y \rightarrow0^+$$
When plugging Equation \ref{eq:approximate_t} into this asymptotic approximation we are left with $\mathbb{E}[T]=\Theta(\sqrt{\text{log} N})$.

It is worth noting that the distribution of $T$ is mostly affected by the tail of distribution of $T_n=T_n^{(M)}$ (as a consequence of Equation \ref{eq:latency_pdf}). Therefore, in practice the Gaussian Assumption \ref{assume:gaussian}, and therefore Equation \ref{eq:approximate_t}, can be inaccurate, especially for small $M$ and large $N$. It is therefore more accurate to use the real value of $\mathbb{E}[T]$, measured without \textit{DropCompute}, to estimate the potential effective speedup. An example for the inaccuracy of this approximation can be seen in Figure \ref{fig:estimate_s_eff_b}, when $T_n^{(M)}$ does not follow a normal distribution.

\textbf{Completed micro-batches.}
The average number of micro-batch computed by a single worker $n$ when using \textit{DropCompute} with threshold $\tau$, is:
$$\tilde{M}(\tau)=\frac{1}{N}\sum_{n=1}^N\sum_{m=1}^M \left\{
                            \begin{array}{lr}
                                1, & \text{if } T_n^{(m)}<\tau\\
                                0, & \text{otherwise }
                            \end{array}\right\} \,.$$
Its expected value can be written as:
$$ \mathbb{E}[\tilde{M}(\tau)]=\sum_{m=1}^M P\left(T_n^{(m)}<\tau\right) \,.$$

In order to use assumption \ref{assume:gaussian}, we can split $\tilde{M}$ into 2 sums and derive a closed-form formula for the expected value:
\begin{equation} \label{eq:tilde_m_split}
    \tilde{M}(\tau)=\sum_{m=1}^{ \lfloor{{\sqrt{M}}}\rfloor} P\left(T_n^{(m)}<\tau\right) + \sum_{m=\lceil{{\sqrt{M}}}\rceil}^M P\left(T_n^{(m)}<\tau\right)\,.
\end{equation}
For the right term we can use assumption \ref{assume:gaussian} so that $P(T_n^{(m)}<\tau)=\Phi\left(\frac{\tau - m\cdot \mu}{\sqrt{m\cdot \sigma^2}} \right)$.
For the left term, when $m<\sqrt{M}$ we use Markov inequality and assumption \ref{assume:tau_limit} to show that
$$ 0\leq P(T_n^{(m)}>\tau)\leq \frac{\mathbb{E}[T_n^{(m)}]}{\tau}=\frac{m\mu}{\tau}\leq \frac{2m}{M}\,.$$
%\leq \frac{2}{\sqrt{M}}=O(M^{-\frac{1}{2}})\,.$$
In other words, when using \textit{DropCompute} with low drop rates, $P(T_n^{(m)}<\tau)$ is very high for $m<\sqrt{M}$. 
The Gaussian approximation for $m<\sqrt{M}$ diminishes exponentially when increasing $M$, as seen by applying Chernoff bound:
$$ 0\leq P(Z^{(m)}>\tau)\leq e^{-\frac{(\tau-m\mu)^2}{2m\sigma^2}}\leq e^{-\frac{(M/2-m)^2\mu^2}{2m\sigma^2}}$$
where $Z^{(m)}\sim\mathcal{N}(m\mu, m\sigma^2)$.
Therefore, the error resulting in replacing $T_n^{(m)}$ with a Gaussian approximation, is bounded:
$$ P(T_n^{(m)}<\tau) - P(Z^{(m)}<\tau) = P(Z^{(m)}>\tau) -P(T_n^{(m)}>\tau) $$
$$\Downarrow$$
$$ -\frac{2m}{M} \leq -P(T_n^{(m)}>\tau) \leq  P(T_n^{(m)}<\tau) - P(Z^{(m)}<\tau) \leq P(Z^{(m)}>\tau)\leq e^{-\frac{(M/2-m)^2\mu^2}{2m\sigma^2}}$$
$$\Downarrow$$
$$ -\frac{2m}{M} + P(Z^{(m)}<\tau) \leq  P(T_n^{(m)}<\tau) \leq P(Z^{(m)}<\tau) + e^{-\frac{(M/2-m)^2\mu^2}{2m\sigma^2}}$$
Plugging these inequalities into the left term in equation \ref{eq:tilde_m_split} gives us:

\begin{equation} \label{eq:tilde_m_split_ineq}
\begin{split}
\sum_{m=1}^{ \lfloor{{\sqrt{M}}}\rfloor} P\left(T_n^{(m)}<\tau\right) \leq& \sum_{m=1}^{ \lfloor{{\sqrt{M}}}\rfloor} \left(P\left(Z^{(m)}<\tau\right)+e^{-\frac{(M/2-m)^2\mu^2}{2m\sigma^2}} \right) \\ \\
\sum_{m=1}^{ \lfloor{{\sqrt{M}}}\rfloor} P\left(T_n^{(m)}<\tau\right) \geq& \sum_{m=1}^{ \lfloor{{\sqrt{M}}}\rfloor}\left( P\left(Z^{(m)}<\tau\right) -\frac{2m}{M}\right) = O(1) +\sum_{m=1}^{ \lfloor{{\sqrt{M}}}\rfloor} P\left( Z^{(m)}<\tau\right)
\end{split}
\end{equation}
%and:
%$$ -1 -\frac{1}{\sqrt{M}} +\sum_{m=1}^{ \lfloor{{\sqrt{M}}}\rfloor} P\left(Z^{(m)}<\tau\right) \leq \sum_{m=1}^{ \lfloor{{\sqrt{M}}}\rfloor} P\left(T_n^{(m)}<\tau\right) \leq\sum_{m=1}^{ \lfloor{{\sqrt{M}}}\rfloor} P\left(Z^{(m)}<\tau\right)+e^{-\frac{(M/2-m)^2\mu^2}{2m\sigma^2}}$$

Combining equations \ref{eq:tilde_m_split},\ref{eq:tilde_m_split_ineq} we can write the expected value as
\begin{equation} \label{eq:approximate_m}
\mathbb{E}[\tilde{M}(\tau)]=\sum_{m=1}^M P\left(T_n^{(m)}<\tau\right) = O(1)+ \sum_{m=1}^M \Phi\left(\frac{\tau - m\cdot \mu}{\sqrt{m\cdot \sigma^2}} \right)\,.
\end{equation}

\textbf{Effective speedup.}
As seen in section \ref{section:automatic_threshold}, we define the effective speedup as
$$ S_{\mathrm{eff}}(\tau)= \frac{\tilde{M}(\tau)(T+T^c)}{M\cdot(\min\left\{\tau,T\right\}+T^c)}\,. $$
We are interested in calculating the expected value for the effective speedup, and in order to use the formulations in equations \ref{eq:approximate_t},\ref{eq:approximate_m} we first need to show that $\mathbb{E}[\tilde{M}(\tau)\cdot T] \approx \mathbb{E}[\tilde{M}(\tau)]\cdot \mathbb{E}[T]$. We examine
$$ \mathbb{E}[\tilde{M}T]=\mathbb{E}[T(\mathbb{E}[\tilde{M}] + \tilde{M}- \mathbb{E}[\tilde{M}])]=\mathbb{E}[\tilde{M}]\mathbb{E}[T]+\mathbb{E}[T(\tilde{M}-\mathbb{E}[\tilde{M}])]$$
Applying Cauchy–Schwarz inequality we get
$$|\mathbb{E}[T(\tilde{M}-\mathbb{E}[\tilde{M}])]| \leq \sqrt{\mathbb{E}[T^2]\mathbb{E}[(\tilde{M}-\mathbb{E}[\tilde{M}])^2]}= \sqrt{\mathbb{E}[T^2]\frac{\sigma_{\tilde{M}_n}^2}{N}}=O(N^{-\frac{1}{2}}) \,,$$
%\xrightarrow[N\to\infty]{}0$$
where $\sigma_{\tilde{M}_n}^2$ denotes the variance of 
$\tilde{M}_n(\tau)=\sum_{m=1}^M \left\{ \begin{array}{lr} 1, & \text{if } T_n^{(m)}<\tau\\ 0, & \text{otherwise } \end{array}\right\}$.
Hence:
$$  \mathbb{E}[\tilde{M}T] = \mathbb{E}[\tilde{M}]\mathbb{E}[T] + O(N^{-\frac{1}{2}})$$
We can now write the expected value for the effective speedup as:
\begin{equation} \label{eq:eff_speedup}
    \mathbb{E}[S_\mathrm{eff}] =
    \frac{\sum_{m=1}^M \Phi\left(\frac{\tau - m\cdot \mu}{\sqrt{m\sigma^2}} \right) }{M} \cdot
    \frac{\mathbb{E}[T]}{\min(\tau, \mathbb{E}[T])+T^c}+ O(M^{-1}+M^{-1}N^{-\frac{1}{2}})
\end{equation}
\begin{equation} \label{eq:T_approximation}
     \mathbb{E}[T]\approx\sqrt{M\sigma^2} \left( (1-\gamma)\cdot \Phi^{-1}\left(1-\frac{1}{N} \right) + \gamma\cdot \Phi^{-1}\left(1-\frac{1}{e N}\right)\right) + M\mu + T^c
\end{equation}
As mentioned above, when the Gaussian Assumption \ref{assume:gaussian} is inaccurate it may be useful to plug instead in the empirical value for $\mathbb{E}[T]$ in equation \ref{eq:eff_speedup} in order to get a more accurate estimation of $\mathbb{E}[S_\mathrm{eff}]$.

\textbf{Finding $\tau^*$.} The optimal threshold $\tau^*$ can be chosen as:
$$ \tau^* = \text{argmax}_\tau \mathbb{E}[S_\mathrm{eff}(\tau)] =\text{argmax}_\tau \left(\frac{1}{\tau + T^c}\cdot \sum_{m=1}^M \Phi\left(\frac{\tau - m\cdot \mu}{\sqrt{m\cdot \sigma^2}} \right)\right)$$

By using the above derivations, we can utilize $\mu$, $\sigma$, $T^c$ to understand the potential value of \textit{DropCompute}. This can be done without actually training and measuring the empiric distribution of $t_n^{(m)}$ as done in appendix section \ref{appendix:auto_threshold}. We note that finding $\tau^*$ does not require any estimation of $T$ and can be done without any statistics that originate from a large scale training session.

\subsection{Additive noise analysis} \label{appendix:noise_analysis}
As a conclusion of the previous sections, we understand that the effectiveness of \textit{DropCompute} is mostly due to the behavior of the stochastic latency of each worker. To analyze this phenomenon we simulate a training of multiple workers using the scheme presented in section \ref{appendix:runtime_performace} with various additive noise types.
As shown in figures \ref{fig:noise_type}, \ref{fig:noise_scale}, the ratio $\mathbb{E}[T]/\mathbb{E}[T_i]$ is good indicator for determining the potential of \textit{DropCompute} on a given training setting. High ratios indicate a gap between the step time for a single worker and the step time for multiple workers, that can be compensated by using \textit{DropCompute}.
% In this section we will attempt to analyze the effect of different types of additive noise as well as the effect of the noise variance. 

% \begin{figure}[h!]
%   \includegraphics[width=\textwidth]{icml_figures/noise/noise_types_annotated2.png}

%   \caption{\textbf{The noise distribution type has an impact on the effectively of \textit{DropCompute}} Simulated scale graphs for a run with 12 accumulations. each accumulation takes $0.45 + \epsilon$ where $\epsilon$ is stochastic with a mean of $0.225$ and variance of $0.05$. Different sub-graphs exhibit different distributions for $\epsilon$. The bottom-left graph was drawn using the approximation from Equation \ref{eq:auto_eff}.  }
%   \label{fig:noise_type}
% \end{figure}

\begin{figure}[!h]
  \centering
  \begin{subfigure}[]{0.32\textwidth}
      \includegraphics[width=\textwidth]{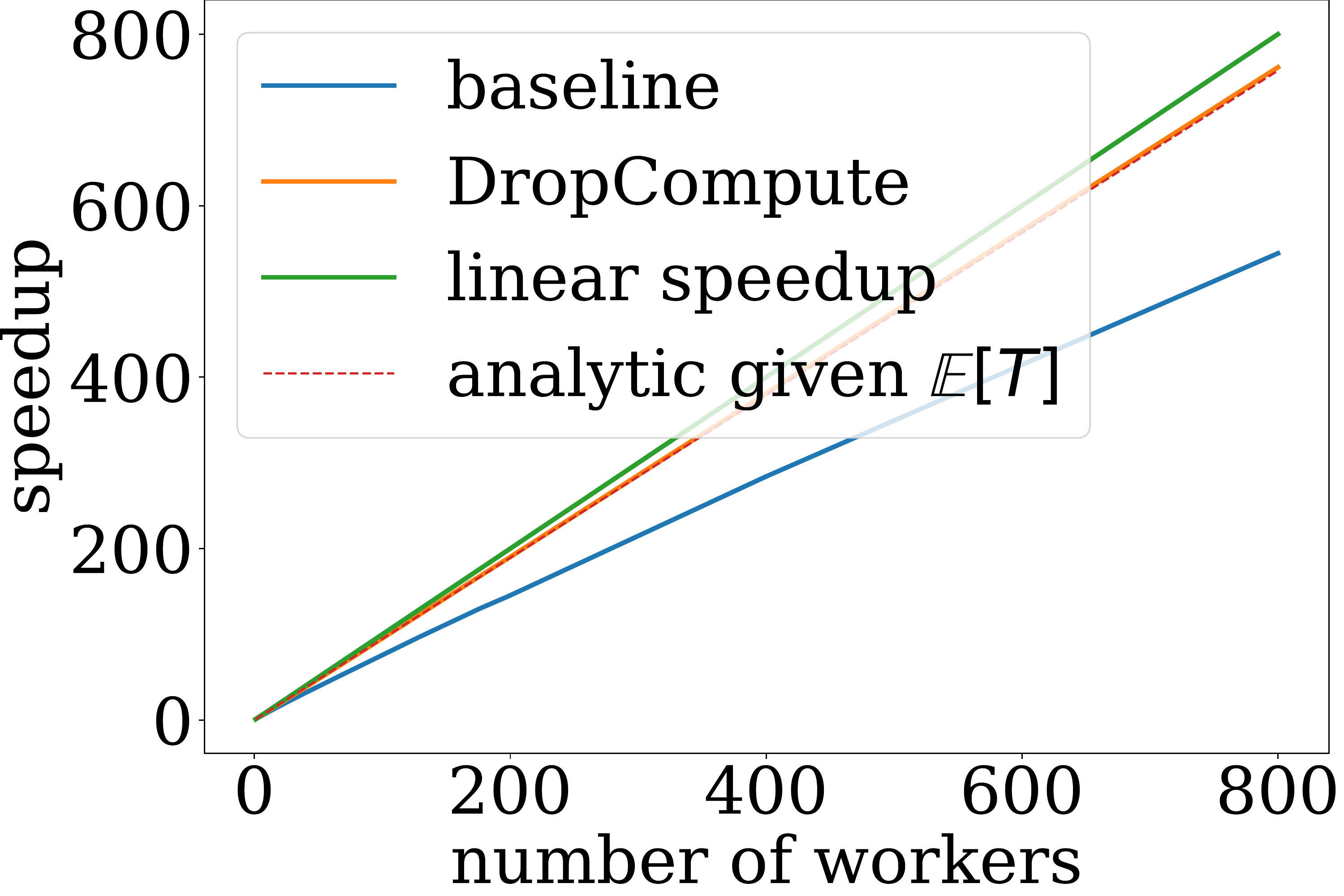}
      \caption{}
  \end{subfigure}
  \hfill
  \begin{subfigure}[]{0.32\textwidth}
      \includegraphics[width=\textwidth]{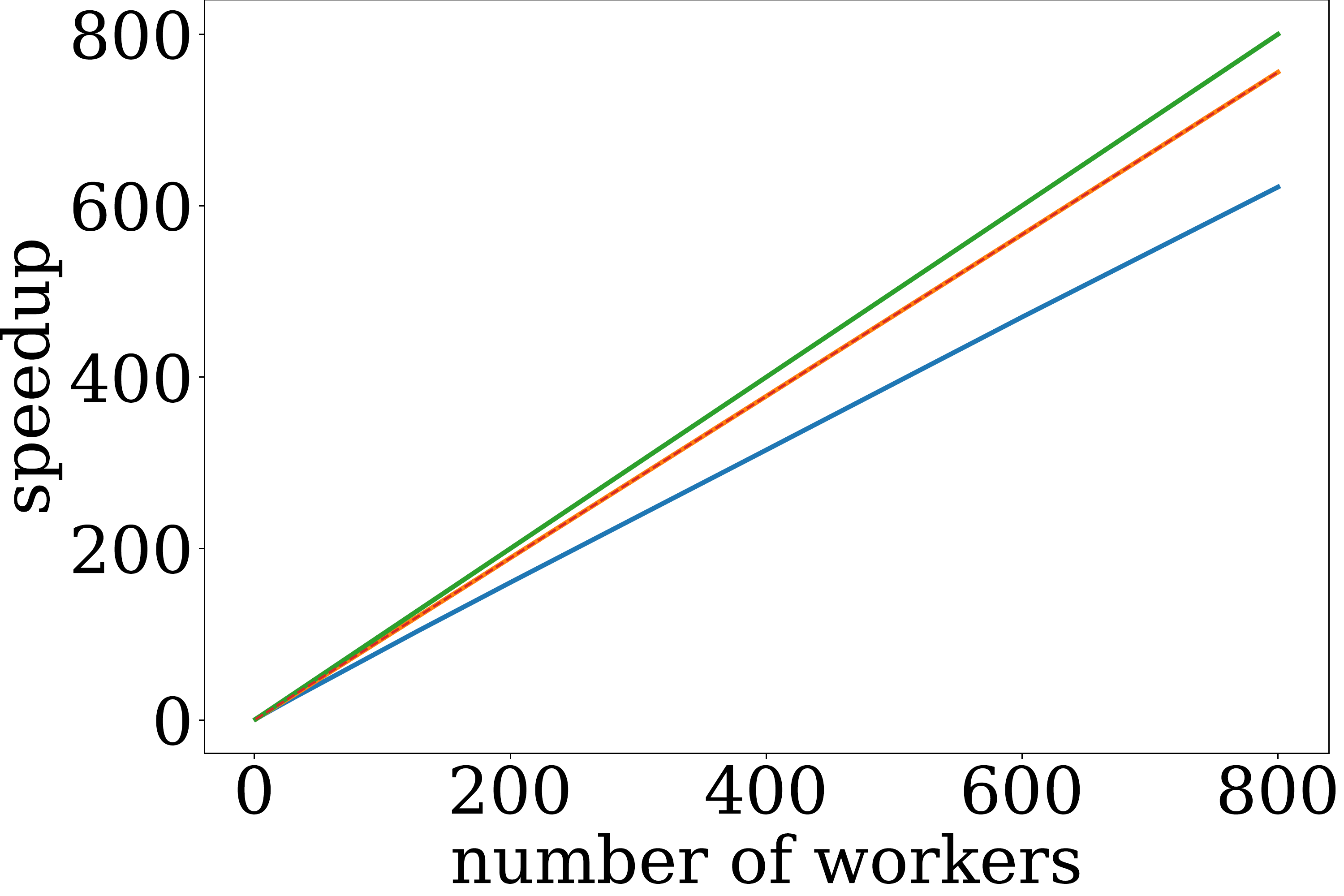}
      \caption{}
  \end{subfigure}
  \hfill
  \begin{subfigure}[]{0.32\textwidth}
      \includegraphics[width=\textwidth]{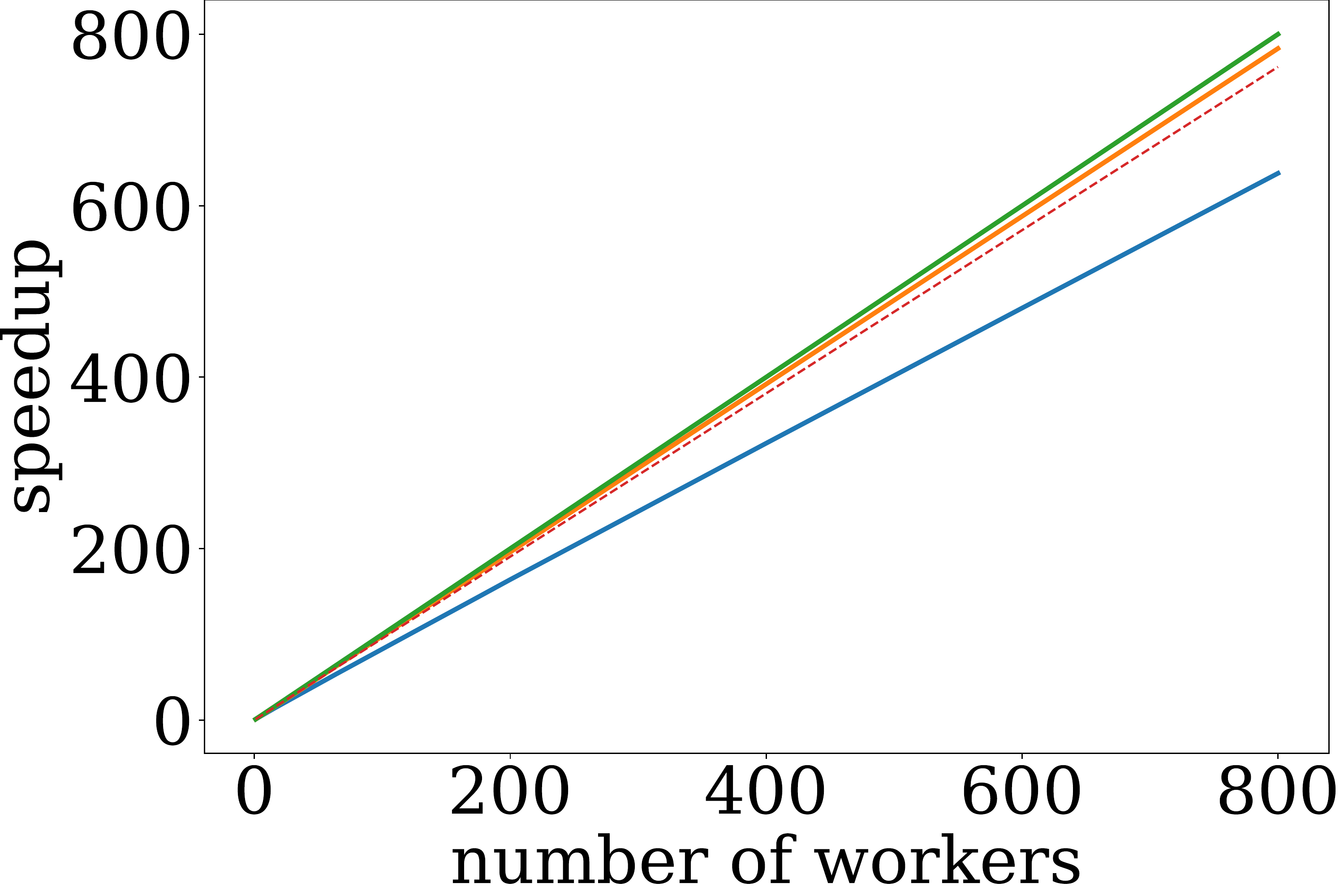}
      \caption{}
  \end{subfigure}
  \begin{subfigure}[]{0.32\textwidth}
      \includegraphics[width=\textwidth]{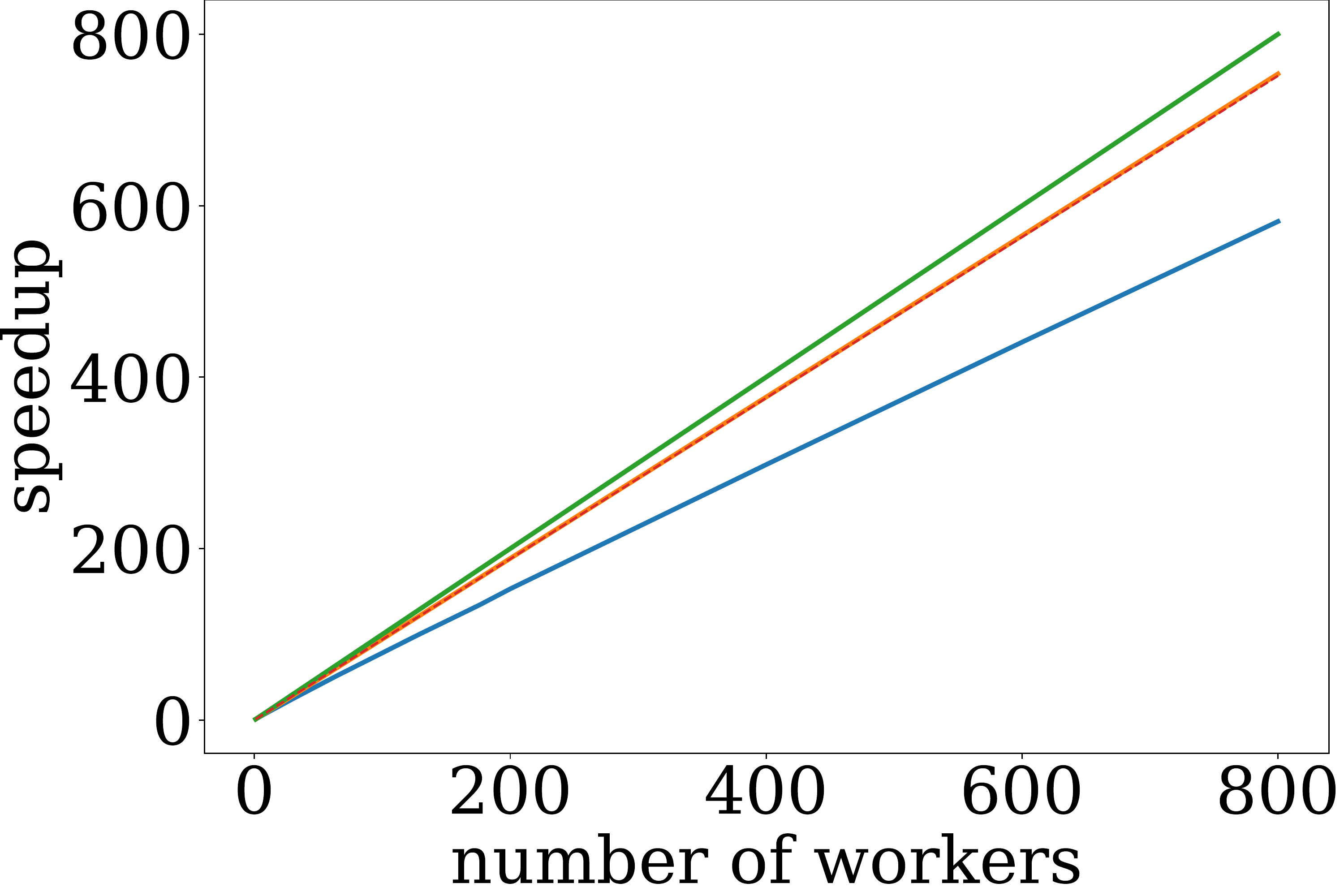}
      \caption{}
  \end{subfigure}
  \hfill
  \begin{subfigure}[]{0.32\textwidth}
      \includegraphics[width=\textwidth]{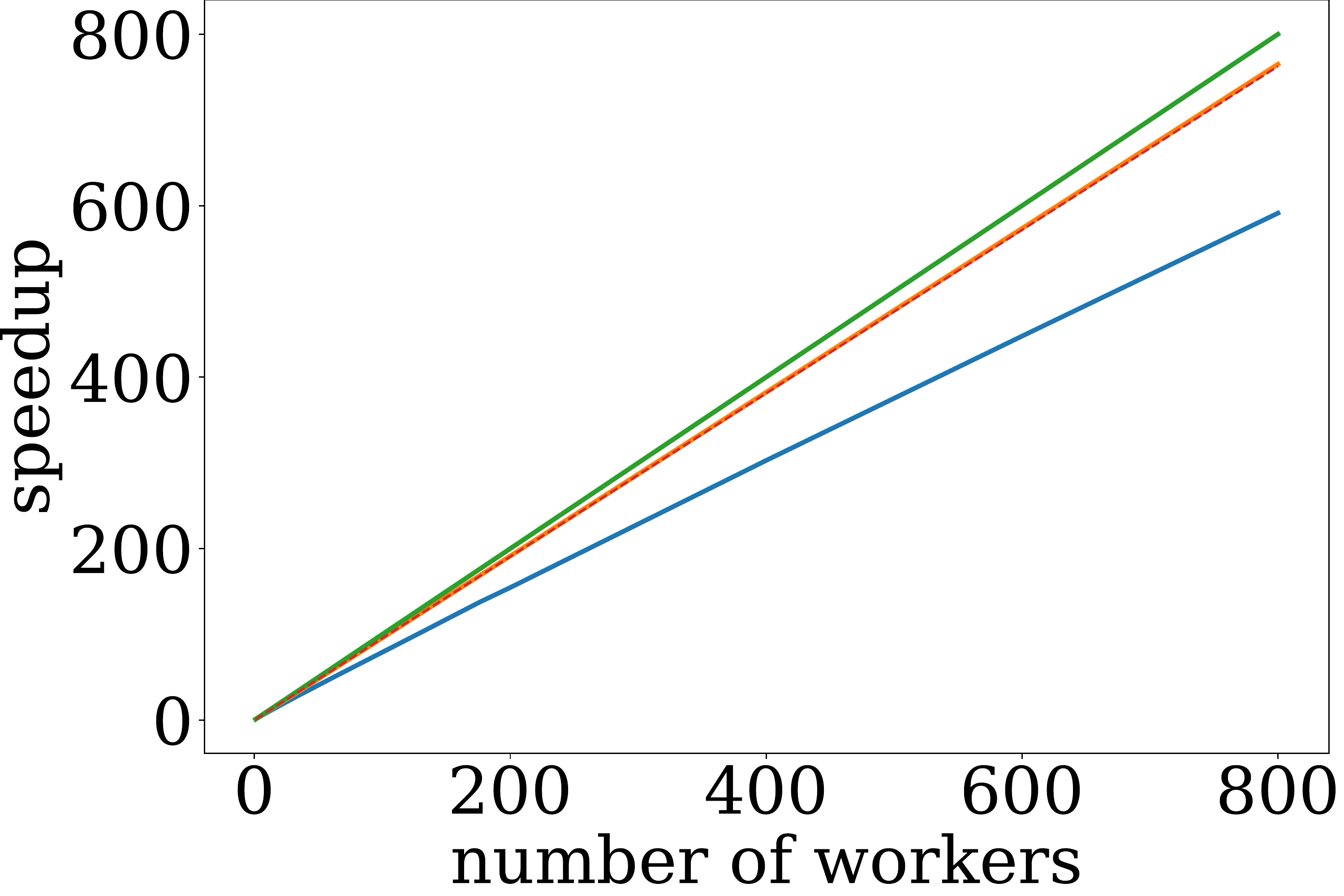}
      \caption{}
  \end{subfigure}
  \hfill
  \begin{subfigure}[]{0.32\textwidth}
      \includegraphics[width=\textwidth]{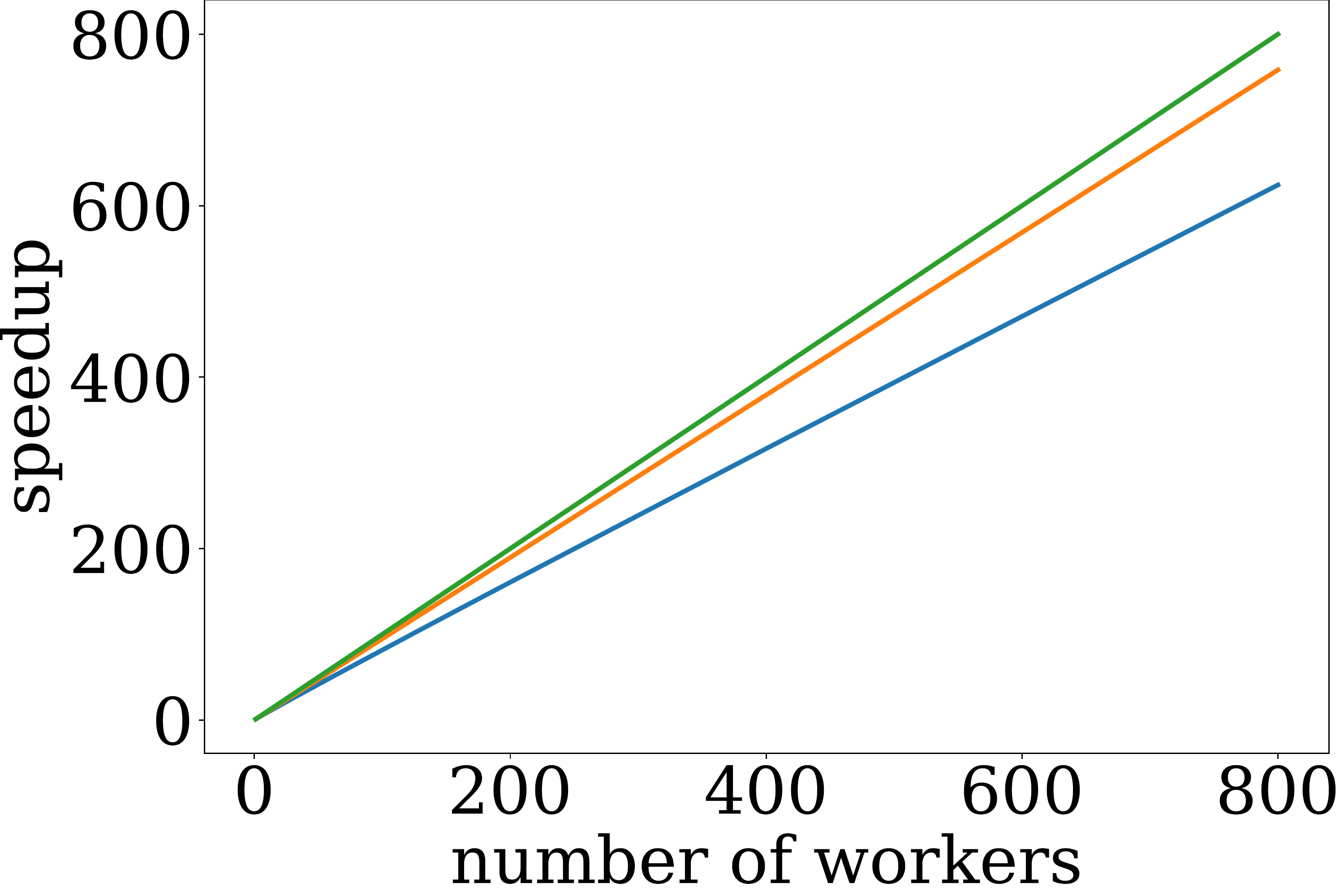}
      \caption{}
  \end{subfigure}
  \caption{\textbf{The noise distribution type impacts the effectiveness of \textit{DropCompute}.} Simulated scale graphs for a run with 12 accumulations. 
  The duration of each accumulation is set to $0.45 + \epsilon$ seconds
  Different sub-graphs exhibit different distributions for $\epsilon$. The bottom-right graph was drawn using the approximation from equation \ref{eq:auto_eff}.
  }
  \begin{tabular}{|c|cccc|c|}
    \hline 
    figure  & Mean($\epsilon$) & Var($\epsilon$) & $\epsilon$ distribution&&$\mathbb{E}[T]/\mathbb{E}[T_i]$ \\\hline
    a & 0.225 & 0.05 & lognormal & $LN(\mu=-1.84, \sigma=0.83)$ & 1.496\\
    b & 0.225 & 0.05 & normal & $\mathcal{N}(\mu=0.23, \sigma=0.22)$ & 1.302\\
    c & 0.225 & 0.05 & bernoulli & $0.45 Br(p=0.5)$ & 1.283\\
    d & 0.225 & 0.05 & exponential & $exp(\lambda=4.47)$ & 1.386\\
    e & 0.225 & 0.05 & gamma & $\gamma(\alpha=1, \beta=4.5)$ & 1.39\\
    \hline
    \end{tabular}
  \label{fig:noise_type}
\end{figure}

\begin{figure}[!h]
  \centering
  \begin{subfigure}[]{0.32\textwidth}
      \includegraphics[width=\textwidth]{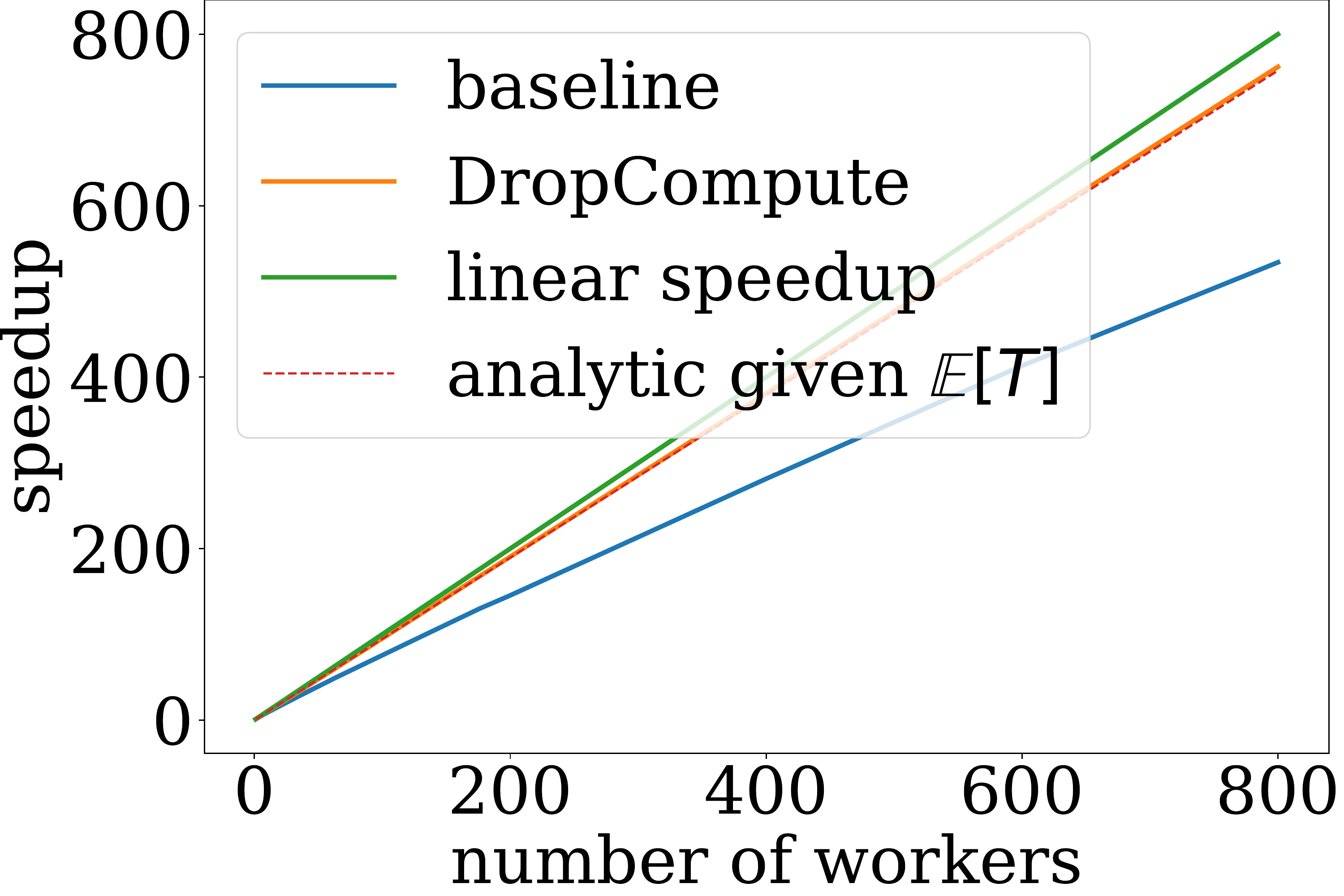}
      \caption{}
  \end{subfigure}
  \hfill
  \begin{subfigure}[]{0.32\textwidth}
      
      \includegraphics[width=\textwidth]{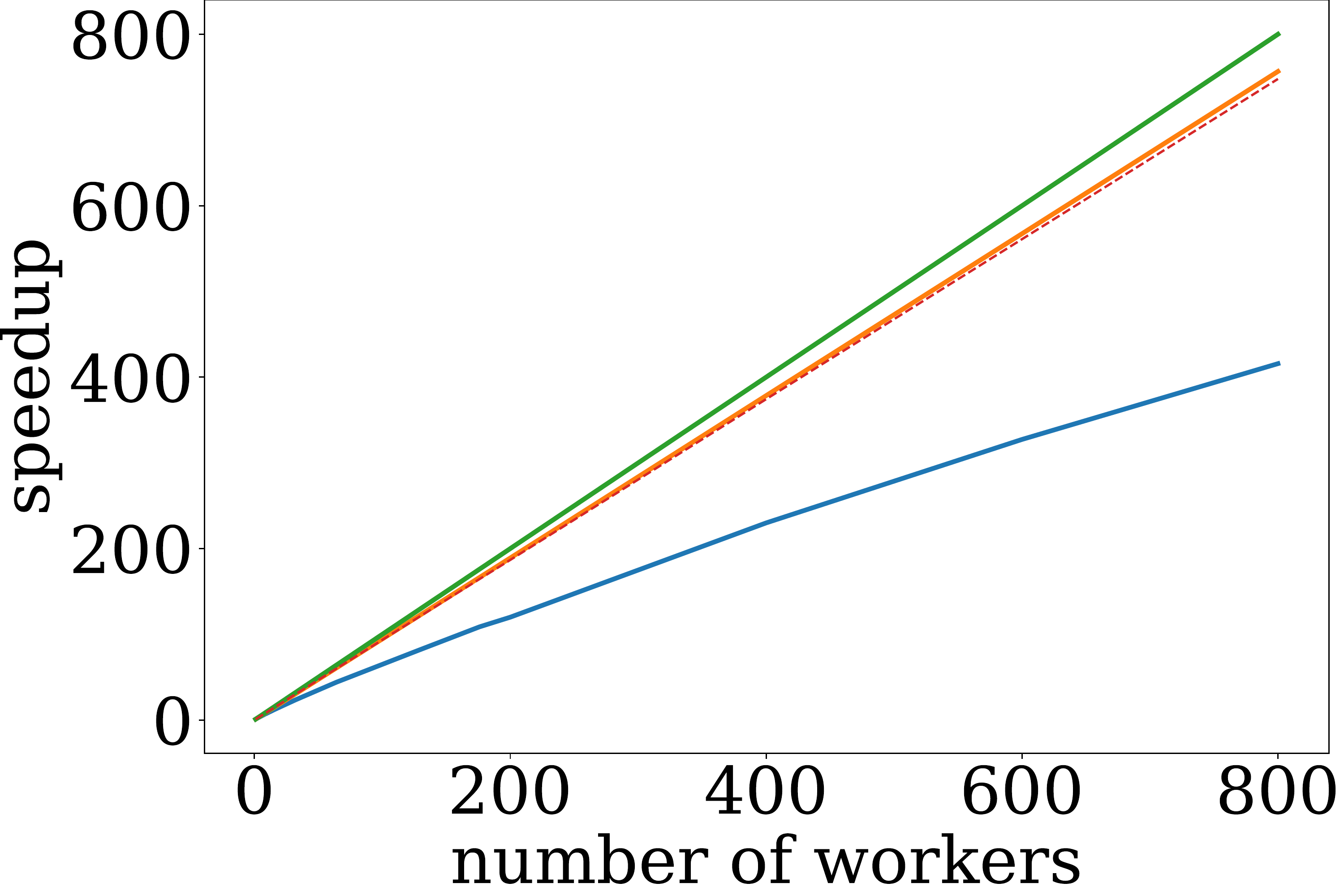}
      \caption{}
  \end{subfigure}
  \hfill
  \begin{subfigure}[]{0.32\textwidth}
      \includegraphics[width=\textwidth]{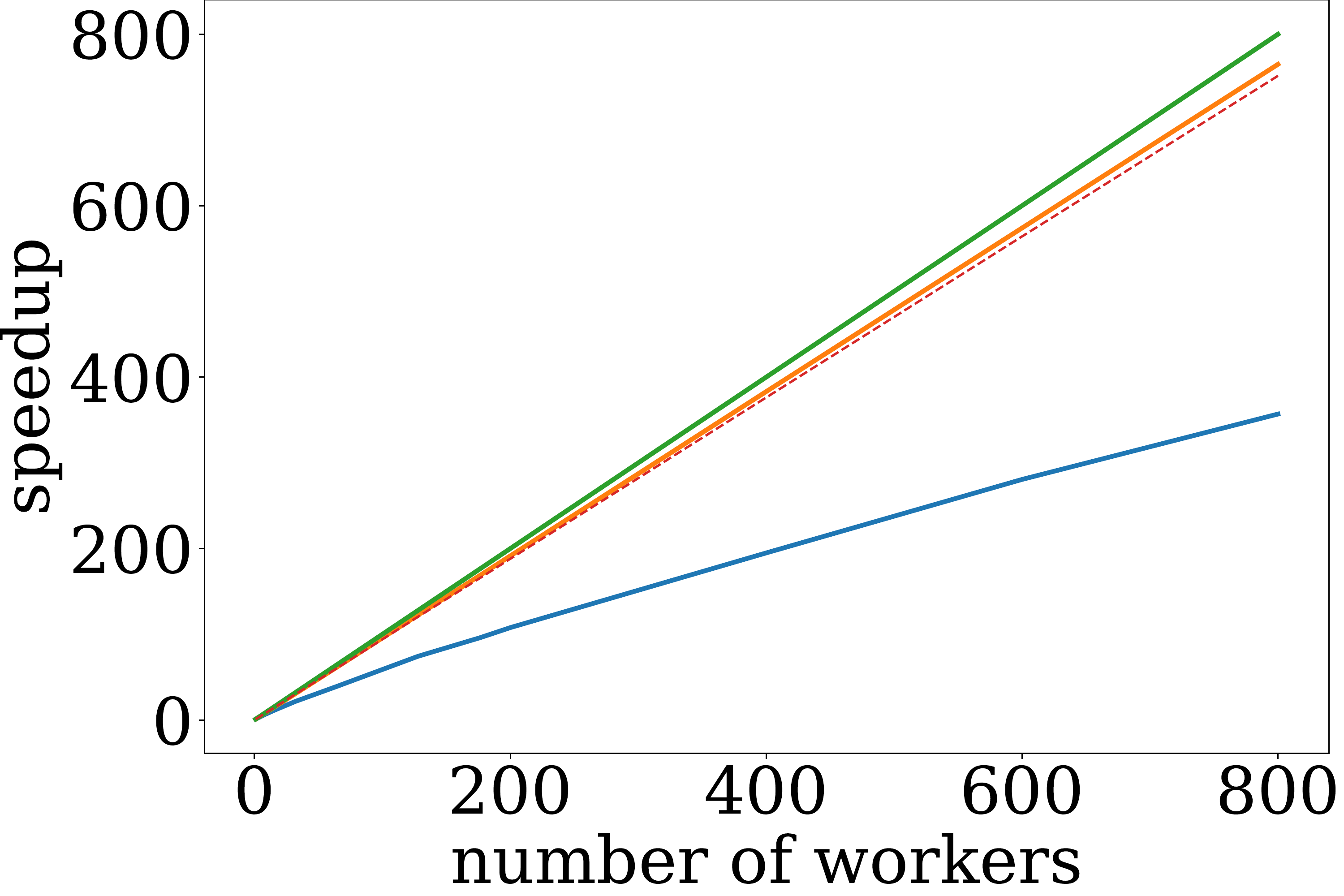}
      \caption{}
  \end{subfigure}
  \begin{subfigure}[]{0.32\textwidth}
      \includegraphics[width=\textwidth]{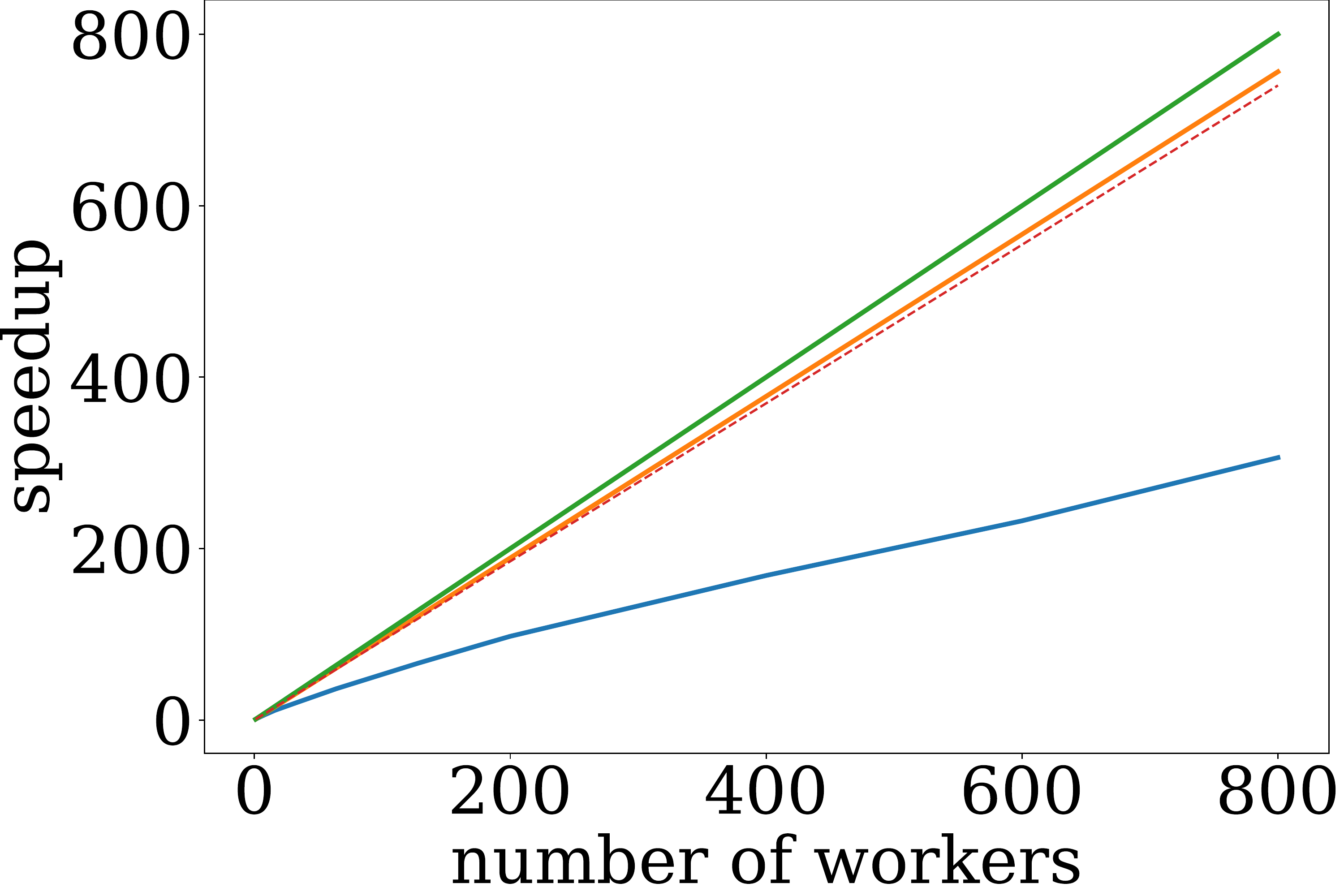}
      \caption{}
  \end{subfigure}
  \hfill
  \begin{subfigure}[]{0.32\textwidth}
      \includegraphics[width=\textwidth]{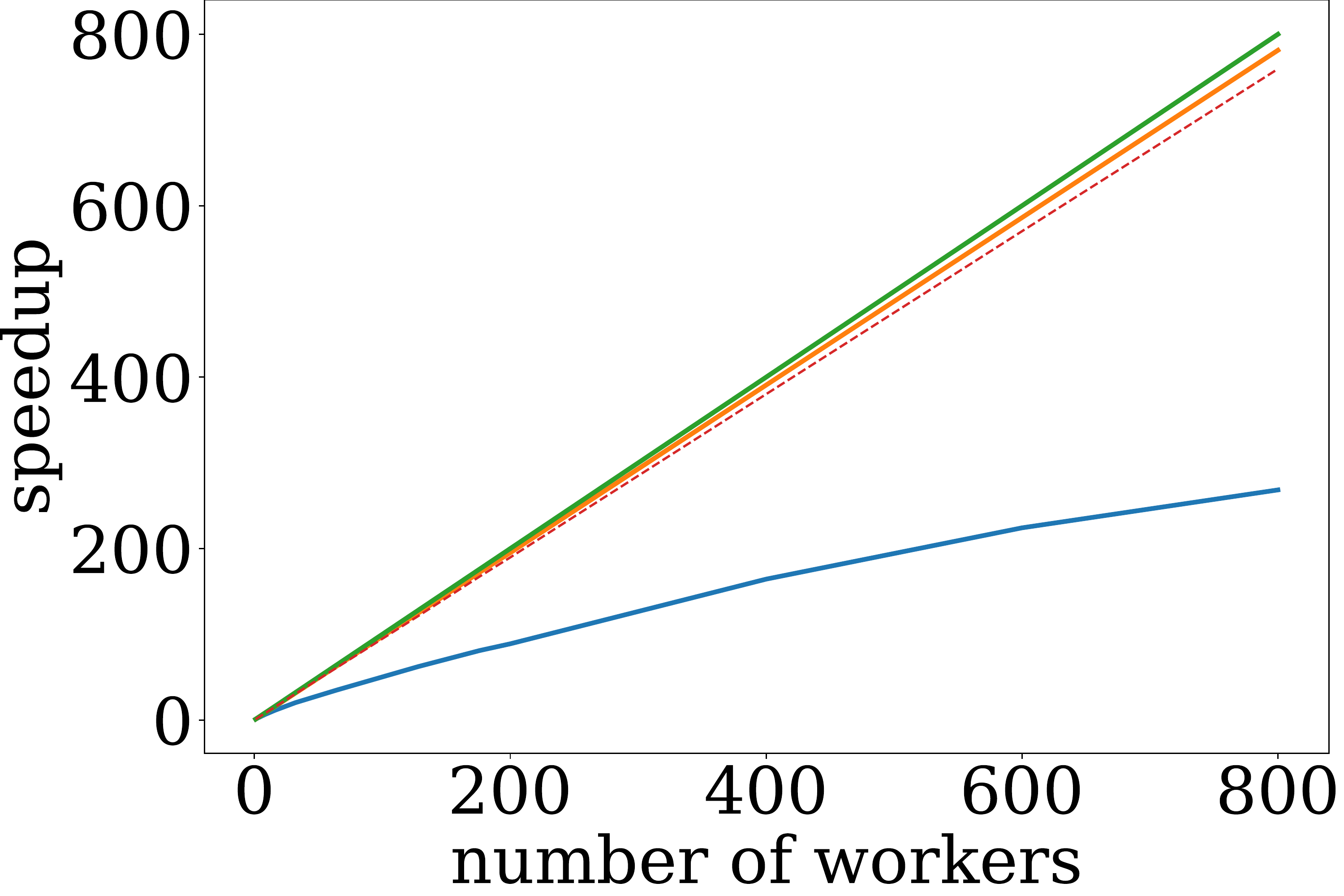}
      \caption{}
  \end{subfigure}
  \hfill
  \begin{subfigure}[]{0.32\textwidth}
      \includegraphics[width=\textwidth]{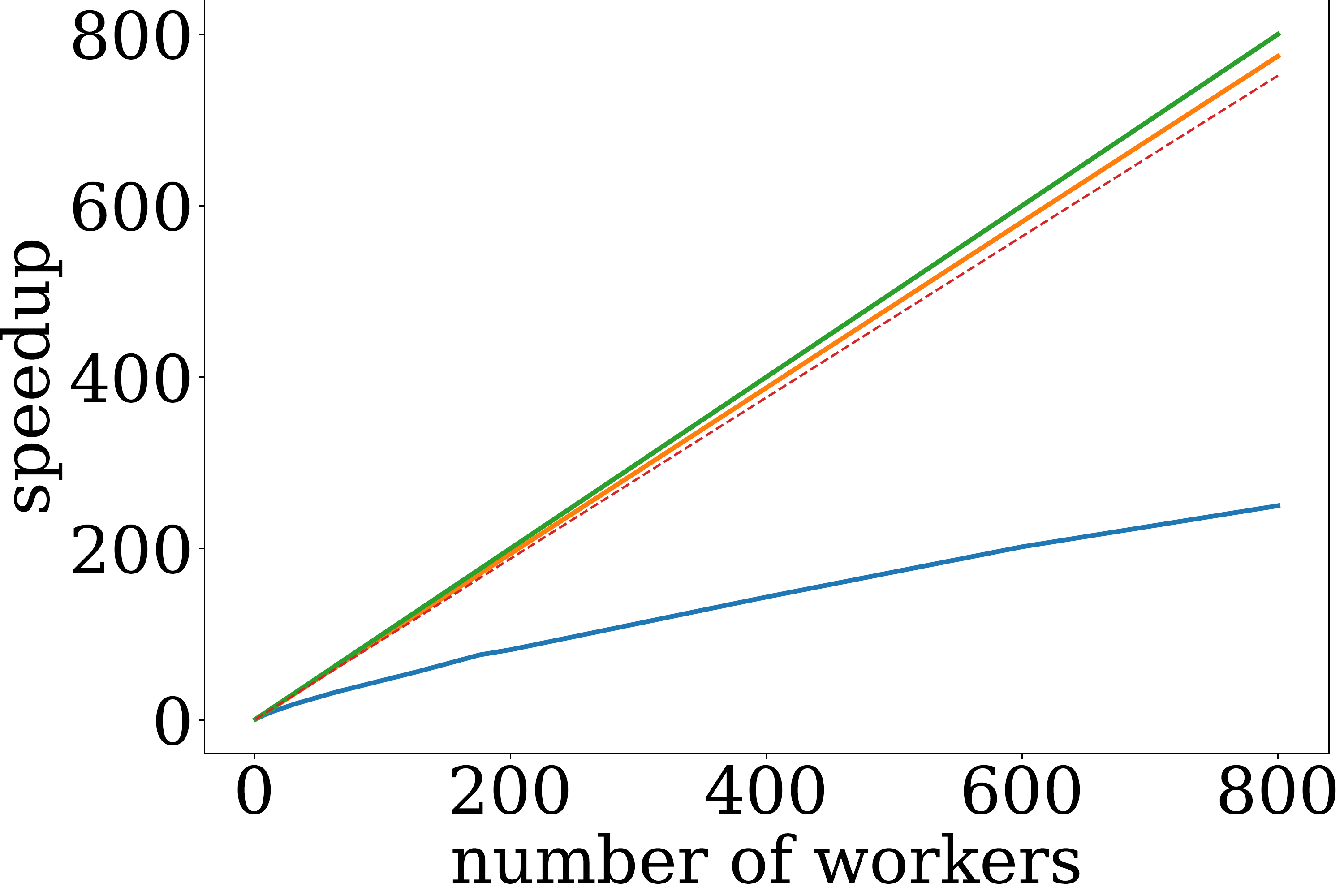}
      \caption{}
  \end{subfigure}
  \caption{\textbf{\textit{DropCompute} increases robustness to the noise variance.} Simulated scale graphs for a run with 12 accumulations. The duration of each accumulation is set to $0.45 + \epsilon$ seconds. $\epsilon$ is a stochastic lognormal random variable with a mean of $0.225$ and changing variance, according to the table below }
    \begin{tabular}{|c|cccc|c|}
    \hline 
    figure  & Mean($\epsilon$) & Var($\epsilon$) & $\epsilon$ distribution&&$\mathbb{E}[T]/\mathbb{E}[T_i]$ \\\hline
    a & 0.225 & 0.05 & lognormal & $LN(\mu=-1.84, \sigma=0.83)$ & 1.496\\
    b & 0.225 & 0.1 & lognormal & $LN(\mu=-2.04, \sigma=1.04)$ & 1.933\\
    c & 0.225 & 0.15 & lognormal & $LN(\mu=-2.18, \sigma=1.17)$ & 2.394\\
    d & 0.225 & 0.2 & lognormal & $LN(\mu=-2.29, \sigma=1.26)$ & 2.773\\
    e & 0.225 & 0.25 & lognormal & $LN(\mu=-2.38, \sigma=1.33)$ & 3.043\\
    f & 0.225 & 0.3 & lognormal & $LN(\mu=-2.46, \sigma=1.39)$ & 3.4\\
    \hline
    \end{tabular}
  \label{fig:noise_scale}
\end{figure}

\clearpage % forces the section to start on a new page
\section{Convergence with stochastic batch size} \label{appendix:proof}
In this section, we provide proof of theorem \ref{theorem:non_convex} (\ref{appendix:non_convex_proof}), as well as convergence proof for the loss itself in the convex case (\ref{appendix:convex_proof}). We also discuss the generalization properties with a stochastic batch in section \ref{app:generalization_discussion}.

\subsection{Proof for convex case} \label{appendix:convex_proof}
% \paragraph{Setting}
% We have a convex smooth function $f:\reals^n \mapsto \reals$, with a global minima $w^*$, and smoothness parameter $L$.

\begin{theorem} \label{theorem:convex_case} Under assumption \ref{assumption:convergence} and the specific case where $f$ is a convex function, for SGD with \textit{DropCompute} (Algorithm \ref{algorithm:drop_compute}), given $N$ workers and a local batch size $b_{local}$, we have that
\al
\label{eq:SGD_GUARANTEES_final}
\E[\mathcal{L}(\mathcal{D},\bar{\theta}) - \mathcal{L}(\mathcal{D},\theta^*)] \leq 
\frac{8Lb_\mathrm{max}\|\theta_1-\theta^*\|^2}{K} + \frac{6\sigma \|\theta_1-\theta^*\|}{\sqrt{K}}~.
\eal
where $K$ is the total number of samples used throughout the algorithm \footnote{We assume that the batch sizes may be stochastic but $K$ is predefined and deterministic}, and the expectation is with respect to the randomization introduced due to sampling from $\mathcal{D}$ throughout the optimization process.
\end{theorem}

\textbf{Proof of theorem \ref{theorem:convex_case}}

% SGD with changing batch size. \\
\textbf{Notation:} 
During the proof we denote by $b_i$ the total batch size (summed over all workers) that we employ at iteration $i$. $K = \sum_{i=1}^S b_i$ is the total number of samples along all $S$ iterations.
At iteration $i$ we maintain a weight vector $\theta_i\in\reals^d$ and query a gradient estimate $g_i$ based on a batch size of $b_i$ samples.
Thus, we set
$$ g_i = \frac{1}{b_i} \sum_{s=1}^{b_i} g_i^s $$
% Here, each $g_s^i$ is an unbiased estimate of $\nabla f(w_s)$ which is the true gradient at $w_s$. Namely,
% $$
% \E[g_s^i\vert w_s] = \nabla f(w_s)
% $$
% We also assume a bound on the variance:
% $$
% \E[\|g_s^i-\nabla f(w_s)\|^2\vert w_s] \leq \sigma^2~;\quad \forall i,s
% $$
where $g_i^s=\nabla\ell(z_i^s,\theta_i)$, and $z_i^s$ is randomly sampled from $\mathcal{D}$. We also maintain importance weights, $\alpha_i = b_i$.

Note that there exists $b_{\max}$ such that $\alpha_i\leq b_{\max}~;\forall i$

Now, the update rule:
\al \label{eq:Update}
\theta_{i+1} = \theta_i - \eta \alpha_i g_i
\eal
Eventually, we output
$$ \bar{\theta}  = \frac{1}{\alpha_{1:S}}\sum_{i=1}^S \alpha_i \theta_i $$
where $\alpha_{1:S}:=\sum_{i=1}^S \alpha_i.$

We assume that the  batch sizes $b_i$ are stopping times w.r.t.~the natural filtration induced by the samples we draw during the optimization process. Informally, this means that the value of $b_i$ depends only on the history of the samples that we have seen prior to setting $b_i$.

We can therefore prove the following lemma,
\begin{lemma} \label{lem:Unbiased}
    Upon choosing $\alpha_i = b_i$ the following holds,
    \als 
    \E [\alpha_i (g_i - \nabla \mathcal{L}(\mathcal{D},\theta_i)\vert \theta_i]
     = \E [\sum_{s=1}^{b_i} (g_i^s - \nabla\mathcal{L}(\mathcal{D},\theta_i))\vert \theta_i] = 0~.
     \eals 
\end{lemma}

% \section{Analysis}
Now, using standard analysis for Online Gradient Descent \citep{hazan2016introduction} with the update rule of \eqref{eq:Update} gives,
$$
\sum_{i=1}^S \alpha_i g_i\cdot(\theta_i-\theta^*) \leq \frac{\|\theta_1-\theta^*\|^2}{\eta} + \eta \sum_{i=1}^S \alpha_i^2 \|g_i\|^2
$$
Taking expectation and using Lemma~\ref{lem:Unbiased} gives,
$$
\E\sum_{i=1}^S \alpha_i \nabla\mathcal{L}(\mathcal{D},\theta_i)\cdot(\theta_i-\theta^*) \leq \frac{\|\theta_1-\theta^*\|^2}{\eta} + \eta \E\sum_{i=1}^S \alpha_i^2 \|g_i\|^2
$$
From convexity we know that $0\leq \mathcal{L}(\mathcal{D},\theta_i)-\mathcal{L}(\mathcal{D},\theta^*)\leq \nabla \mathcal{L}(\mathcal{D},\theta_i)\cdot(\theta_i-\theta^*)$, therefore the above implies,
\al \label{eq:SGD_GUARANTEES}
\E\sum_{i=1}^S \alpha_i( \mathcal{L}(\mathcal{D},\theta_i)-\mathcal{L}(\mathcal{D},\theta^*) ) \leq \frac{\|\theta_1-\theta^*\|^2}{\eta} + \eta \E\sum_{i=1}^S \alpha_i^2 \|g_i\|^2
\eal

Now, we write $g_i = \nabla \mathcal{L}(\mathcal{D},\theta_i) + \xi_i$ where $\xi_i = g_i - \nabla \mathcal{L}(\mathcal{D},\theta_i)$ and note that
$$\alpha_i \xi_i  = \sum_{s=1}^{b_i}(g_i^s - \nabla \mathcal{L}(\mathcal{D},\theta_i)) =  \sum_{s=1}^{b_i} \xi_i^i$$
where we denote $\xi_i^s: = g_i^s - \nabla \mathcal{L}(\mathcal{D},\theta_i)$.
Next, we shall use the following lemma,
\begin{lemma}\label{lem:BoundG1}
The following holds,
\al
\label{eq:VarStep2_FinalA}
\E \alpha_i^2 \|g_i\|^2 
&\leq
2 b_{\max} \E \alpha_i \|\nabla \mathcal{L}(\mathcal{D},\theta_i)\|^2 + 2\sigma^2 \E b_i~.
\eal 
Moreover, due to the $L$-smoothness of $f(\cdot)$, and global optimality of $\theta^*$, the following holds,
\al 
\label{eq:VarStep2_FinalB}
\E \alpha_i^2 \|g_i\|^2 
&\leq
4 b_{\max}L \E \alpha_i ( \mathcal{L}(\mathcal{D},\theta_i)-\mathcal{L}(\mathcal{D},\theta^*)) + 2\sigma^2 \E b_i~.
\eal 
\end{lemma}

\paragraph{Final Bound:}
Plugging the above lemma back into Eq.~\eqref{eq:SGD_GUARANTEES} gives,
\al
\label{eq:SGD_GUARANTEES2}
\E\sum_{i=1}^S \alpha_i( \mathcal{L}(\mathcal{D},\theta_i)-\mathcal{L}(\mathcal{D},\theta^*) ) 
&\leq 
\frac{\|\theta_1-\theta^*\|^2}{\eta} +  4\eta b_{\max} \E\sum_{i=1}^S\alpha_i L( \mathcal{L}(\mathcal{D},\theta_i)-\mathcal{L}(\mathcal{D},\theta^*)) + 8\eta\sigma^2 \E\sum_{i=1}^S b_i \non 
&\leq 
\frac{\|\theta_1-\theta^*\|^2}{\eta} +  4\eta b_{\max}L \E\sum_{i=1}^S\alpha_i ( \mathcal{L}(\mathcal{D},\theta_i)-\mathcal{L}(\mathcal{D},\theta^*)) + 8\eta\sigma^2 K ~,
\eal
where we used $K = \sum_{i=1}^S b_i$.

Now if we pick $\eta$ such that $4\eta b_{\max}L \leq 1/2$ then we can move the second term in the RHS to the LHS and obtain,
\al
\label{eq:SGD_GUARANTEES3}
\frac{1}{2}\E\sum_{i=1}^S \alpha_i( \mathcal{L}(\mathcal{D},\theta_i)-\mathcal{L}(\mathcal{D},\theta^*) ) 
&\leq 
\frac{\|\theta_1-\theta^*\|^2}{\eta} +  8\eta\sigma^2 K ~,
\eal
Thus, choosing $\eta = \min\left\{ \frac{\|\theta_1-\theta^*\|}{\sigma\sqrt{8K}} , \frac{1}{8L b_{\max}} \right\}$ gives the following bound,
\al
\label{eq:SGD_GUARANTEES4}
\E\sum_{i=1}^S \alpha_i( \mathcal{L}(\mathcal{D},\theta_i)-\mathcal{L}(\mathcal{D},\theta^*) ) 
&\leq 
8Lb_{\max}\|\theta_1-\theta^*\|^2 + 6\sigma \|\theta_1-\theta^*\|\sqrt{K}
\eal
Now, recalling that $K : = \sum_{i=1}^S b_i= \sum_{i=1}^S \alpha_i$ and using Jensen's inequality together with $\bar{\theta} = \frac{1}{\alpha_{1:S}}\sum_{i=1}^S \alpha_i \theta_i$ yields,
\al
\label{eq:SGD_GUARANTEES5}
\E[\mathcal{L}(\mathcal{D},\bar{\theta}) - \mathcal{L}(\mathcal{D},\theta^*)]
&\leq
\E\sum_{i=1}^S \frac{\alpha_i}{\alpha_{1:S}}( \mathcal{L}(\mathcal{D},\theta_i)-\mathcal{L}(\mathcal{D},\theta^*) ) 
\leq 
\frac{8Lb_{\max}\|\theta_1-\theta^*\|^2 + 6\sigma \|\theta_1-\theta^*\|\sqrt{K}}{\alpha_{1:S}}\non 
&\leq 
\frac{8Lb_{\max}\|\theta_1-\theta^*\|^2}{K} + \frac{6\sigma \|\theta_1-\theta^*\|}{\sqrt{K}}
\eal
where we used $\alpha_{1:S} = K$. \hspace{0.5cm}\qedsymbol{}

\subsection{Proof for non-convex case} \label{appendix:non_convex_proof}

\textbf{Proof of theorem \ref{theorem:non_convex}}

We use the same notation for $b_i$ and $g_i$ as before. And again used weights,
$$ \alpha_i = b_i $$
We also assume that $b_i\leq b_{\max}~,\forall i$.

The update rule is the following,
$$ \theta_{i+1} = \theta_i - \eta \alpha_i g_i $$
And the output is $\bar{\theta}$, where we define,
$$ \bar{\theta} = \theta_i~;\quad \text{w.p.}~~ \frac{\alpha_i}{\alpha_{1:S}} $$
Thus,
$$ \E \|\nabla \mathcal{L}(\mathcal{D},\bar{\theta})\|^2=\frac{1}{\alpha_{1:S}} \sum_{i=1}^S \alpha_i \E\|\nabla \mathcal{L}(\mathcal{D},\theta_i)\|^2 $$

% \subsection{Analysis}
Using smoothness,
\als
\mathcal{L}(\mathcal{D},\theta_{i+1}) 
&\leq 
\mathcal{L}(\mathcal{D},\theta_i) - \nabla \mathcal{L}(\mathcal{D},\theta_i)\cdot (\theta_{i+1}-\theta_i) + \frac{L}{2}\|\theta_{i+1}-\theta_i \|^2 \non
&\leq 
\mathcal{L}(\mathcal{D},\theta_i) - \eta \alpha_i \nabla \mathcal{L}(\mathcal{D},\theta_i)\cdot g_i + \frac{L\eta^2}{2}\|\alpha_i g_i \|^2 \non
&\leq 
\mathcal{L}(\mathcal{D},\theta_i) - \eta \alpha_i \|\nabla \mathcal{L}(\mathcal{D},\theta_i)\|^2-\eta \alpha_i \nabla \mathcal{L}(\mathcal{D},\theta_i)\cdot (g_i-\nabla \mathcal{L}(\mathcal{D},\theta_i)) + \frac{L\eta^2}{2}\|\alpha_i g_i \|^2 
\eals
Re-arranging the above yields,
\als
\eta \alpha_i \|\nabla \mathcal{L}(\mathcal{D},\theta_i)\|^2
&\leq 
\mathcal{L}(\mathcal{D},\theta_i) - \mathcal{L}(\mathcal{D},\theta_{i+1}) -\eta \alpha_i \nabla \mathcal{L}(\mathcal{D},\theta_i)\cdot (g_i-\nabla \mathcal{L}(\mathcal{D},\theta_i)) + \frac{L\eta^2}{2}\|\alpha_i g_i \|^2 
\eals
Summing the above, and dividing by $\eta$, we obtain,
\al \label{eq:NonCvx1}
\sum_{i=1}^S \alpha_i \|\nabla \mathcal{L}(\mathcal{D},\theta_i)\|^2
&\leq 
\frac{1}{\eta}(\mathcal{L}(\mathcal{D},\theta_1) - \mathcal{L}(\mathcal{D},\theta_{S+1})) - \sum_{i=1}^S\alpha_i \nabla \mathcal{L}(\mathcal{D},\theta_i)\cdot (g_i-\nabla \mathcal{L}(\mathcal{D},\theta_i)) + \frac{L\eta}{2}\sum_{i=1}^S\|\alpha_i g_i \|^2  \non 
&\leq 
\frac{1}{\eta}(\mathcal{L}(\mathcal{D},\theta_1) - \mathcal{L}(\mathcal{D},\theta^*)) - \sum_{i=1}^S\alpha_i \nabla \mathcal{L}(\mathcal{D},\theta_i)\cdot (g_i-\nabla \mathcal{L}(\mathcal{D},\theta_i)) + \frac{L\eta}{2}\sum_{i=1}^S\|\alpha_i g_i \|^2  
\eal
where we uses $\mathcal{L}(\mathcal{D},\theta^*) \leq \mathcal{L}(\mathcal{D},\theta_{S+1})$ since $\theta^*$ is the global minimum of $\mathcal{L}(\mathcal{D},\cdot)$.

Now recall that from Lemma~\ref{lem:Unbiased} we have,
\al \label{eq:Doob2B}
\E[\alpha_i (g_i-\nabla \mathcal{L}(\mathcal{D},\theta_i))\vert \theta_i] = 0~. 
\eal
And from Lemma~\ref{lem:BoundG1} we have,
\al
\label{eq:VarStep2B}
\E \alpha_i^2 \|g_i\|^2 
&\leq
2 b_{\max} \E \alpha_i \|\nabla \mathcal{L}(\mathcal{D},\theta_i)\|^2 + 2\sigma^2 \E b_i~.
\eal 

Thus, taking expectation in Eq.~\eqref{eq:NonCvx1}, and  plugging Eq.~\eqref{eq:Doob2B} and \eqref{eq:VarStep2B}, yields,
\al \label{eq:NonCvx2}
\E \sum_{i=1}^S \alpha_i \|\nabla \mathcal{L}(\mathcal{D},\theta_i)\|^2
&\leq 
\frac{1}{\eta}(\mathcal{L}(\mathcal{D},\theta_1) - \mathcal{L}(\mathcal{D},\theta^*)) - \sum_{i=1}^S\E \alpha_i \nabla \mathcal{L}(\mathcal{D},\theta_i)\cdot (g_i-\nabla \mathcal{L}(\mathcal{D},\theta_i)) + \frac{L\eta}{2}\sum_{i=1}^S \E\|\alpha_i g_i \|^2 \non
&\leq 
\frac{1}{\eta}(\mathcal{L}(\mathcal{D},\theta_1) - \mathcal{L}(\mathcal{D},\theta^*)) + L\eta b_{\max} \sum_{i=1}^S \E\alpha_i \|\nabla \mathcal{L}(\mathcal{D},\theta_i)\|^2 
+L\eta \sigma^2 \E \sum_{i=1}^S b_i \non
&\leq 
\frac{1}{\eta}(\mathcal{L}(\mathcal{D},\theta_1) - \mathcal{L}(\mathcal{D},\theta^*)) + L\eta b_{\max} \sum_{i=1}^S \E\alpha_i \|\nabla \mathcal{L}(\mathcal{D},\theta_i)\|^2 
+L\eta \sigma^2 \cdot K
\eal
where the last line uses $K:= \sum_{i=1}^S b_i$.

Now if we pick $\eta$ such that $\eta b_{\max}L \leq 1/2$ then we can move the second term in the RHS to the LHS and obtain,
\al \label{eq:NonCvx3}
\frac{1}{2}\E \sum_{i=1}^S \alpha_i \|\nabla \mathcal{L}(\mathcal{D},\theta_i)\|^2
&\leq 
\frac{1}{\eta}(\mathcal{L}(\mathcal{D},\theta_1) - \mathcal{L}(\mathcal{D},\theta^*))  
+L\eta \sigma^2 \cdot K
\eal
Thus, choosing $\eta = \min\left\{ \frac{\sqrt{\mathcal{L}(\mathcal{D},\theta_1) - \mathcal{L}(\mathcal{D},\theta^*)}}{\sigma\sqrt{LK}} , \frac{1}{2L b_{\max}} \right\}$ gives the following bound,
\al \label{eq:NonCvx4}
\E \sum_{i=1}^S \alpha_i \|\nabla \mathcal{L}(\mathcal{D},\theta_i)\|^2
&\leq  
2L b_{\max}(\mathcal{L}(\mathcal{D},\theta_1) - \mathcal{L}(\mathcal{D},\theta^*))  
+2\sigma\sqrt{L(\mathcal{L}(\mathcal{D},\theta_1) - \mathcal{L}(\mathcal{D},\theta^*)) } \cdot \sqrt{K}
\eal
Dividing by $K:=\alpha_{1:S}$ and using the definition of $\bar{\theta}$ yields,
\al \label{eq:NonCvx5}
\E\|\nabla \mathcal{L}(\mathcal{D},\bar{\theta})\|^2 = 
\E \frac{1}{\alpha_{1:S}}\sum_{i=1}^S \alpha_i \|\nabla \mathcal{L}(\mathcal{D},\theta_i)\|^2
&\leq 
\frac{2L b_{\max}(\mathcal{L}(\mathcal{D},\theta_1) - \mathcal{L}(\mathcal{D},\theta^*))}{K}  
+\frac{2\sigma\sqrt{L(\mathcal{L}(\mathcal{D},\theta_1) - \mathcal{L}(\mathcal{D},\theta^*)) }}{\sqrt{K}}
\eal 

\subsection{Remaining Proofs}
\subsubsection{Proof of Lemma~\ref{lem:BoundG1}} 

\begin{proof}[Proof of Lemma~\ref{lem:BoundG1}]
We can write,
\al \label{eq:VarStep1}
\alpha_i^2 \|g_i\|^2 
&= 
\|b_i \nabla \mathcal{L}(\mathcal{D},\theta_i) + \sum_{s=1}^{b_i}\xi_i^s\|^2 \non
&\leq 
2\|b_i \nabla \mathcal{L}(\mathcal{D},\theta_i)\|^2 + 2\|\sum_{s=1}^{b_i}\xi_i^s\|^2 \non
&=
2 b_i^2 \|\nabla \mathcal{L}(\mathcal{D},\theta_i)\|^2 + 2\|\sum_{s=1}^{b_i}\xi_i^s\|^2 \non
&\leq
2 b_{\max} \alpha_i \|\nabla \mathcal{L}(\mathcal{D},\theta_i)\|^2 + 2\|\sum_{s=1}^{b_i}\xi_i^s\|^2 \non
%&\leq
%4 b_{\max} \alpha_s L( f(w_s)-f(w^*)) + 2\|\sum_{i=1}^{b_s}\xi_s^i\|^2~,
\eal
where the second line uses $\|a+b\|^2 \leq 2\|a\|^2+2\|b\|^2$; the fourth line uses $b_i\leq b_{\max}$ as well as $\alpha_i = b_i$ implying that $b_i^2 \leq b_{\max}\alpha_i$.

\paragraph{Bounding $E\|\sum_{s=1}^{b_i}\xi_i^s\|^2$:}
Given $i$ and $\theta_i$, Let us define the following sequence, $Q_0=0$, $Q_1 = \|\xi_i^1\|^2-\sigma$, and for any $k>1$
$$
Q_k = \sum_{s=1}^k\|\xi_i^k\|^2-\sigma^2 \cdot k + 2\sum_{s=1}^j \sum_{n=s+1}^j \xi_i^s\cdot \xi_i^n   
$$
It can be directly shown that $\{Q_k\}_k$ is a Supermartingale sequence, and that
$$
\|\sum_{s=1}^{b_i}\xi_i^s\|^2 = Q_{b_i}+ \sigma^2 \cdot b_i
$$
Thus, since $b_i$ is a bounded stopping time, we can use Doob's optional stopping theorem which implies that,
$$
\E \|\sum_{s=1}^{b_i}\xi_i^s\|^2 = \E Q_{b_i}+ \sigma^2 \E\cdot b_i \leq EQ_0 + \sigma^2 \E\cdot b_i = 0+\sigma^2 \E b_i
$$
Plugging the above back into Eq.~\eqref{eq:VarStep1} yields,
\al
\label{eq:VarStep2_Pre}
\E \alpha_i^2 \|g_i\|^2 
&\leq
2 b_{\max} \E \alpha_i \| \nabla \mathcal{L}(\mathcal{D},\theta_i)\|^2 + 2\sigma^2 \E b_i~.
\eal 
Now, since $\mathcal{L}(\mathcal{D},\cdot)$ is $L$-smooth and $\theta^*$ is its global minima, then the following holds: $\|\nabla \mathcal{L}(\mathcal{D},\theta_i)\|^2 \leq 2L( \mathcal{L}(\mathcal{D},\theta_i)-\mathcal{L}(\mathcal{D},\theta^*))$; See e.g.~\citet{levy2017online} for the proof.
Plugging this into the above equation we obtain,
\al
\label{eq:VarStep2_Post}
\E \alpha_i^2 \|g_i\|^2 
&\leq
4 b_{\max} \E \alpha_i L( \mathcal{L}(\mathcal{D},\theta_i)-\mathcal{L}(\mathcal{D},\theta^*)) + 2\sigma^2 \E b_i~.
\eal 
\end{proof}

\subsubsection{Proof of Lemma~\ref{lem:Unbiased}}
\begin{proof}[Proof of Lemma~\ref{lem:Unbiased}]
We can  define the following Martingale sequence for each  step $s$: $M_0 = 0$, and
$M_j = \sum_{s=1}^j (g_i^s - \nabla \mathcal{L}(\mathcal{D},\theta_i))$ for any $j=1,2,\ldots$.

Thus, since the mixing time $b_i$ is bounded by $b_{\max}$, then according to Doob's optional stopping theorem \citep{levin2017markov} we have that,
\al \label{eq:Doobs}
\E[M_{b_i}\vert \theta_i] =  \E [\sum_{s=1}^{b_i} (g_i^s - \nabla \mathcal{L}(\mathcal{D},\theta_i))\vert \theta_i] = \E [M_0\vert \theta_i] = 0~. 
\eal

Now, notice that for any $i$, we have $\alpha_i (g_i-\nabla \mathcal{L}(\mathcal{D},\theta_i)) = M_{b_i}$, and therefore,
\al \label{eq:Doob2}
\E[\alpha_i (g_i-\nabla \mathcal{L}(\mathcal{D},\theta_i))\vert \theta_i] =\E[ M_{b_i}\vert \theta_i] = 0~. 
\eal

\end{proof}

\subsection{Generalization discussion} \label{app:generalization_discussion}
An interesting observation arising from our results is the small impact of gradient dropping as measured in final test accuracy. One explanation for this can be based on viewing \textit{DropCompute} as noise induced over gradients. Optimization using variants of SGD is inherently noisy due to the use of data samples used to evaluate the intermediate error. The stochastic nature of computed weight gradients was previously found to provide generalization benefits for the final trained model, although this is still part of ongoing debate \citep{geiping2022stochastic}. Nevertheless, several works found generalization benefits with the \emph{injection} of noise into the weights' gradients \citep{gradnoise} or their use in computed update rule \citep{lrdropout}.

\end{document}